\documentclass{article}
\usepackage{amsmath}
\usepackage{lipsum}
\usepackage{amsfonts}
\usepackage{graphicx}
\usepackage{epstopdf}
\usepackage{algorithm}
\usepackage{algorithmic}
\usepackage{amsopn}
\usepackage{amsthm}
\usepackage{hyperref}

\title{A  vector total variation of feature image model for image restoration}
\author{
	Wei Wang\footnotemark[1]\;,
	Xiang-Gen Xia\footnotemark[1] \footnotemark[2]\;, 
	Shengli Zhang\footnotemark[1]\;,
	Chuanjiang He\footnotemark[3]
}

\begin{document}

\maketitle

\footnotetext[1]{College of Information Engineering, Shenzhen
	University, Shenzhen, China. (\href{mailto:aiyulunhui@sina.com}{aiyulunhui@sina.com}, \href{mailto:xxia@ee.udel.edu}{xxia@ee.udel.edu}, \href{mailto:zsl@szu.edu.cn}{zsl@szu.edu.cn}
\href{mailto:cjhe@cqu.edu.cn}{cjhe@cqu.edu.cn}) }
\footnotetext[2]{Department of Electrical
	and Computer Engineering, University of Delaware, Newark, DE 19716, USA.}
\footnotetext[3]{College of Mathematics and Statistics,Chongqing University,Chongqing 401331, China.
	
This research is partially supported by the research grants from the Chinese NSF project (61372078), National Natural Science Foundation of China (No. 61701320, No. 61561019), Guangdong NSF project (2014A030313549), Key project of Guangdong (2016KZDXM006) and the Shenzhen NSF project (JCYJ20160226192223251).}

\theoremstyle{definition}
\newtheorem{Lemma}{Lemma}
\newtheorem{Remark}{Remark} 
\newtheorem{Theorem}{Theorem}
\newtheorem{Definition}{Definition}
\newtheorem*{keywords}{keywords}
\newcommand{\tabincell}[2]{
	\begin{tabular}{@{}#1@{}}#2\end{tabular}
}

\begin{abstract}
   In this paper, we propose a vector total variation (VTV) of feature image model for image restoration.  The  VTV imposes different smoothing powers on different features (e.g. edges and cartoons) based on choosing various regularization parameters. Thus, the model can simultaneously preserve edges and remove noises. Next, the existence of  solution for the model is proved and the split Bregman algorithm is used to solve the model. At last, we use the wavelet filter banks to explicitly define the feature operator and 
   present some experimental results to show its  advantage over the related methods  in both quality and efficiency.
\end{abstract}

\begin{keywords}
  Image restoration, wavelet, vectore total variation, variational method，
\end{keywords}

\section{Introduction}
Image restoration including image denoising, deblurring, inpainting etc. is a procedure of improving the quality of a given image that is degraded in various ways during the process of acquisition and commutation. Since many advanced applications in computer vision depend heavily on the input of high quality images, image restoration becomes an indispensable and preprocessing step of these applications. Therefore, image restoration is a basic but very important area in image processing and analysis. 
Image restoration can be modeled as a linear inverse problem:
$$f = Au + \eta$$
where $f$  is the observed images, $A$  is a  linear operator and $\eta $  represents the additive noise whose type depends on its probability density function (PDF). The goal of image restoration is to find the unknown true image $u$  from the observed image  $f$. The problem is usually an ill-posed inverse problem and often solved by imposing a prior regularization  assumption. Among all regularization-based methods for image restoration, variational based \cite{Choksi2017Anisotropic}\cite{Marquina2009Nonlinear}\cite{Chan1998Total}\cite{Wen2012A}\cite{Ng2011Fast}\cite{Bredies2012A}\cite{WangH} and wavelet frame based \cite{Cai2009Linearized}\cite{Chaux2011A}\cite{Fadili2009Inpainting}\cite{Figueiredo2003An}\cite{Cai2009Split}\cite{Elad2005Simultaneous}\cite{Starck2005Image}\cite{Fan}\cite{Cai2009Split}\cite{Elad2005Simultaneous}\cite{Starck2005Image}\cite{Fan}\cite{Cai2007Restoration}\cite{Cai2008A}\cite{Chan2004Tight} methods have attracted a lot of attention in the past.

The common assumption of the regularization based methods is that images can be sparsely approximated in some transformed domains. Such transforms can be gradient operator, wavelet frame transform, Fourier transform, Gabor transform  etc. In order to utilize the sparsity, one solves (1) by finding a sparse solution in the corresponding transformed domain. Typically, the ${l_1}$  norm is used as a penalty of the sparsity. 
One of the well-known variational approaches is the Rudin-Osher-Fatemi (ROF) \cite{ROF} model, which imposes the total variation (TV) regularization on the image  $u$:
$$\mathop {\inf }\limits_{u \in BV(\Omega )} \left\{ {\lambda \int_\Omega  {|\nabla u|}  + \int_\Omega  {{{(u - f)}^2}} dx} \right\}.$$
Here, $\Omega$ is the image domian, $BV(\Omega)$ is the bounded variational space,  $\int_\Omega  {|\nabla u|} $  and $\int_\Omega  {{{(u - f)}^2}} dx$ are the TV (regularization) term and fidelity (fitting) term, respectively. The ROF model performs well for removing noise while preserving edges. However, it tends to produce piecewise constant results (called staircase effect in the literature). After the ROF model was proposed, the TV regularization has been extended to many other image restoration applications (such as deblurring, inpainting, superresolution, etc.) and has been modified in a variety of ways to improve its performance \cite{Gilboa2008Nonlocal}\cite{Chao2010An}\cite{Louchet2014Total}\cite{Sutour2014Adaptive}.To solve the ROF model, many algorithms have been proposed \cite{Chambolle2004An}\cite{Zhu2008An}\cite{Beck2009Fast}\cite{Esser2010A}\cite{Goldstein2009The} \cite{ChambolleAntonin}\cite{WangHe}and the split Bregman method \cite{Goldstein2009The} is the one of the most widely used ones . 

Wavelet frames represent images as a summarization of smooth components (i.e. cartoon), and local features (i.e. singularities). In wavelet frame domain, smooth image components are the coefficient images obtained from low-pass filters, while local features are those obtained from high-pass filters. Thus, we can impose different strengths of regularization on smooth components and local features separately when restoring images in the wavelet frame transform domain. There exist plenty of wavelet frame based image restoration models in the literature, such as, the synthesis based approach \cite{Cai2009Linearized}\cite{Chaux2011A}\cite{Fadili2009Inpainting}\cite{Figueiredo2003An}, the analysis based approach \cite{Cai2009Split}\cite{Elad2005Simultaneous}\cite{Starck2005Image}\cite{Fan}, and the balanced approach \cite{Cai2007Restoration}\cite{Cai2008A}\cite{Chan2004Tight}. A typical analysis based model of ${l_1}$  norm regularization has a similar form with the ROF model:
$$\mathop {\min }\limits_u \lambda {\left\| {Wu} \right\|_1} + \frac{1}{2}\left\| {Au - f} \right\|_2^2,$$
where $W$  represents the wavelet frame transform. Since the ${l_1}$  norm regularizations usually exist in the wavelet frame based models, the split Bregman algorithm is also a widely used method to solve them.

In this paper, we propose a vector total variation (VTV) of feature image model  for image restoration where the VTV is used as the regularization. The vector images are generated by mapping the original image into the feature spaces which represent different features of the image such as edges and cartoons. By choosing different regularization parameters, we can impose different smoothing powers on different features. Thus, the model can simultaneously preserve edges and remove noises. Next, the existence of  solution for the model is proved and the split Bregman algorithm is used to solve the model. At last, we use the wavelet filter banks to explicitly define the feature operator and  present some experimental results to show its  advantage over the related methods  in both quality and efficiency.

The rest of the paper is organized as follows. In Section 2.1, we describe the proposed model, prove the existence of its solution and use the convolution to define the feature operators. In Section 2.2, we use the split Bregamn algorithm to solve the model. In Section 3, we present some experimental results to show the advantage of our model. In Section 4, we conclude this paper.

\section{The wavelet  coefficient total variational model}

\subsection{Proposed model}
Before describing the model, we briefly introduce some mathematical definitions related to our paper.

Let $\Omega$ be a bounded open subset of $R^2$, $f\in L^2(\Omega,R)$,  $F=(F_1,F_2,...,F_m)$ be an operator such that $Ff=(F_1f,F_2f,...,F_mf)\in L^2(\Omega,R^m)$ and $F_if\in L^2(\Omega,R)$. In this paper, we call $F_if$ a feature image of $f$ and $F$ a feature operator.
\begin{Definition}
	The inner product in  $L^2(\Omega,R^m)$ is defined as 
	$$<u,v>_{L^2(\Omega,R^m)}=\int_\Omega\sum_{i=1}^{m}u_iv_i,$$
	where $u=(u_1,u_2,...,u_m),v=(v_1,v_2,...,v_m)\in L^1(\Omega,R^m)$.
\end{Definition}	
\begin{Definition}
	$BV(\Omega)$ is the subspace of $L^1(\Omega,R)$ such that the following quantity is finite:
	$$|f|_{TV}=\mathop {\sup}\limits_{\eta \in K}<f,{\rm{div}}\eta>_{L^2(\Omega,R)},$$
	where ${\rm{div}}\eta=\frac{\partial\eta_1}{\partial x}+\frac{\partial\eta_2}{\partial y}$, $K=\{\eta=(\eta_1,\eta_2)\in C_0^1(\Omega,R^2):\left\|\eta\right\|_{L^\infty(\Omega,R^2)} \le 1\}$, $C_0^1(\Omega,R^2)$ is the space of continuously differentiable functions with compact support in $\Omega$, $\left\|\eta\right\|_{L^\infty(\Omega,R^2)}=\mathop{\sup}\limits_{x\in\Omega}\sqrt{\eta_1^2+\eta_2^2}$ and the quantity $|f|_{TV}$ is called the total variation of $f$.
\end{Definition}	

\begin{Definition}
  The vector total variation of function $g=(g_1,g_2,...,g_m)\in L^1(\Omega,R^m)$ is defined as 
  $$|g|_{VTV}=\mathop {\sup}\limits_{\xi \in P}<g,{\rm{div}}\xi>_{L^2(\Omega,R^m)},$$
  where $\xi=(\xi_1,\xi_2,...,\xi_m)$, ${\rm{div}}\xi=({\rm{div}}\xi_1,{\rm{div}}\xi_2,...,{\rm{div}}\xi_m)$, $P=\{\xi\in C_0^1(\Omega,R^{2\times m}):\left\|\xi(x)\right\|_{L^\infty(\Omega,R^{2\times m})} \le 1\}$ and $\left\|\xi(x)\right\|_{L^\infty(\Omega,R^{2\times m})}=\mathop {\max}\limits_{i=1,2...,m} \left\|\xi_i(x)\right\|_{L^\infty(\Omega,R^{2})}.$
\end{Definition}	
For $g$ smooth enough, we have \cite{Duval2008A}\cite{VTV}
\begin{equation}\label{1}
|g|_{VTV}= \sum_{i=1}^{m}|g_i|_{TV}= \sum_{i=1}^{m}\int_\Omega|\nabla g_i|.
\end{equation}
For more details about the vector total variation, refer to \cite{Duval2008A}\cite{VTV}\cite{Ambrosio}\cite{Schnei}\cite{Bresson2017Chan}.

In this paper, we propose the following model for image recovery:
\begin{equation}\label{2}
\mathop {\arg \inf }\limits_{u \in BV(\Omega )}\{ E(u) = {\left| {\vec \lambda Fu} \right|_{VTV}} +\frac{1}{2} {\left\|Au - f\right\|_{L^2(\Omega,R)}^2}\},
\end{equation}
where $A$ is a linear operator, $f$ is an observed image, $F$ is a feature operator, $\vec \lambda=(\lambda_1,\lambda_2,...,\lambda_m)\in (R^+)^m$ is the regularization parameter vector and $\vec \lambda Fu=(\lambda_1F_1u,\lambda_2F_2u,...,\lambda_mF_mu)$ represents the feature images. The ideal case is that different $F_mu$ can represent different features of the image $u$, such as edges, cartoons, etc. and thus we can impose different smooth strengths on them by choosing different $\lambda_i$. Clearly, if $m=1$ and $F=I$, then the proposed model is reduced to the ROF model.
   
In the following, we show the existencce of the solution for problem (\ref{2}).
\begin{Lemma}
	If $F_i$ is linear, then the adjoint operator $F^*$ of $F$ is
	$$F^*g=\sum_{i=1}^mF_i^*g_i,$$
	where  $g=(g_1,g_2,...,g_m)\in L^2(\Omega,R^m)$ and $F_i^*$ is the adjoint operator of $F_i$.
\end{Lemma}
\begin{proof}
	Let $f\in  L^2(\Omega,R)$ and $g \in L^2(\Omega,R^m)$, then 
	\begin{align*}
	<Ff,g>=&\sum_{i=1}^m<F_if,g_i>\\
	=&\sum_{i=1}^m<f,F_i^*g_i>\\
	=&<f,\sum_{i=1}^mF_i^*g_i>
	\end{align*}
Thus, 	$F^*g=\sum_{i=1}^mF_i^*g_i$.
\end{proof}

\begin{Lemma}
	If $F$ is linear and bounded and for any $\xi\in C_0^1(\Omega,R^{2\times m})$, the operator ${\rm{div}}$ and $F^*$ are commutative, i.e., 
	$F^*{\rm{div}}\xi={\rm{div}} F^*\xi$,
	then for any $f\in BV(\Omega)$, we have
	$${|Ff}|_{VTV}\le\left\|F\right\||f|_{TV}.$$
\end{Lemma}
\begin{proof}
\begin{align*}
|Ff|_{VTV}=&\mathop {\sup}\limits_{\xi \in P}<Ff,{\rm{div}}\xi>\\
=&\mathop {\sup}\limits_{\xi \in P}<f,F^*{\rm{div}}\xi>\\
=&\mathop {\sup}\limits_{\xi \in P}\left\|F^*\xi\right\|_{L^\infty(\Omega,R^2)}<f,{\rm{div}}(\frac{F^*\xi}{\left\|F^*\xi\right\|_{L^{\infty}(\Omega,R^2)}})>\\
\le&\mathop {\sup}\limits_{\xi \in P}\left\|F^*\right\|\left\|\xi\right\|_{L^\infty(\Omega,R^{2\times m})}<f,{\rm{div}}(\frac{F^*\xi}{\left\|F^*\xi\right\|_{L^\infty(\Omega,R^2)}})>\\
\le&\left\|F^*\right\||f|_{TV}\\
=&\left\|F\right\||f|_{TV}
\end{align*}
\end{proof}
\begin{Lemma}
	If $F$ is linear and bounded,  invertible, and for any $\eta\in C_0^1(\Omega,R^{2})$, the operator ${\rm{div}}$ and $(F^{-1})^*$ are commutative, i.e., 
	$(F^{-1})^*{\rm{div}}\eta={\rm{div}} (F^{-1})^*\eta$,
	then for any $f\in BV(\Omega)$, we have
	$$|f|_{TV}\le\left\|F^{-1}\right\||Ff|_{VTV}.$$
\end{Lemma}

\begin{proof}
	\begin{align*}
	|f|_{TV}=&\mathop {\sup}\limits_{\eta \in K}<F^{-1}Ff,{\rm{div}}\eta>\\
	=&\mathop {\sup}\limits_{\eta \in K}<Ff,(F^{-1})^*{\rm{div}}\eta>\\
	=&\mathop {\sup}\limits_{\eta \in K}\left\|(F^{-1})^*\eta\right\|_{L^\infty(\Omega,R^{2\times m})}<Ff,{\rm{div}}(\frac{(F^{-1})^*\eta}{\left\|(F^{-1})^*\eta\right\|_{L^\infty(\Omega,R^{2\times m})}})>\\
	\le&\mathop {\sup}\limits_{\eta \in K}\left\|(F^{-1})^*\right\|\left\|\eta\right\|_{L^\infty(\Omega,R^{2})}<Ff,{\rm{div}}(\frac{(F^{-1})^*\eta}{\left\|(F^{-1})^*\eta\right\|_{L^\infty(\Omega,R^{2\times m})}})>\\
	\le&\left\|(F^{-1})^*\right\|\mathop {\sup}\limits_{\eta \in K}<Ff,{\rm{div}}(\frac{(F^{-1})^*\eta}{\left\|(F^{-1})^*\eta\right\|_{L^\infty(\Omega,R^{2\times m})}})>\\
	\le&\left\|(F^{-1})\right\||Ff|_{VTV}
	\end{align*}
\end{proof}
\begin{Remark}
	From the proof, we can see the condition we need is that there exists $F_L^{-1}: L^2(\Omega,R^m)\to L^2(\Omega,R)$ such that $F_L^{-1}F=I$ and $(F_L^{-1})^*{\rm{div}}\eta={\rm{div}}(F_L^{-1})^*\eta$ for any $\eta\in C_0^1(\Omega,R^{2})$.
\end{Remark}
\begin{Lemma}
    If $F$ is  linear and bounded, and invertible, then  for any $\xi\in C_0^1(\Omega,R^{2\times m})$ and $\eta\in C_0^1(\Omega,R^{2})$, the following two conditions are equivalent:
    \begin{description}
    	\item[(a)] $F^*{\rm{div}}\xi={\rm{div}} F^*\xi,$
    	\item[(b)] $(F^{-1})^*{\rm{div}}\eta={\rm{div}} (F^{-1})^*\eta.$
    \end{description}
\begin{proof} By 
	$(a) \Rightarrow (b):$
	Let $\xi=(F^{-1})^*\eta$ in (a), then
	$$F^*{\rm{div}}(F^{-1})^*\eta={\rm{div}}F^*(F^{-1})^*\eta={\rm{div}}(F^{-1}F)^*\eta={\rm{div}}\eta.$$
	Multiplying $(F^{-1})^*$ on both sides of the above equation, we obtain
	$${\rm{div}}(F^{-1})^*\eta=(F^{-1})^*{\rm{div}}\eta,$$
	
	$(b) \Rightarrow (a):$ Let $\eta=F^*\xi$ in (b), then we have
		$$(F^{-1})^*{\rm{div}}F^*\xi={\rm{div}} (F^{-1})^*F^*\xi={\rm{div}}\xi.$$
		Therefore, $${\rm{div}}F^*\xi=F^*{\rm{div}}\xi.$$
\end{proof} 
\end{Lemma}
\begin{Remark}
	By $1=\left\|F^{-1}F\right\|\le\left\|F^{-1}\right\|\left\|F\right\|$, we only have $\left\|F^{-1}\right\|\ge\left\|F\right\|^{-1}$. If  $\left\|F^{-1}\right\|=\left\|F\right\|^{-1}$, then combining Lemma 2 and Lemma 3, we can obtain ${|Ff}|_{VTV}=\left\|F\right\||f|_{TV}.$
\end{Remark}
\begin{Theorem}
	If the conditions of Remark 1 are satisfied, $\lambda_{\min}=\mathop{\min}\limits_{i=1,2,...m}\lambda_i>0$ and  $A\otimes 1=\int_\Omega A(x)dx\ne 0$, then Problem (\ref{2}) admits a solution.
\end{Theorem}
\begin{proof}
	Let $u_n$ be a minimizing sequence of Problem (\ref{2}), then there exists a constant $M>0$ such that $E(u_n)\le M,$ i.e.
$${\left| {\vec \lambda Fu} \right|_{VTV}} + \frac{1}{2}{\left\|Au - f\right\|_{L^2(\Omega,R)}^2}\le M.$$
Therefore, 
\begin{equation}\label{3}
{\left| {\vec \lambda Fu} \right|_{VTV}}\le M
\end{equation}
\begin{equation}\label{4}
{\left\|Au - f\right\|_{L^2(\Omega,R)}^2}\le 2M
\end{equation}
From equation (\ref{3}), by Lemma 3, we have 
\begin{equation}\label{5}
\begin{aligned}
|u|_{TV}&\le\left\|F^{-1}\right\||Fu|_{VTV}\\
&\le\frac{1}{\lambda_{\min}}\left\|F^{-1}\right\||\vec \lambda Fu|_{VTV}\\
&\le \frac{M}{\lambda_{\min}}\left\|F^{-1}\right\|
\end{aligned}
\end{equation}
As proved in in \cite{Ambrosio}\cite{Vese}, from equation (\ref{4}), we can obtain
\begin{equation}\label{6}
\left\|u\right\|_{L^2(\Omega,R)}\le M
\end{equation}

Combining equations (\ref{5}) and (\ref{6}), we have that the sequence $\{u_n\}$ is bounded in $BV(\Omega)$.  Thus, there exists $u^*\in BV(\Omega)$ such that ${u_n}$  converges to ${u^*}$  weakly in ${L^2}(\Omega )$  and strongly in  ${L^1}(\Omega )$. Since the operator $F$  is linear and bounded, we have that $F{u_n}$ also converges to $F{u^*}$  weakly in ${L^2}(\Omega )$  and strongly in  ${L^1}(\Omega )$. Then a standard process can show that ${u^*}$  is a minimizer of  $E(u)$.
\end{proof}
\begin{Theorem}
	If there exists a $F_{k}$ such that $F_k$ is invertible and $F_k{\rm{div}} \eta={\rm{div}} F_k\eta$ for any $\eta\in C_0^1(\Omega,R^{2})$, $\lambda_{\min}=\mathop{\min}\limits_{i=1,2,...m}\lambda_i> 0$ and  $A\otimes 1=\int_\Omega A(x)dx\ne 0$, then Problem (\ref{2}) admits a solution.
\end{Theorem}
\begin{proof}
	Let $u_n$ be a minimizing sequence of Problem (\ref{2}), then there exists a constant $M>0$ such that $E(u_n)\le M,$ i.e.
	$${\left| {\vec \lambda Fu} \right|_{VTV}} + \frac{1}{2}{\left\|Au - f\right\|_{L^2(\Omega,R)}^2}\le M.$$
	Therefore, 
	\begin{equation*}
	{\left| {\vec \lambda Fu} \right|_{VTV}}=\sum_{i}^{m}|\lambda_iF_iu|_{TV}\le M
	\end{equation*}
	and so
		\begin{equation*}
	|F_ku|_{TV}\le \frac{M}{\lambda_k}
	\end{equation*}
	
	Since $F_k$ is invertible and $F_k{\rm{div}} \eta={\rm{div}} F_k\eta$ for any $\eta\in C_0^1(\Omega,R^{2})$, like the proof in Lemma 2, we can obtain
	$$|u|_{TV}\le\left\|F_k^{-1}\right\||F_ku|_{TV}\le\frac{M}{\lambda_k}\left\|F_k^{-1}\right\|$$
	
	The left part of the proof is the same as that in Theorem 1 and so is omitted.	
\end{proof}
\begin{Remark}
To assure the existence of the solution, we can intentionally design a $F_k$ such that $F_k$ satisfies the conditions in Theorem 2. The simplest $F_k$ satisfying the conditions is $F_k=I$.
\end{Remark}

In the following, we give the definition of the feature operator $F$. In  signal processing, a filter of convolution is usually used to extract some feature of a signal. In this paper, we also use the convolution to define $F_i$.

Let $\{K_i(x)\}$ be a family of functions with a compact support $E=[-r,r]\times [-r,r]$ such that $K_i(x)=0$ for $x\in R^2-E$ and $K_i(x)\in L^{\infty}(R^2,R)$. For any $f\in  L^{2}(\Omega,R)$, we define $F_if$ as
\begin{equation}\label{7}
F_if(x)=f\otimes K_i=\int_\Omega f(y)K_i(x-y)dy.
\end{equation}
Since 
\begin{align*}
|\int_\Omega f(y)K_i(x-y)dy|^2&\le\int_\Omega f^2(y)dy\int_\Omega K_i^2(x-y)dy\\
&\le\left\|K_i\right\|^2_{L^{\infty}(R^2,R)}\left\|f\right\|^2_{L^{2}(\Omega,R)}
\end{align*}
we have $F_if\in L^{2}(\Omega,R)$.

By the Fubini's theorem,  we have
\begin{equation}\label{8}
F_i^*f(x)=f\otimes \hat K_i=\int_\Omega f(y)K_i(y-x)dy,
\end{equation}
where $\hat K_i(x)=K_i(-x)$.

For any $\eta=(\eta_1,\eta_2)\in C_0^1(\Omega,R^{2})$, we have
\begin{align*}
F_i^*{{\rm{div}}} \eta&=\int_\Omega {\rm{div}}\eta(y)K_i(y-x)dy\\
&=-\int_\Omega \eta(y)\nabla_y (K_i(y-x))dy\\
&=\int_\Omega \eta(y)\nabla_x (K_i(y-x))dy\\
&=\int_\Omega \eta_1(y)\frac{{\partial {K_i}(y - x)}}{{\partial {x_1}}}+ \eta_2(y)\frac{{\partial {K_i}(y - x)}}{{\partial {x_2}}}dy\\
&={\rm{div}}(\int_\Omega \eta_1(y)K_i(y-x)dx,\int_\Omega \eta_2(y)K_i(y-x)dx)\\
&={\rm{div}}(F_i^*\eta_1,F_i^*\eta_2)\\
&={\rm{div}}F_i^*\eta
\end{align*}
where we assume $x=(x_1,x_2)$.
Combining the above equation with Lemma 1, we thus have 
$$F^*{\rm{div}}\eta={\rm{div}}F^*\eta.$$
\subsection{Algorithm}
In this section, we utilize the split Bregman algorithm \cite{Goldstein2009The} to solve problem (\ref{2}). Before describing the algorithm, we need to give the discrete version of Problem (\ref{2}).
 
Let $f$ be a  matrix, such as an image of size $s\times t$, $K_i$ a matrix convolution filter of size $(2r+1)\times (2r+1)$. Then the discrete convolution of $f$ and $K_i$ is defined by
\begin{align*}
F_if(k,l)=(f\otimes K_i)(k,l)
&=\sum_{p=k-r,q=l-r}^{p=k+r,q=l+r}f(k,l)K_i(k-p,k-q)\\
&=\sum_{p,q=-r}^{r}f(k-p,l-q)K_i(p,q)
\end{align*}
where we assume the center of $K_i$ is the  origin of coordinates and extend $f$ circularly. 
By equation (\ref{8}), we have 
\begin{align*}
F_i^*f(k,l)=(f\otimes \hat K_i)(k,l)
&=\sum_{p=k-r,q=l-r}^{p=k+r,q=l+r}f(k,l)K_i(k-p,l-q)\\
&=\sum_{p,q=-r}^{r}f(k-p,l-q)K_i(-p,-q)
\end{align*}

Let $g=(g_1,g_2,...,g_r)$, $g_i$ be a matrix of size $s\times t$. Define the p-norm of $g$ as
\begin{align*}
\left\|g\right\|_1&=\sum_{i=1}^r\left\|g_i\right\|_1\\
\left\|g\right\|_2&=\sqrt{\sum_{i=1}^r\left\|g_i\right\|_2^2}
\end{align*}
Then by equation (\ref{1}),  the discrete version of Problem (\ref{2}) is 

\begin{equation}\label{9}
\begin{aligned}
\mathop{\arg\min}\limits_{u}\left\{ 
{E(u)=\sum_{i=1}^m\lambda_i\left\|\nabla  F_iu\right\|_1+\frac{1}{2}{\left\|Au - f\right\|_2^2}}
   \right\}
\end{aligned}
\end{equation}

 Let  $d = \nabla (Fu)$, i.e. $d_i=\nabla (F_iu)$, ${d^0} = {b^0} = 0$, then Problem (\ref{9}) is equivalent to the following iterations:

for j=0,1,2,…
\begin{align}
&{u^{j + 1}} = \arg \min \frac{1}{2}\left\| {Au - f} \right\|_2^2 + \sum_{i=1}^m\frac{\gamma_i}{2}\left\| {\nabla (F_iu) - {d^j_i} + {b^j_i}} \right\|_2^2\label{10}\\
&{d^{j + 1}} =\arg \min  \sum_{i=1}^m(\lambda_i{\left\| d_i \right\|_1} + \frac{\gamma_i }{2}\left\| {d_i - \nabla ( F_i{u^{j + 1}}) - {b_i^j}} \right\|_2^2)\label{11}\\
&{b_i^{j + 1}} = {b_i^j} + (\nabla ( F_i{u^{j + 1}}) - {d_i^{j + 1}})\label{12}
\end{align}

The KKT condition for problem (\ref{10}) is 
\[\begin{array}{l}
{A^*}(Au - f) + \sum_{i=1}^m\gamma_i{F_i^*}{\nabla ^*} (\nabla F_iu - {d_i^j} + {b_i^j})
= 0
\end{array}\] Therefore, the solution for Problem (\ref{10}) is 
\begin{equation}\label{13}
{u^{j + 1}} = FFT^{ - 1}(\frac{{FFT(\sum_{i=1}^m\gamma_i{F_i^*}{\nabla ^*} ({d_i^j} - {b_i^j}) + {A^*}f)}}{{FFT({A^*}A +  \sum_{i=1}^m\gamma_i{F_i^*}{\nabla ^*} \nabla F_i)}})
\end{equation}

The solution for problem (\ref{11}) can be explicitly solved:
\begin{equation}\label{14}
{d_i^{j + 1}} = TH(\nabla ( F_i{u^{j + 1}}) + {b_i^j},\frac{\lambda_i}{\gamma_i} )
\end{equation}
where  $TH(x,T) = {\mathop{\rm sgn}} (x)\max (\left| x \right| - T,0)$.

At last, the algorithm is summarized in Algorithm \ref{alg:buildtree}.
\begin{algorithm}
	\caption{}
	\label{alg:buildtree}
\begin{algorithmic}
	\STATE{Set the initial values  ${d^0} = {b^0} = 0$,  $k = 0$, the maximal iteration $M > 0$  and the tolerant error  $tol > 0$.}
	\WHILE{$j < M$ and $\frac{{\left\| {{u^{j + 1}} - {u^j}} \right\|}}{{\left\| {{u^k}} \right\|}} > tol$}
	\STATE{Update  ${u^{j + 1}}$,  ${d^{j + 1}}$ and ${b^{j+ 1}}$  by equations (\ref{13}), (\ref{14}) and (\ref{12}), respectively.}
    \ENDWHILE
    \RETURN $u^{j+1}$
\end{algorithmic}
\end{algorithm}
\begin{Remark}
	To implement Algorithm 1, if we don't consider the existence of solution, we only need $F_i^*$ but don't need $F^{-1}$ or $F_i^{-1}$.
\end{Remark}
The proof of the convergence of the split Bregman algorithm was given in \cite{doi:10.1137/090753504}. The  convergence of our algorithm can be proved accordingly and so is not listed here.

\section{Experimental results}
In this section, we conduct some numerical experiments on image denoising and image deblurring using Algorithm \ref{alg:buildtree}. To implement Algorithm 1, we need to construct  $K_i$ explicitly. In our experiments, the piecewise linear B-spline wavelet frame is used for constructing $K_i$.
The filter banks  of the B-spline wavelet frame are
$$[{h_1}( - 1),{h_1}(0),{h_1}(1)] = \frac{1}{4}[1,2,1]$$  $$[{h_2}( - 1),{h_2}(0),{h_2}(1)] = \frac{{\sqrt 2 }}{4}[ 1,0,-1]$$ $$[{h_3}( - 1),{h_3}(0),{h_3}(1)] = \frac{1}{4}[ - 1,2, - 1].$$
Then the $K_i$ used in our experiment are
$K_1=h_1^\dagger h_1$ and 
\begin{equation}\label{A2}
K_{3(i-1)+j}=h_i^\dagger h_j, \;\;\;i,j=1,2,3.
\end{equation}
Clearly, the feature image generated by $K_1$ have more even region than those by the other $K_i$. So in the experiments, we set a larger $\lambda_1$ and $\gamma_1$ to impose a more smooth effect on it.

By the unitary extension principle (UEP) in \cite{Amos},  we have that 
\begin{equation}\label{A3}
F^*F=I.
\end{equation}
 Thus, $F_L^{-1}=F^*$.
By Lemma 1 and (\ref{A3}), we also have  
\begin{equation}\label{15}
\sum_{i=1}^mF_i^*F_i=1
\end{equation}

If we set the parameter vector $\gamma$  as $\gamma_1>\gamma_2=\gamma_3=...=\gamma_m$, then by equation (\ref{15}), we have
\begin{align*}
\sum_{i=1}^m\gamma_i{F_i^*}{\nabla ^*} \nabla F_i&=(\gamma_1-\gamma_2)F_1^*F_1\nabla^*\nabla+\gamma_2{\nabla ^*} \nabla\sum_{i=1}^m{F_i^*} F_i\\
&=(\gamma_1-\gamma_2)F_1^*F_1\nabla^*\nabla+\gamma_2{\nabla ^*} \nabla,
\end{align*}
where we can interchange these operators since they are all convolutions. Thus equation (\ref{13}) in Algorithm 1 can be reduced to
\begin{equation}\label{16}
{u^{j + 1}} = FF{T^{ - 1}}(\frac{{FFT(\sum_{i=1}^m\lambda_i{F_i^*}{\nabla ^*} ({d_i^j} - {b_i^j}) + {A^*}f)}}{{FFT({A^*}A +  (\gamma_1-\gamma_2)F_1^*F_1\nabla^*\nabla+\gamma_2{\nabla ^*} \nabla)}})
\end{equation}

\begin{figure}[htbp]	
	\centering
	\includegraphics[width=0.3\textwidth]{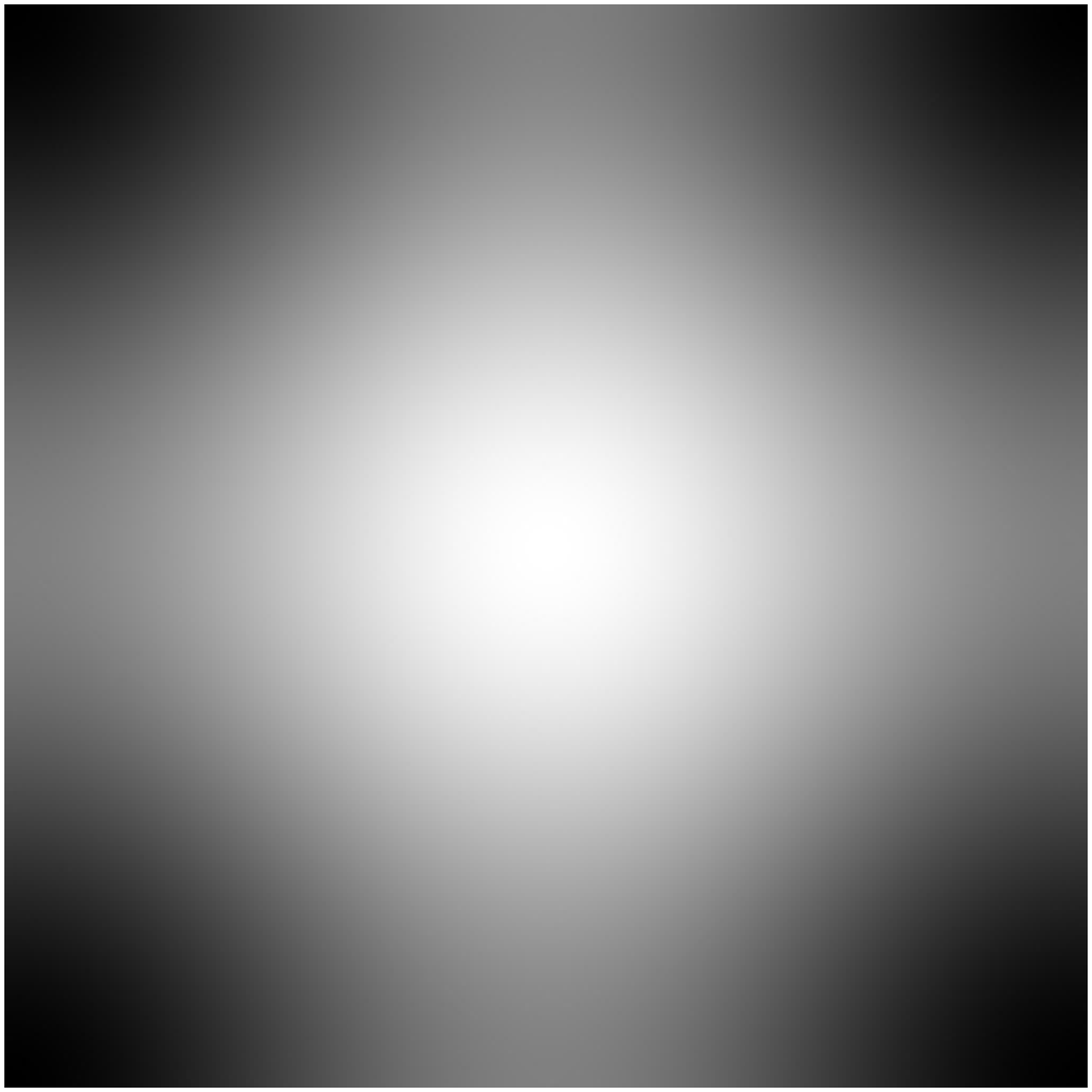}
	\includegraphics[width=0.3\textwidth]{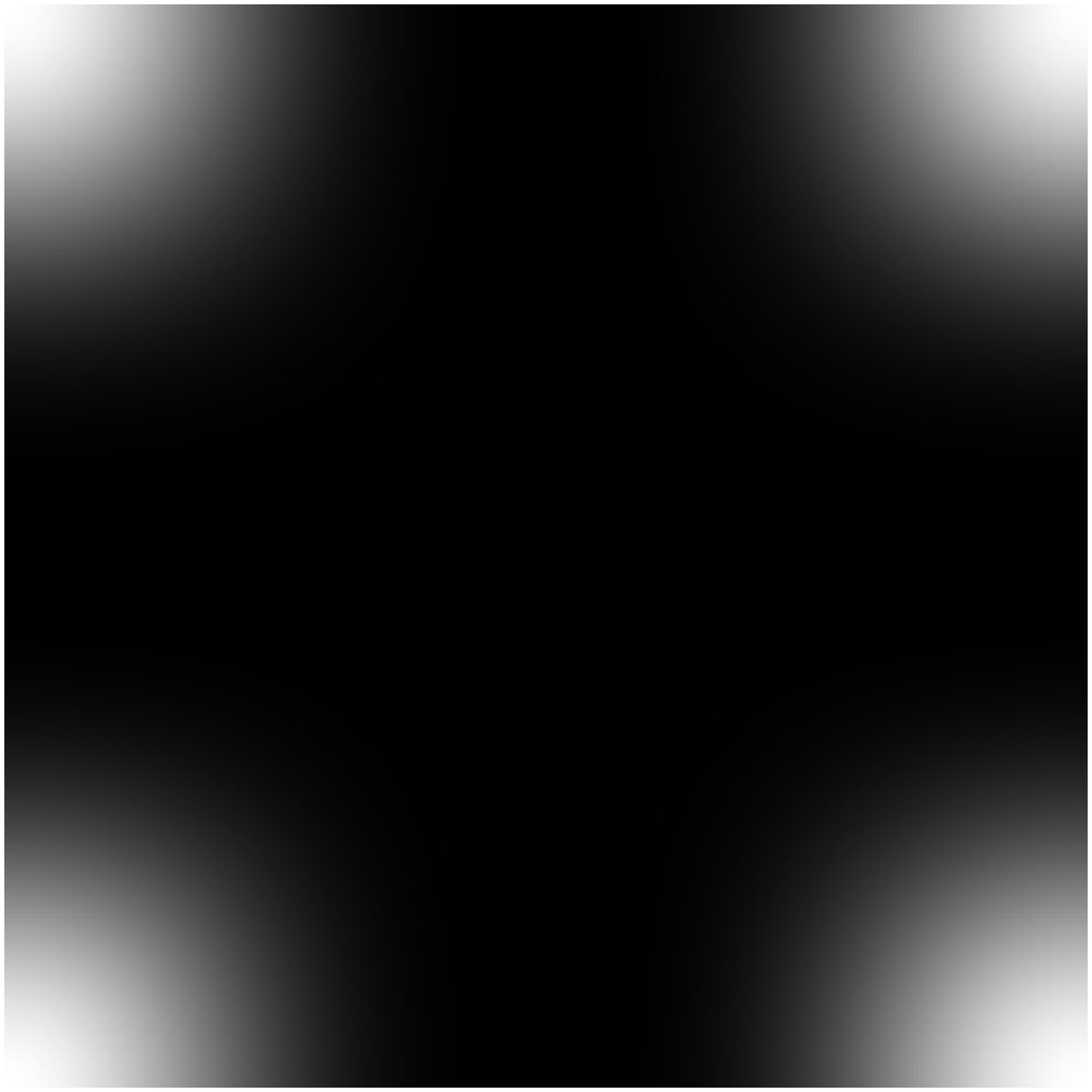}
	\caption{From left to right: FFTs of ${\nabla ^*}\nabla$ and $F_1^*F_1$.}
	\label{F:FFT}
\end{figure}

In Figure \ref{F:FFT}, we show the Fast Fourier transform (FFT) of ${\nabla ^*}\nabla$ and
$F_1^*F_1$, from which we can see that  the effect of the FFT of $F_1^*F_1$ is opposite to that of ${\nabla ^*}\nabla$. Since the effect of  ${\nabla ^*}\nabla$ is smoothing the image, the term $F_1^*F_1$ may slow down the convergence speed and cause artifacts in the denoised images. Thus in the numerical simulations, we neglect the term $(\gamma_1-\gamma_2)F_1^*F_1\nabla^*\nabla$  for computational efficiency. Equation (\ref{16}) is thus changed into
\begin{equation}\label{17}
{u^{j + 1}} = FF{T^{ - 1}}(\frac{{FFT(\gamma\sum_{i=1}^m\lambda_i{F_i^*}{\nabla ^*} ({d_i^j} - {b_i^j}) + {A^*}f)}}{{FFT({A^*}A + {\gamma _1}{\nabla ^*}\nabla )}})
\end{equation}
Equation (\ref{17}) can be seen as an analogy of the original split Bregman algorithm who solves $u^{j+1}$ by
$${u^{j + 1}} = FF{T^{ - 1}}(\frac{{FFT(\lambda {\nabla ^*} ({d^j} - {b^j}) + {A^*}f)}}{{FFT({A^*}A + {\gamma _1}{\nabla ^*}\nabla )}})$$
By experiments, we find that using equation (\ref{17}) can indeed get better results than using equation (\ref{13}). To demonstrate it, we give  both results of using equations (\ref{13}) and  (\ref{17})  in the compared experiments.

In Section \ref{se3.1}, we compare the experimental results of our model with the related methods (such as split Bregman algorithm (SB) \cite{Goldstein2009The}, dual tree complex wavelet transform (DTCWT) \cite{Selesnick}, local contextual hidden markov model (LHMM) \cite{Fan} and the  ${l_o}$  minimization in wavelet frame based model ( ${l_o}$-WF) \cite{Bin}. In Section \ref{se3.2}, we perform the comparison on the simulation of image debluring with our model,   the  ${l_o}$-WF model and the EDWF\cite{Jae} model. For measuring the image quality quantitatively, we use the index peak signal to noise ratio (PSNR) defined by
$$PSNR: = 10{\log _{10}}(\frac{{255*255*N}}{{\left\| {ref - \tilde u} \right\|_2^2}})$$
where $ref$  and $\tilde u$  are the true image and recovered image, respectively.
 
\subsection{Image denoising}
\label{se3.1}
In this subsection, we compare the performances of our model with  SB, DTCWT, LHMM and  ${l_o}$-WF on image denoising. Six images (all of size 256*256) shown in Figure \ref{F:orig} are  tested. The noisy images are generated by the MATLAB command `imnoise' with `type=Guassian', $m = 0$  and variance $v = 0.01$. 

\begin{figure}[htbp]	
	\centering
	\includegraphics[width=0.3\textwidth]{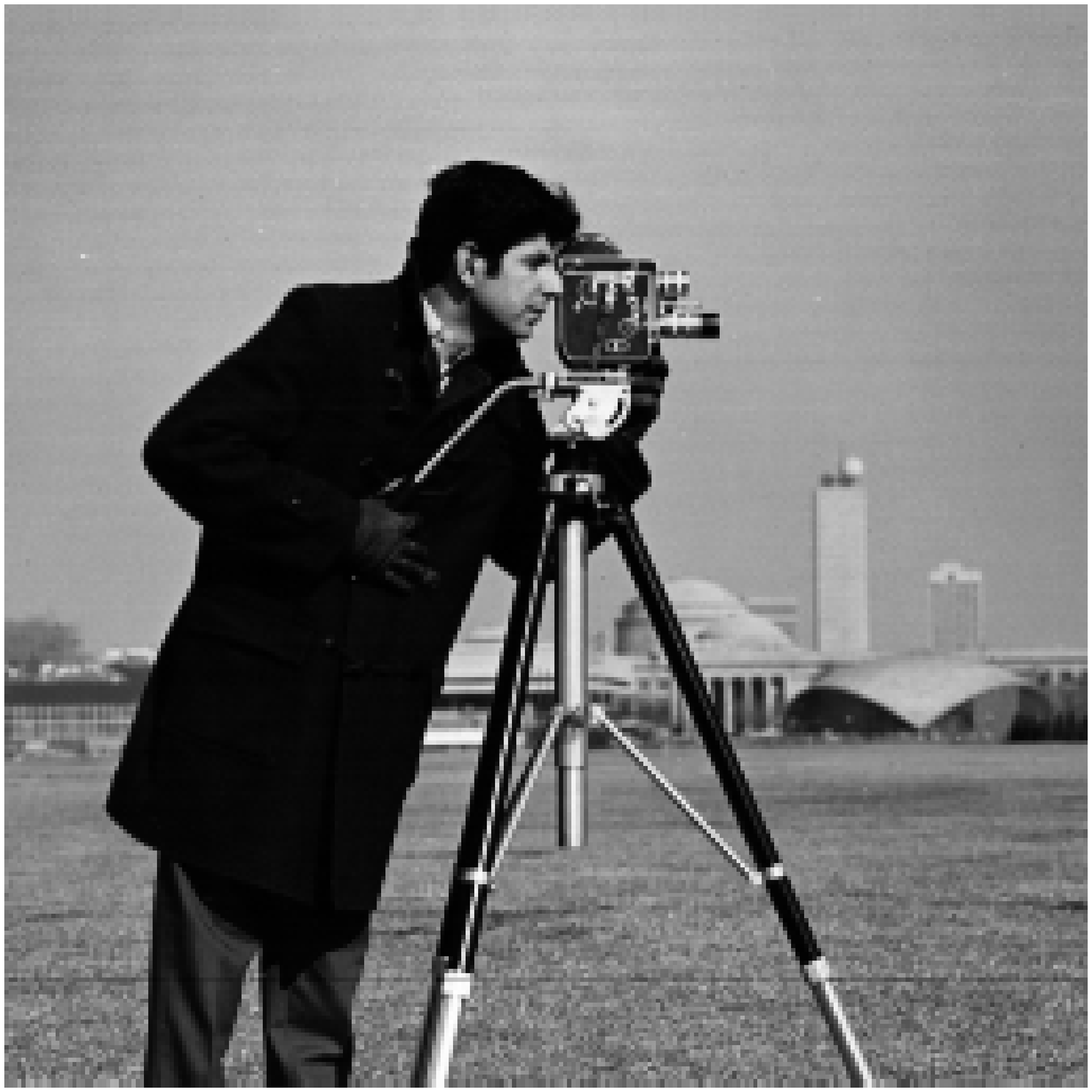}
	\includegraphics[width=0.3\textwidth]{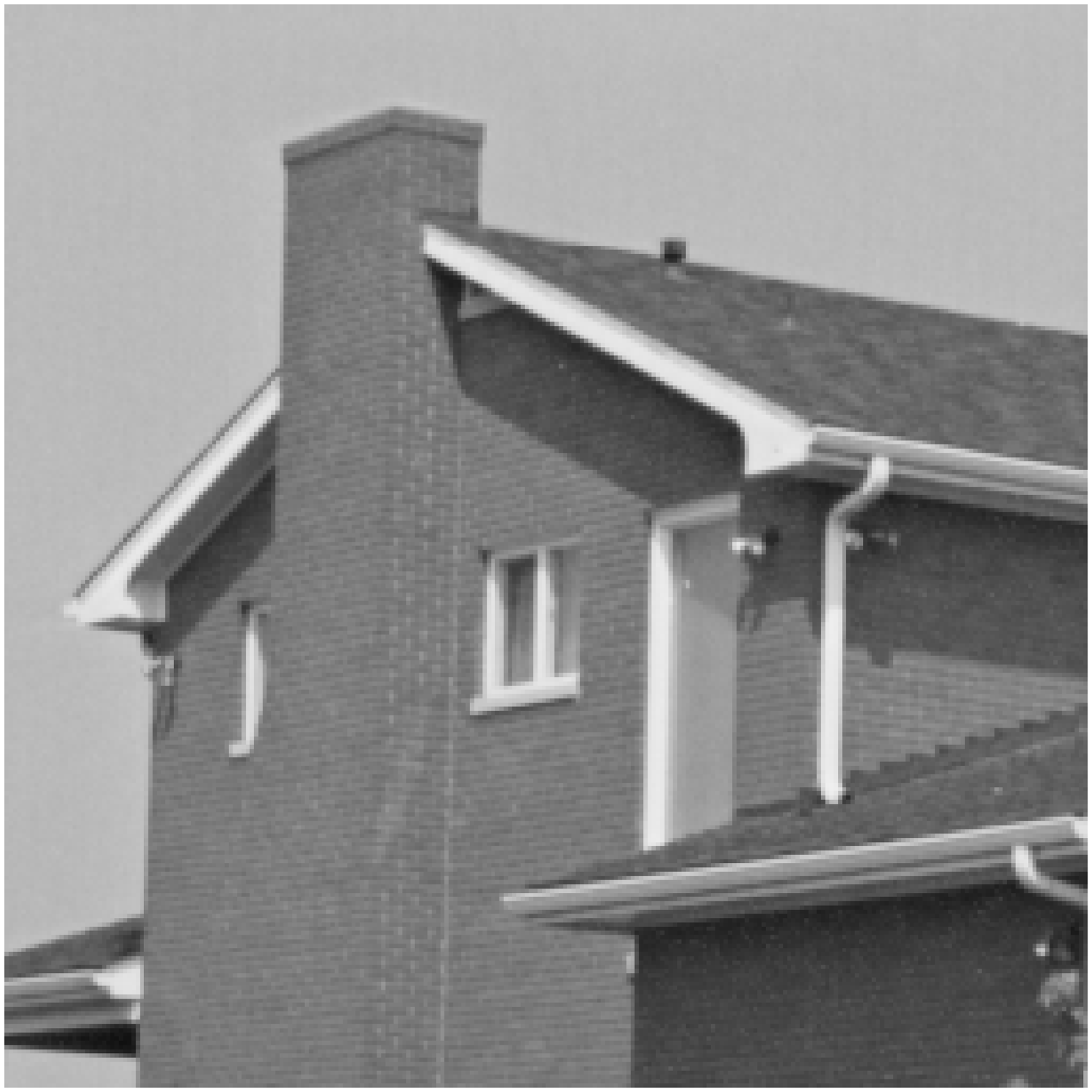}
	\includegraphics[width=0.3\textwidth]{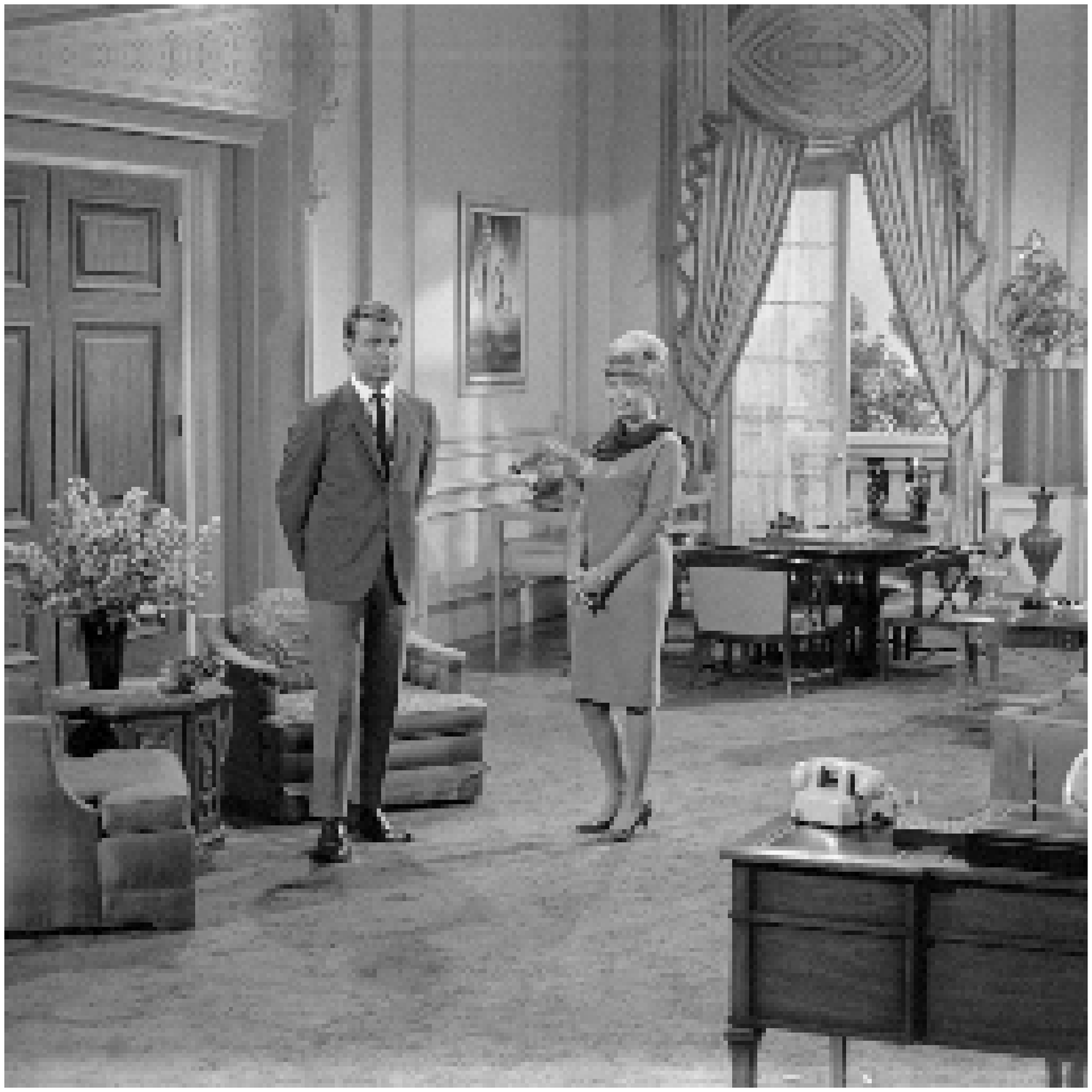}\\
	\includegraphics[width=0.3\textwidth]{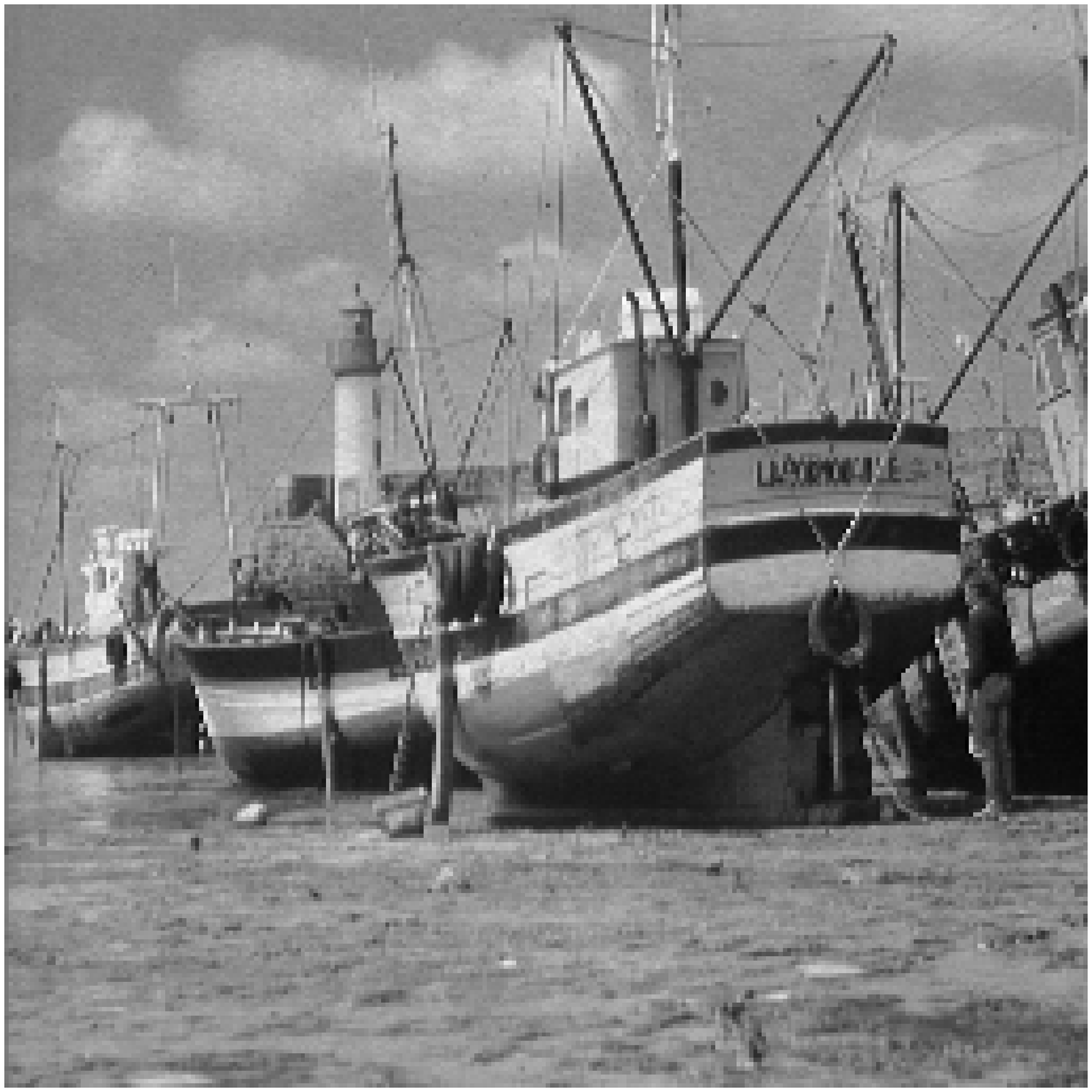}
	\includegraphics[width=0.3\textwidth]{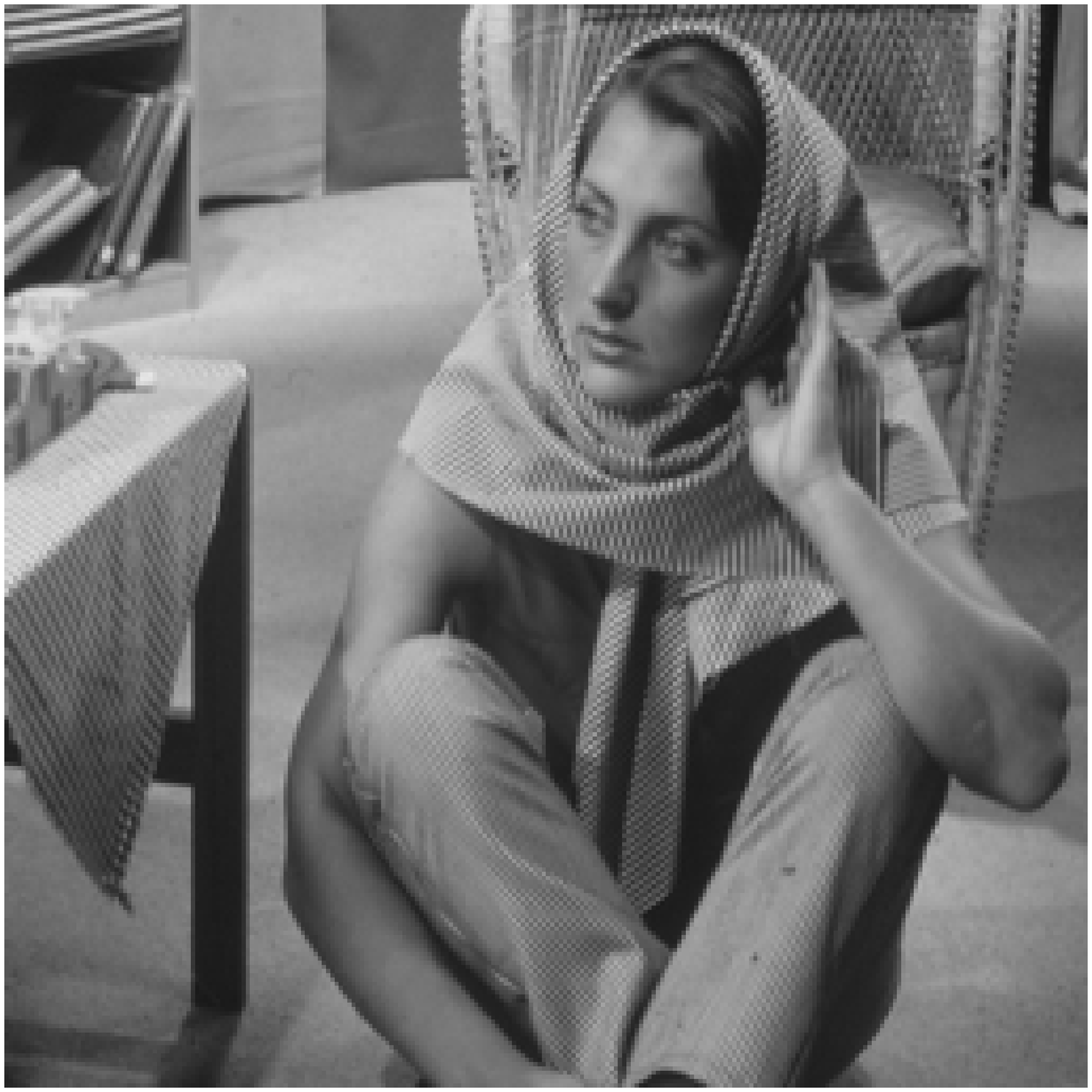}
	\includegraphics[width=0.3\textwidth]{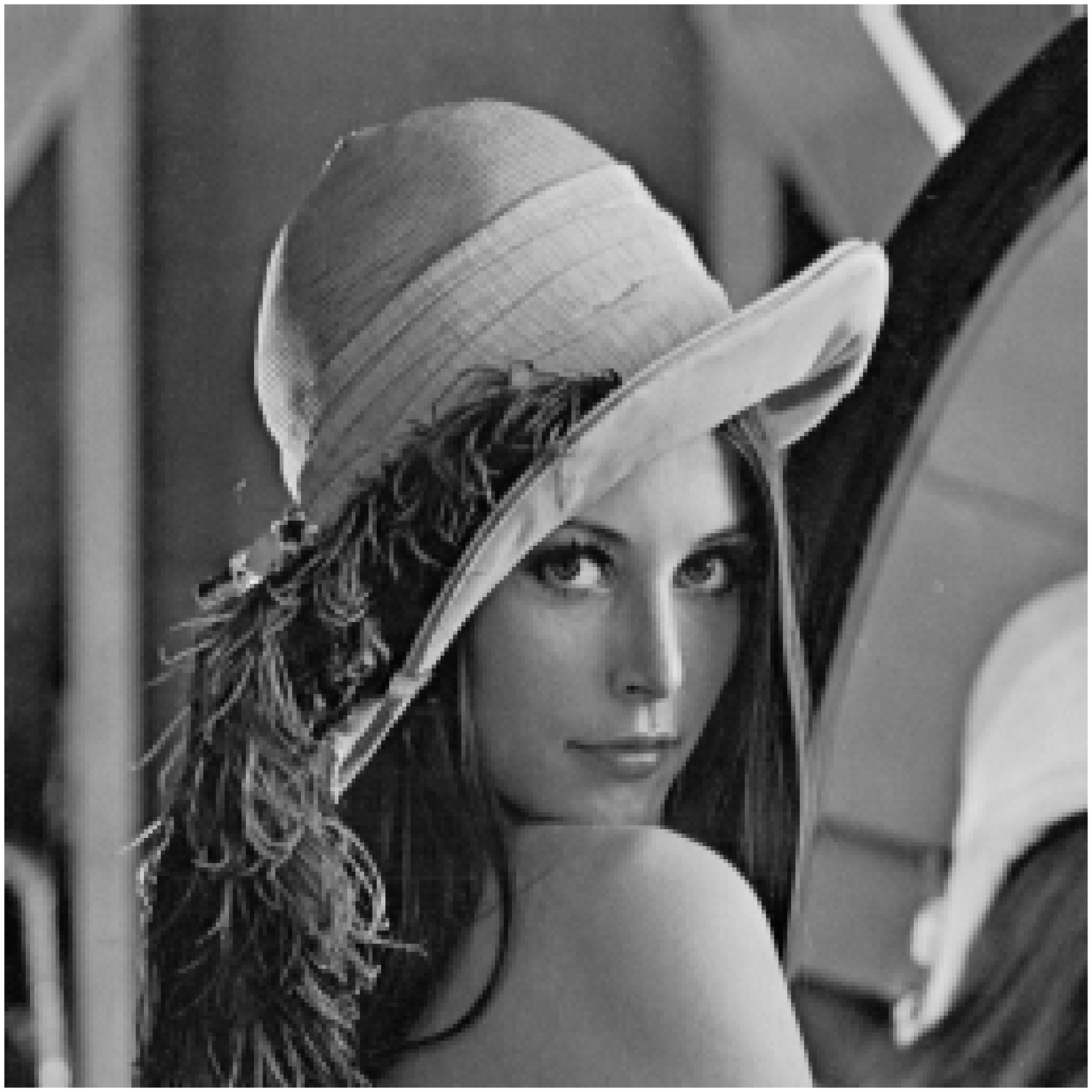}
	\caption{The  test images.}
	\label{F:orig}
\end{figure}

In the implementation of image denoising, the parameters of our algorithm using equation (\ref{13}) are set as:  ${\lambda _1} =\lambda _2==...=\lambda _9 = 0.2$,  ${\gamma _1} = 8$,  ${\gamma _2}=\gamma _3=...=\gamma _9 = 4$,  $tol = 1e - 4$; the parameters of our algorithm using equation (\ref{17}) are set as:  ${\lambda _1} = 2$,  ${\lambda _2}=\lambda _3=...=\lambda _9 = 1.5$,  ${\gamma _1} = 12$,  ${\gamma _2}=\gamma _3=...=\gamma _9 = 4.5$,  $tol = 5e - 4$. For  SB algorithm, the parameters are set as:  $\lambda  = 19$,  $gamma  = 1$,  $tol = 5e - 4$. For  DTCWT, the decomposition level  $L = 4$, threshold  $T = 20$. For ${l_o}$-WF, the piecewise linear B-spline wavelet frame is also used and the decomposition level  $L = 1$,  $\lambda  = 350$,  $\mu  = 1$,  $\gamma  = 0.2$,  $tol = 5e - 4$. The LHMM model has no free parameters. For fair comparison, all the paramters are uniformly set for  the six tested images and we try our best to tune them to get the highest average PSNR.

In Figure \ref{F；analysis}, we show the feature (coefficient) images of a noisy image and its denoised version by our method using equation (\ref{17}). From Figure \ref{F；analysis}, we can see that the noises in every feature images are all removed.

\begin{figure}[htbp]
	\centering
	\includegraphics[width=0.19\textwidth]{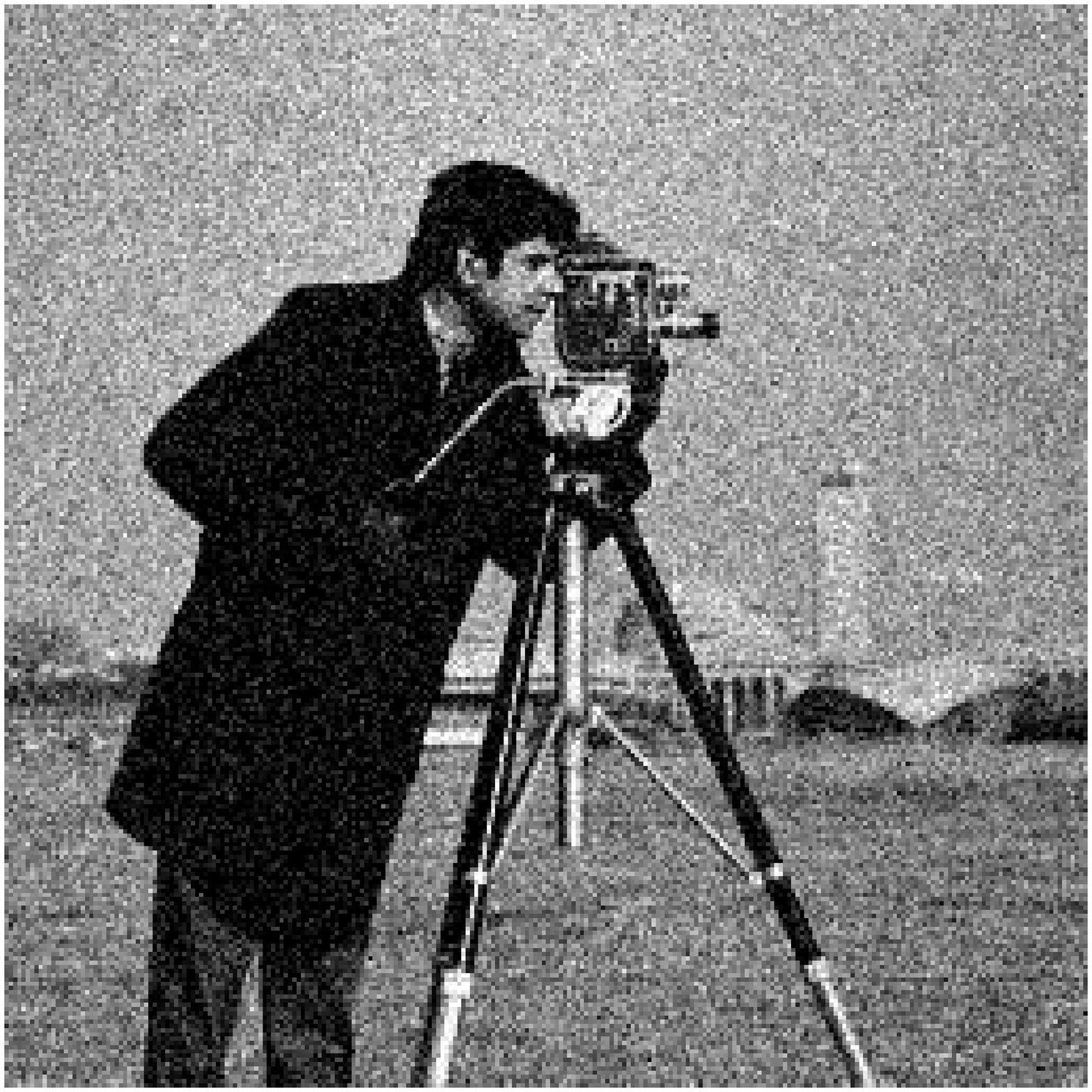}
	\includegraphics[width=0.19\textwidth]{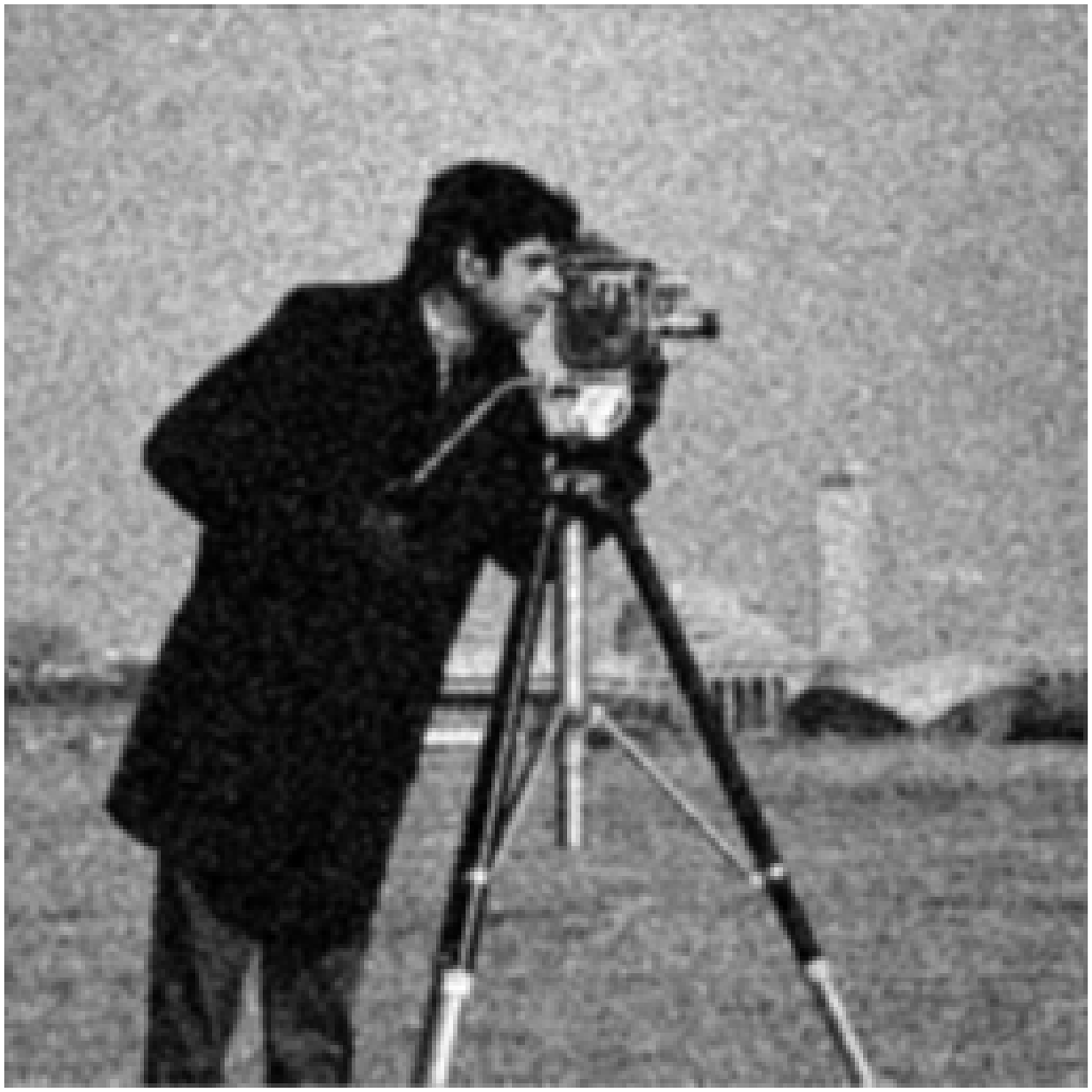}
	\includegraphics[width=0.19\textwidth]{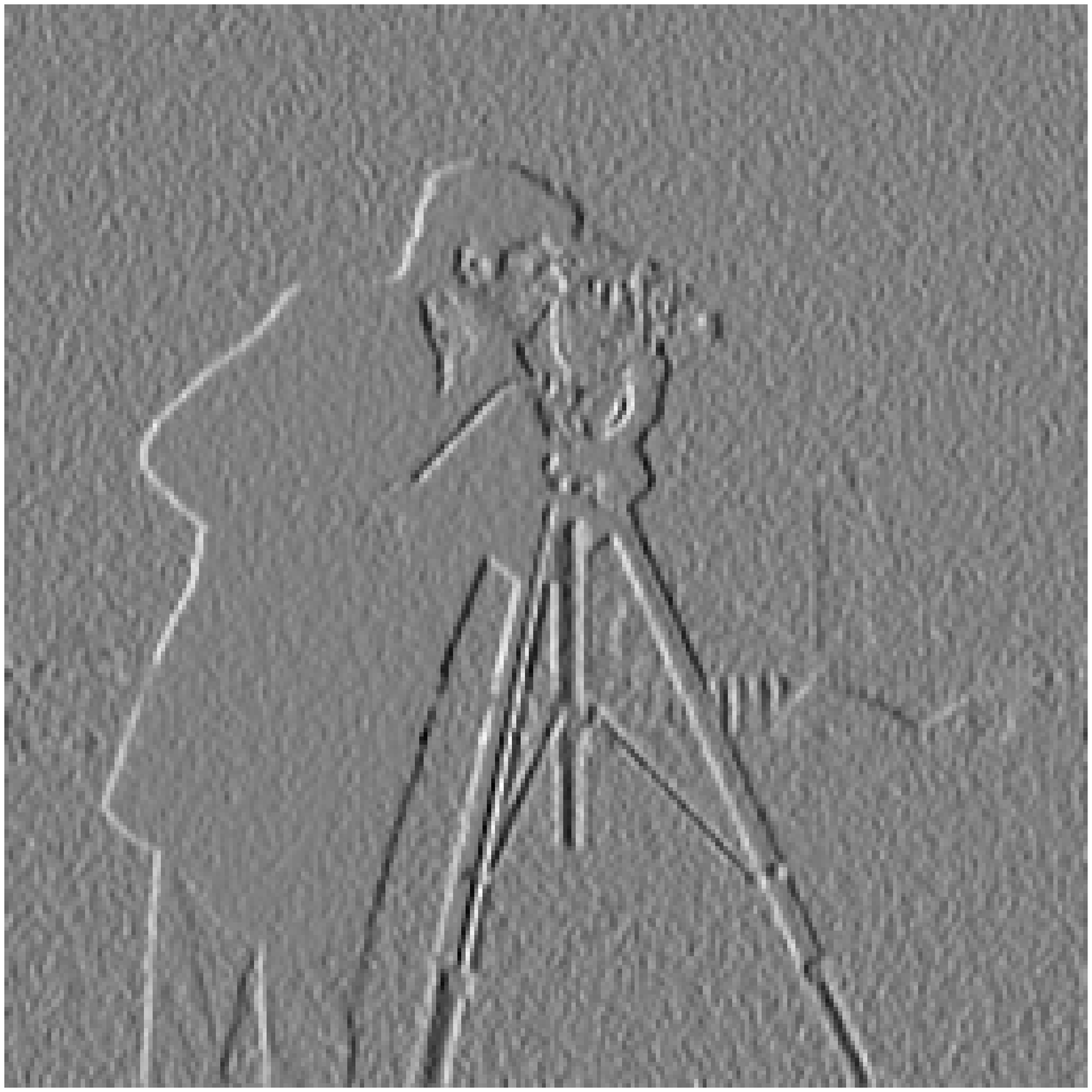}
	\includegraphics[width=0.19\textwidth]{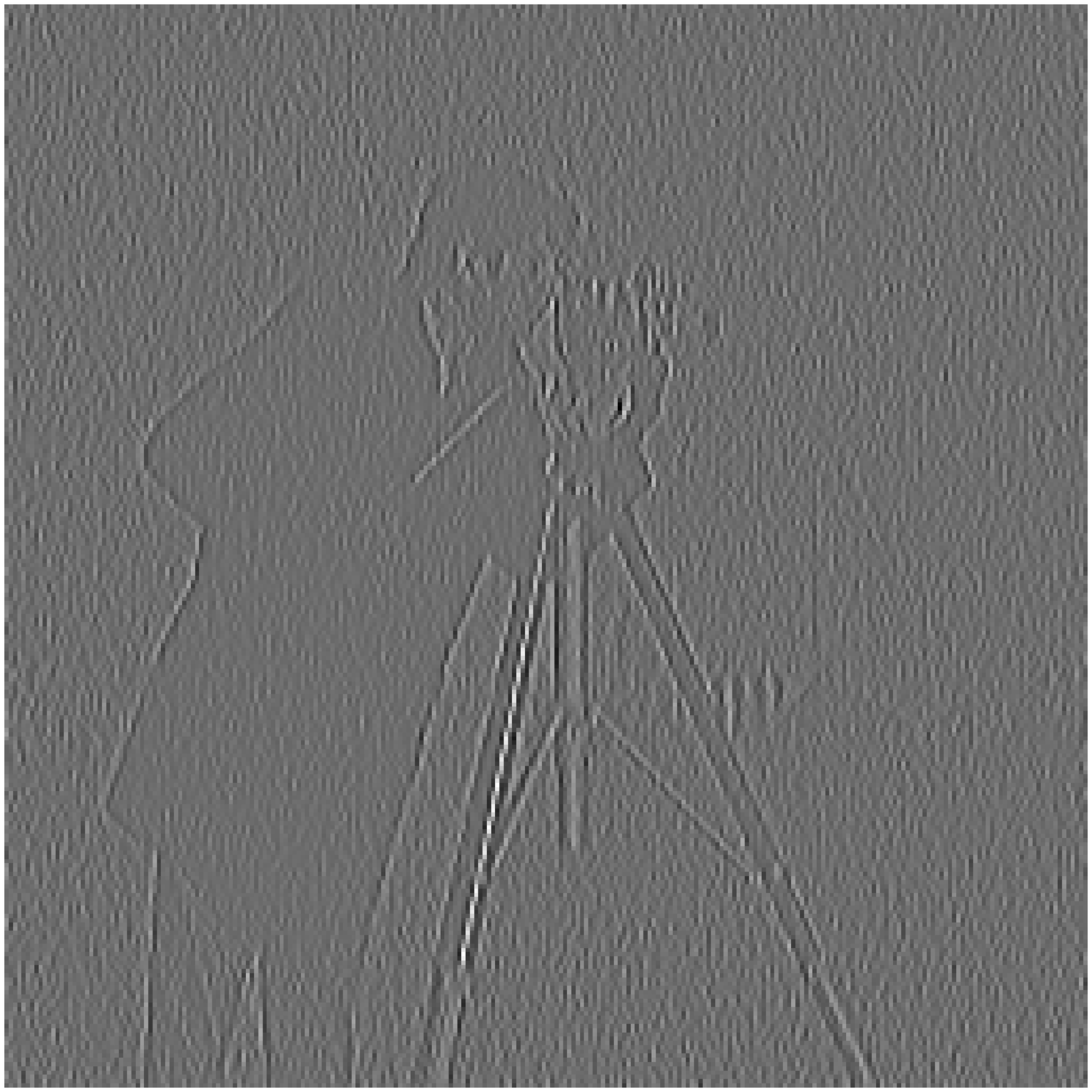}
	\includegraphics[width=0.19\textwidth]{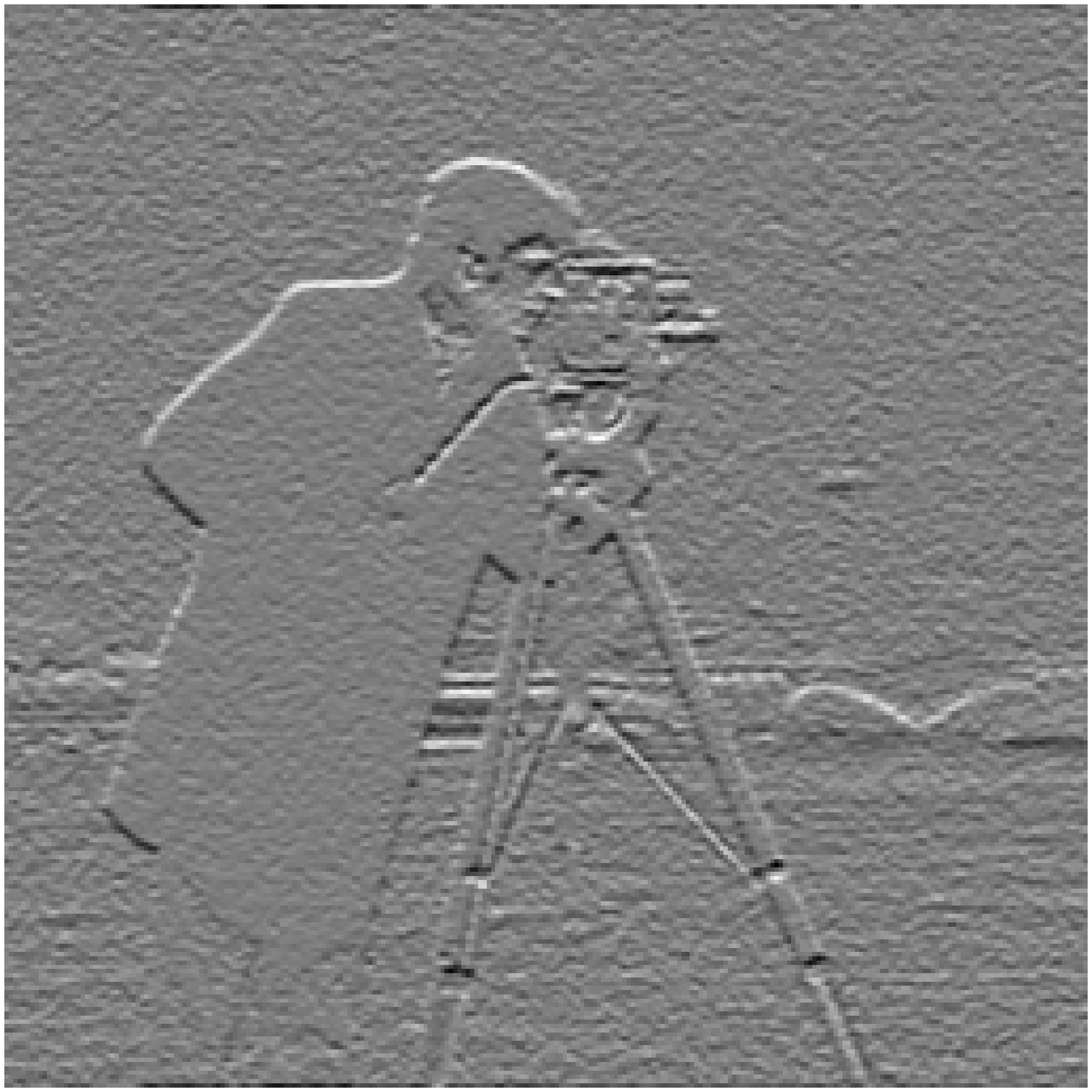}\\
	\includegraphics[width=0.19\textwidth]{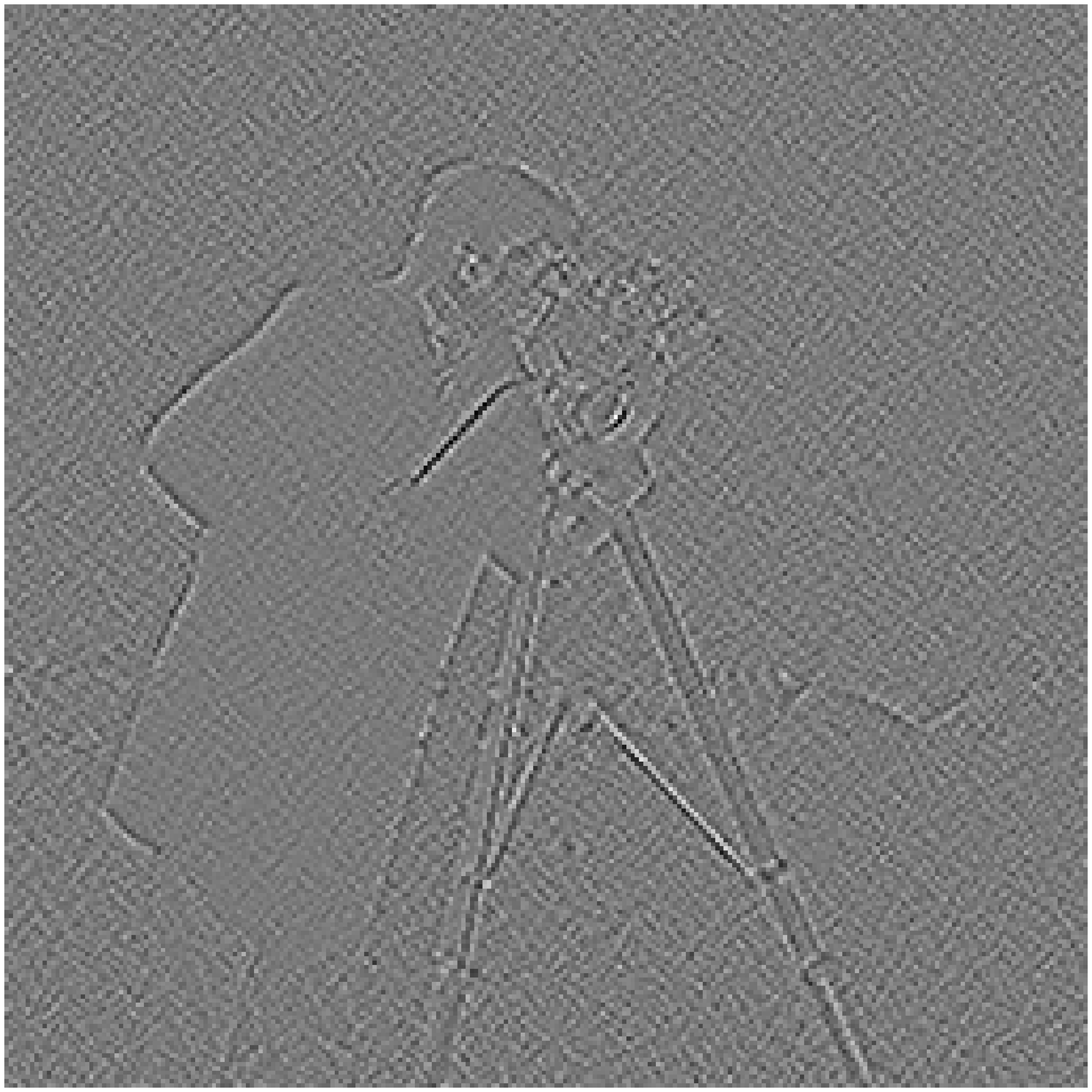}
	\includegraphics[width=0.19\textwidth]{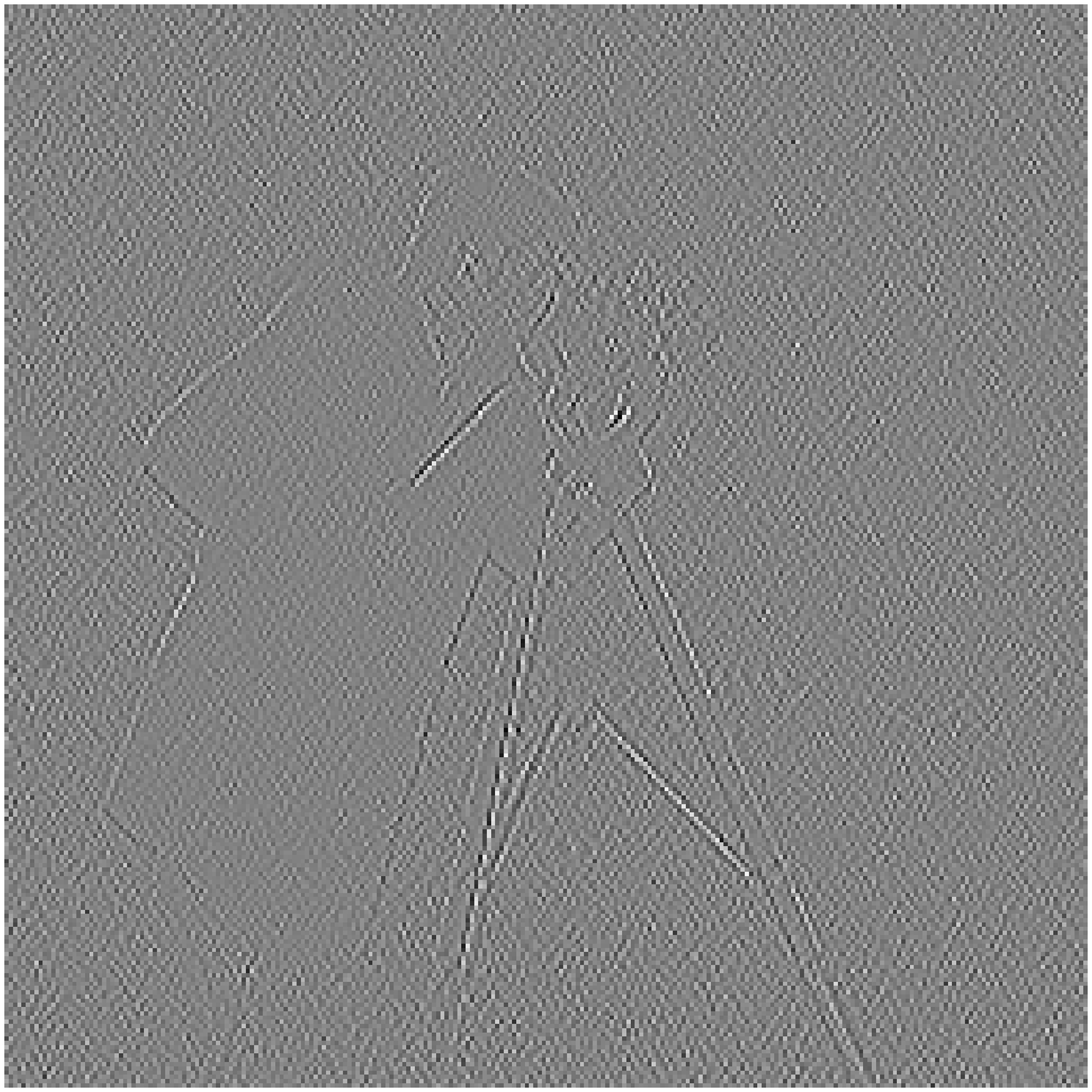}
	\includegraphics[width=0.19\textwidth]{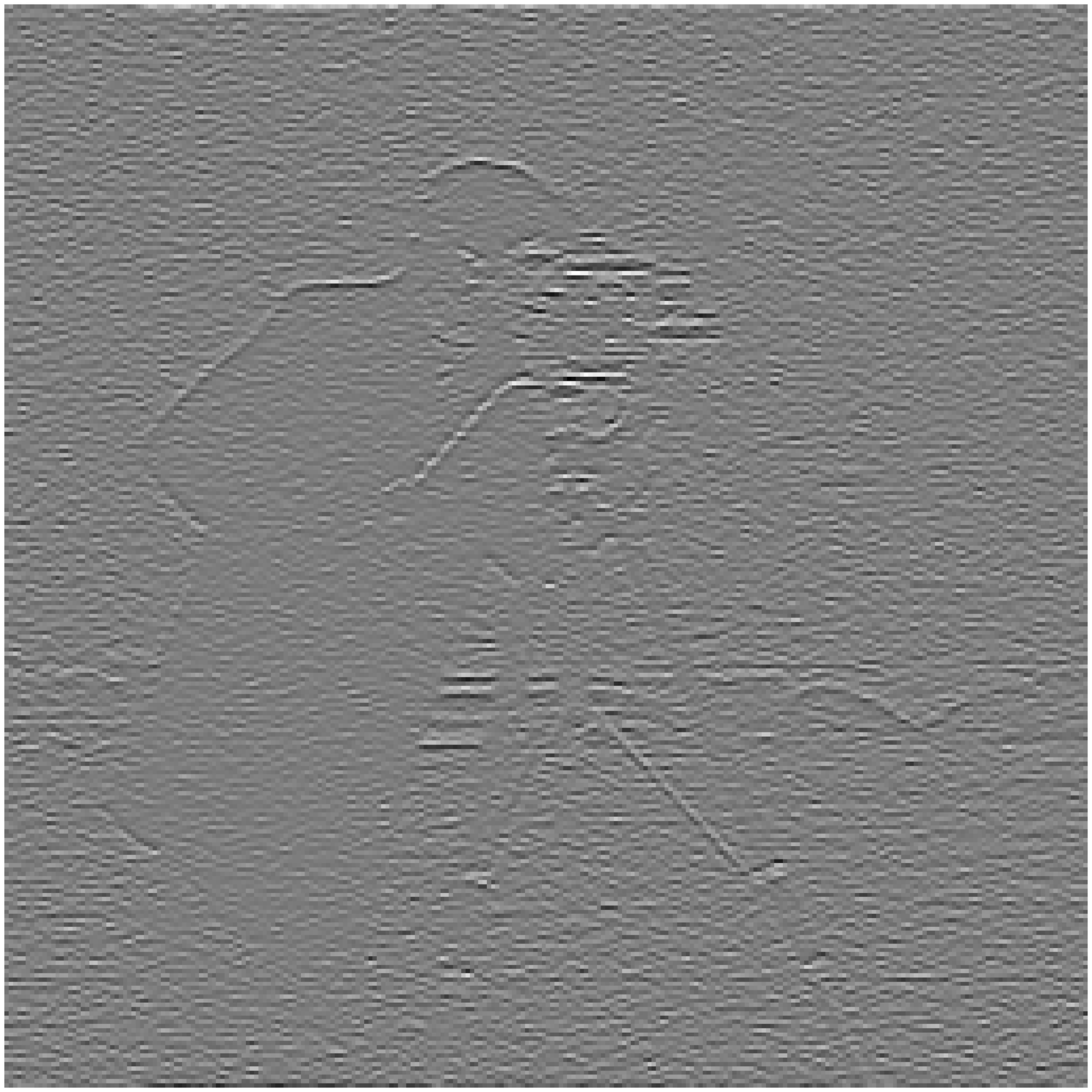}
	\includegraphics[width=0.19\textwidth]{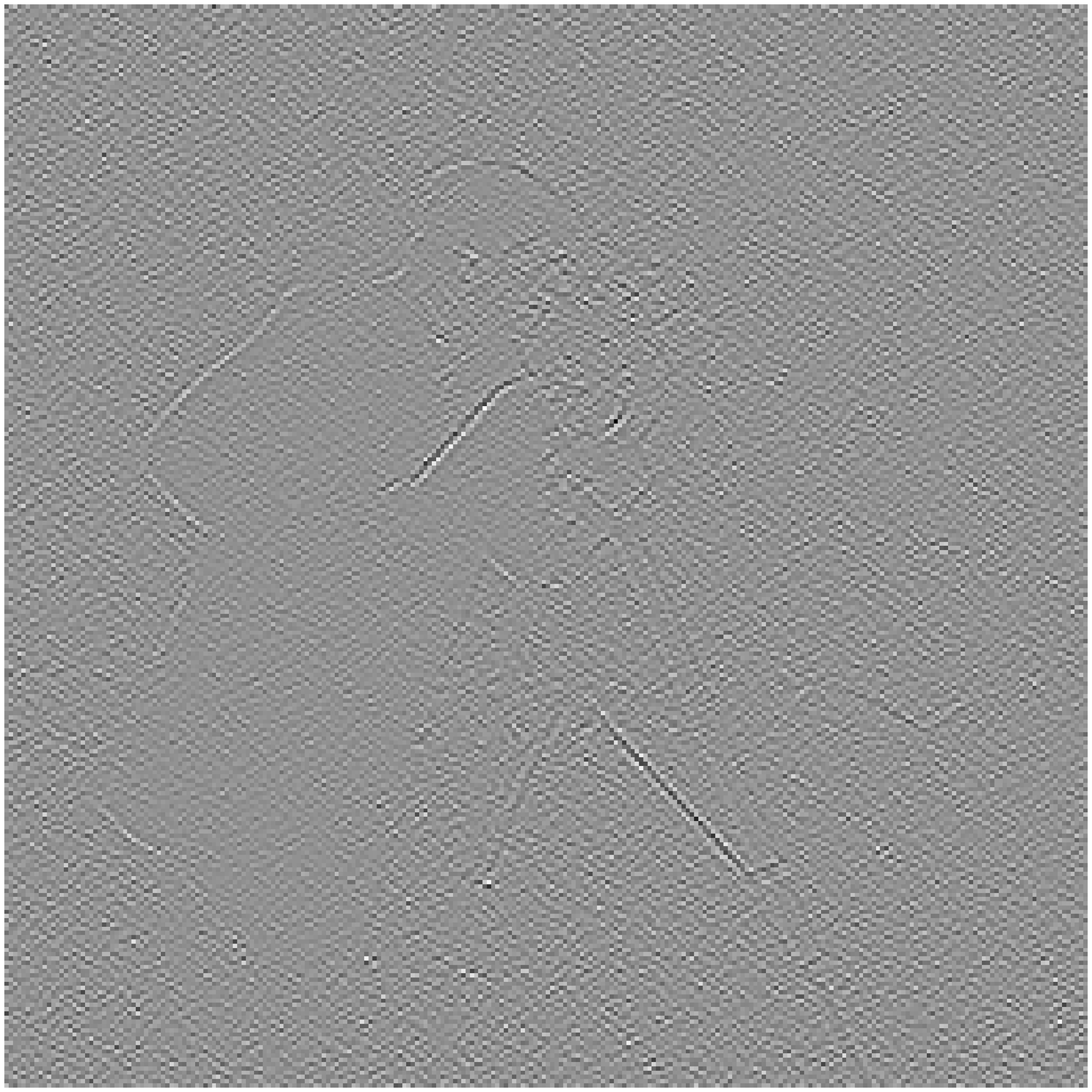}
	\includegraphics[width=0.19\textwidth]{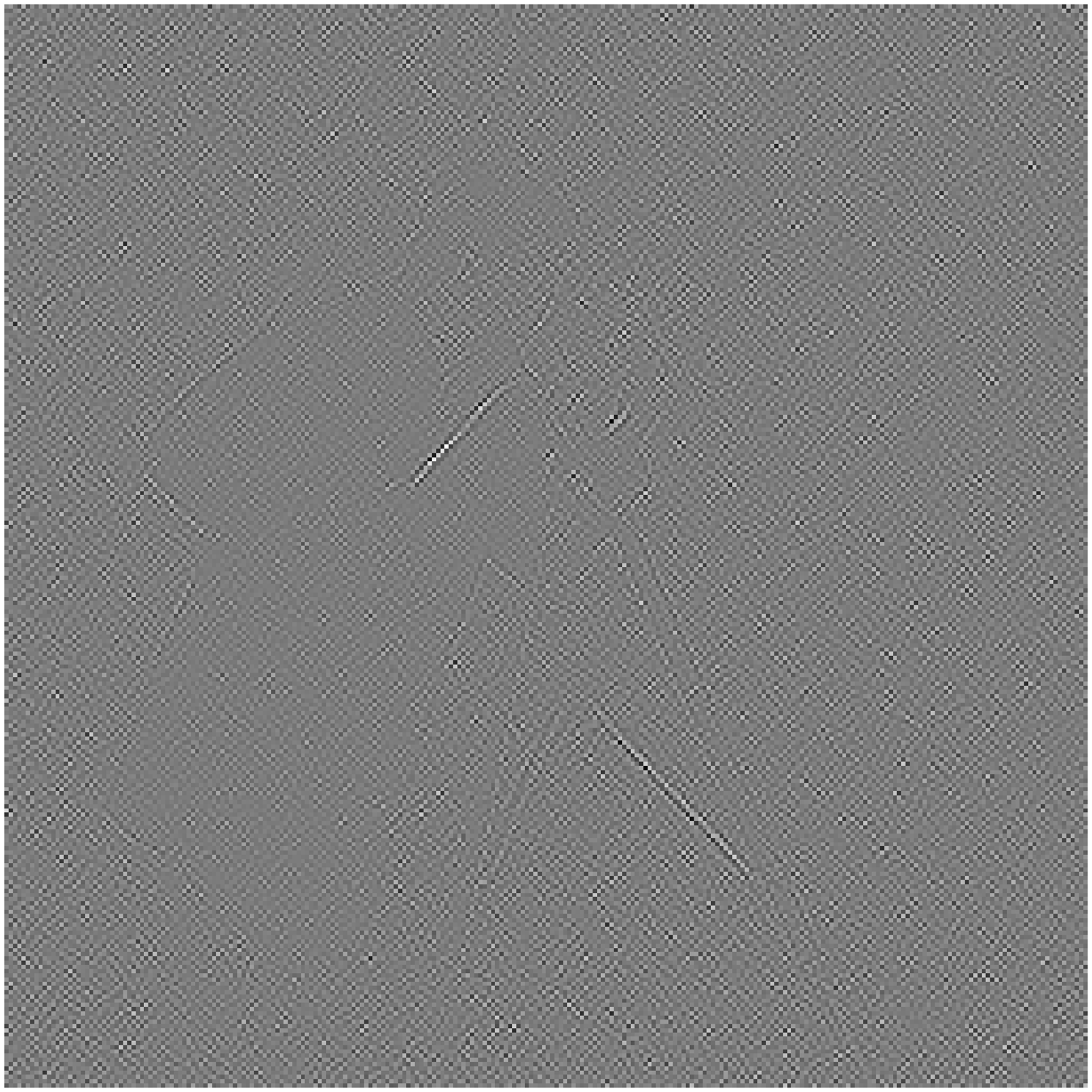}\\
	\includegraphics[width=0.19\textwidth]{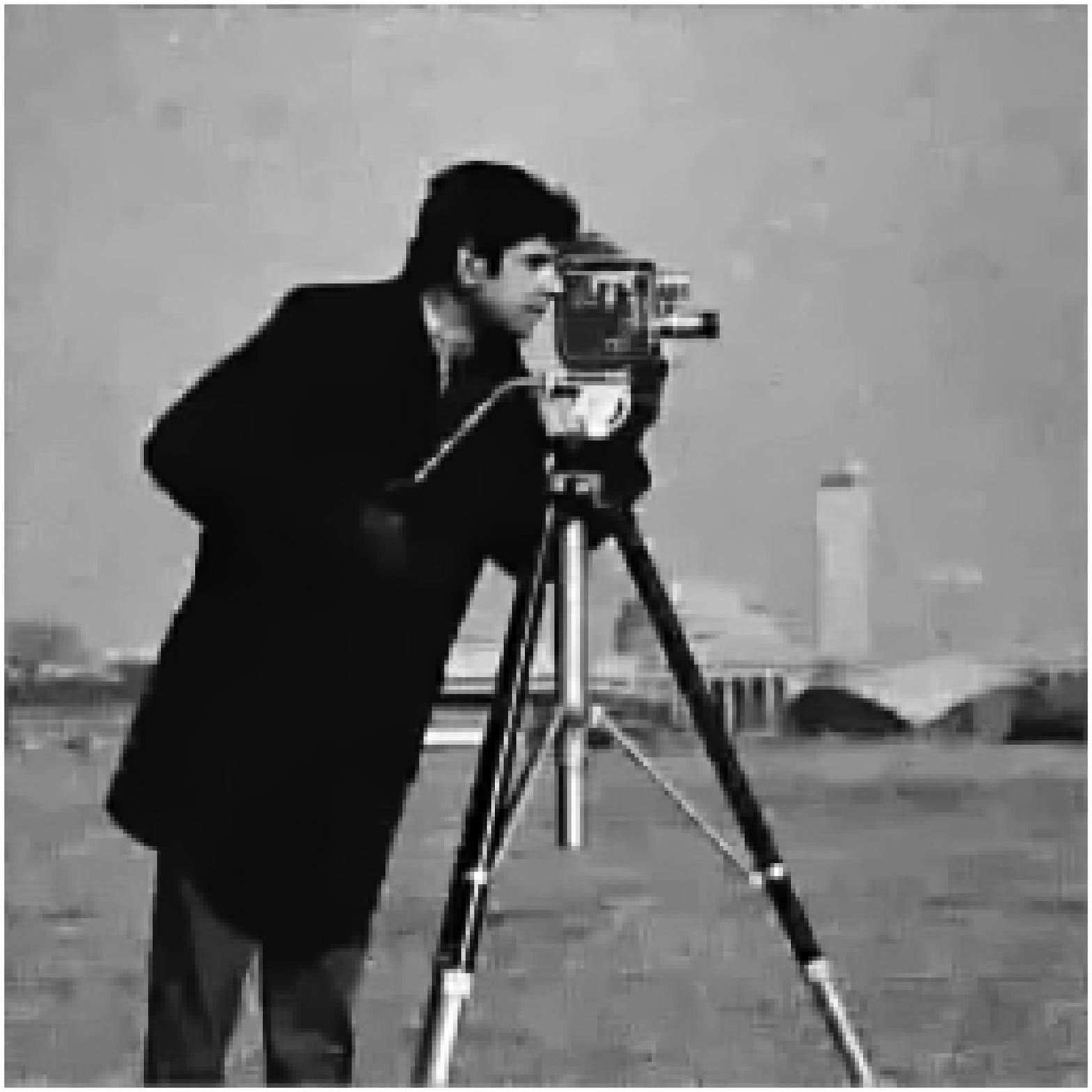}
	\includegraphics[width=0.19\textwidth]{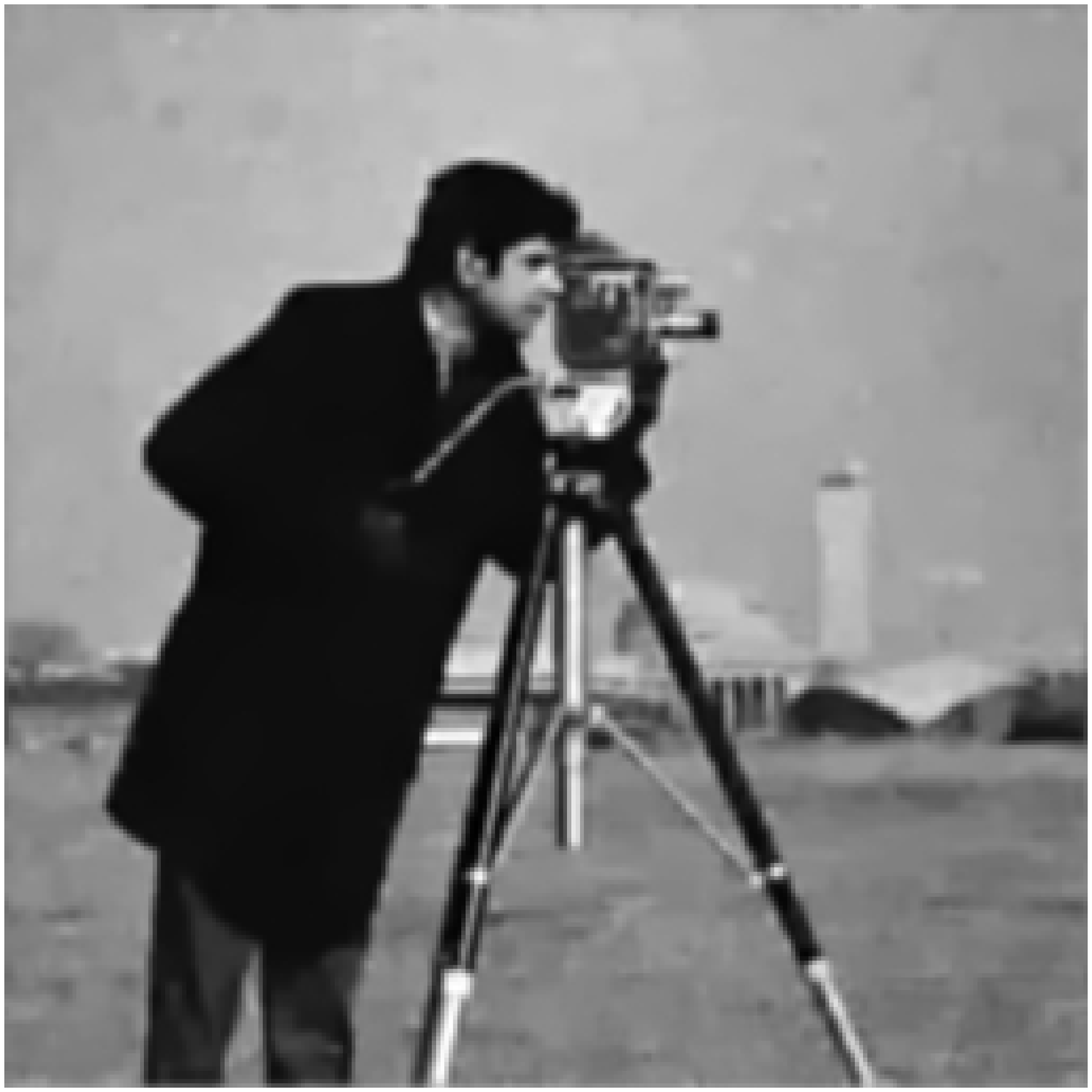}
	\includegraphics[width=0.19\textwidth]{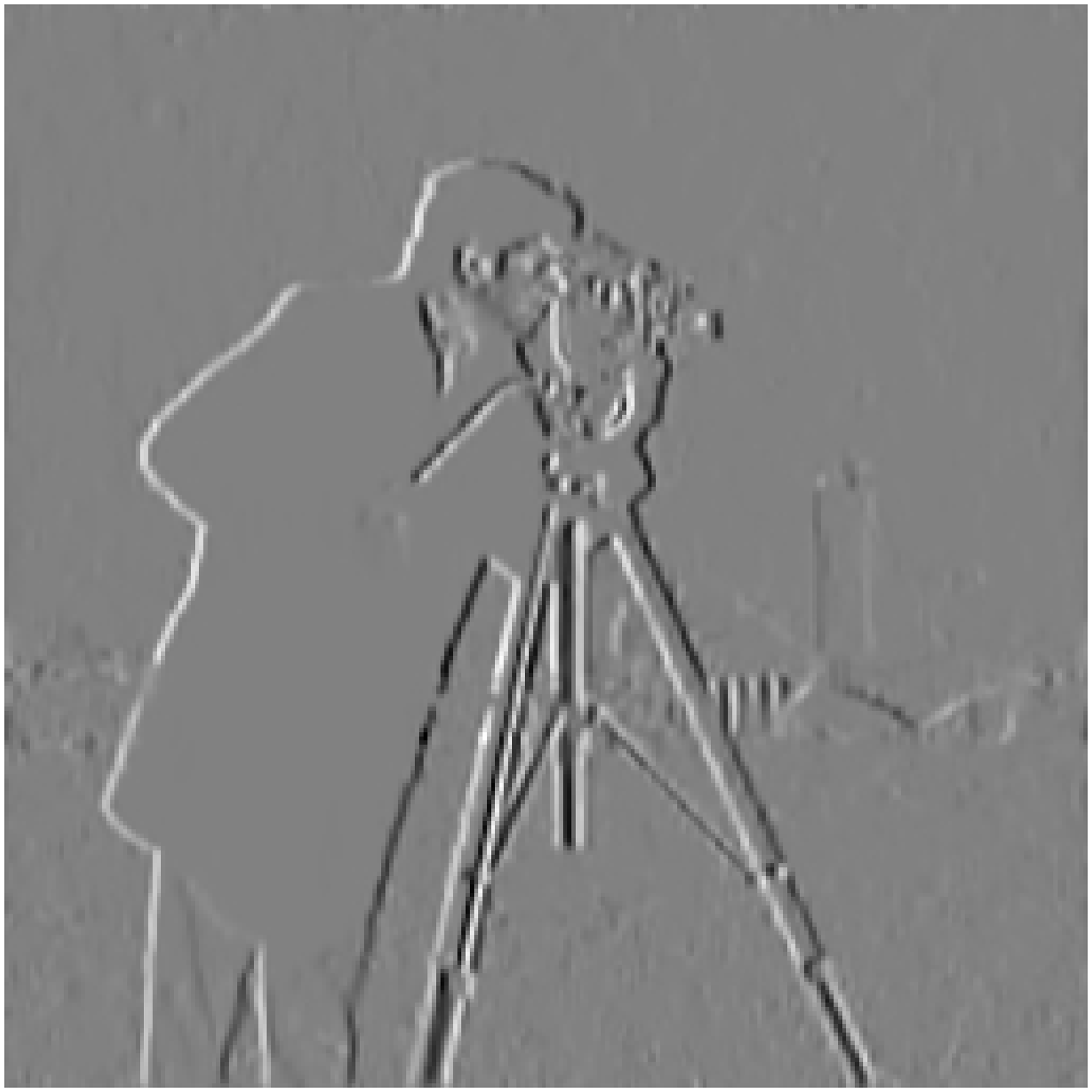}
	\includegraphics[width=0.19\textwidth]{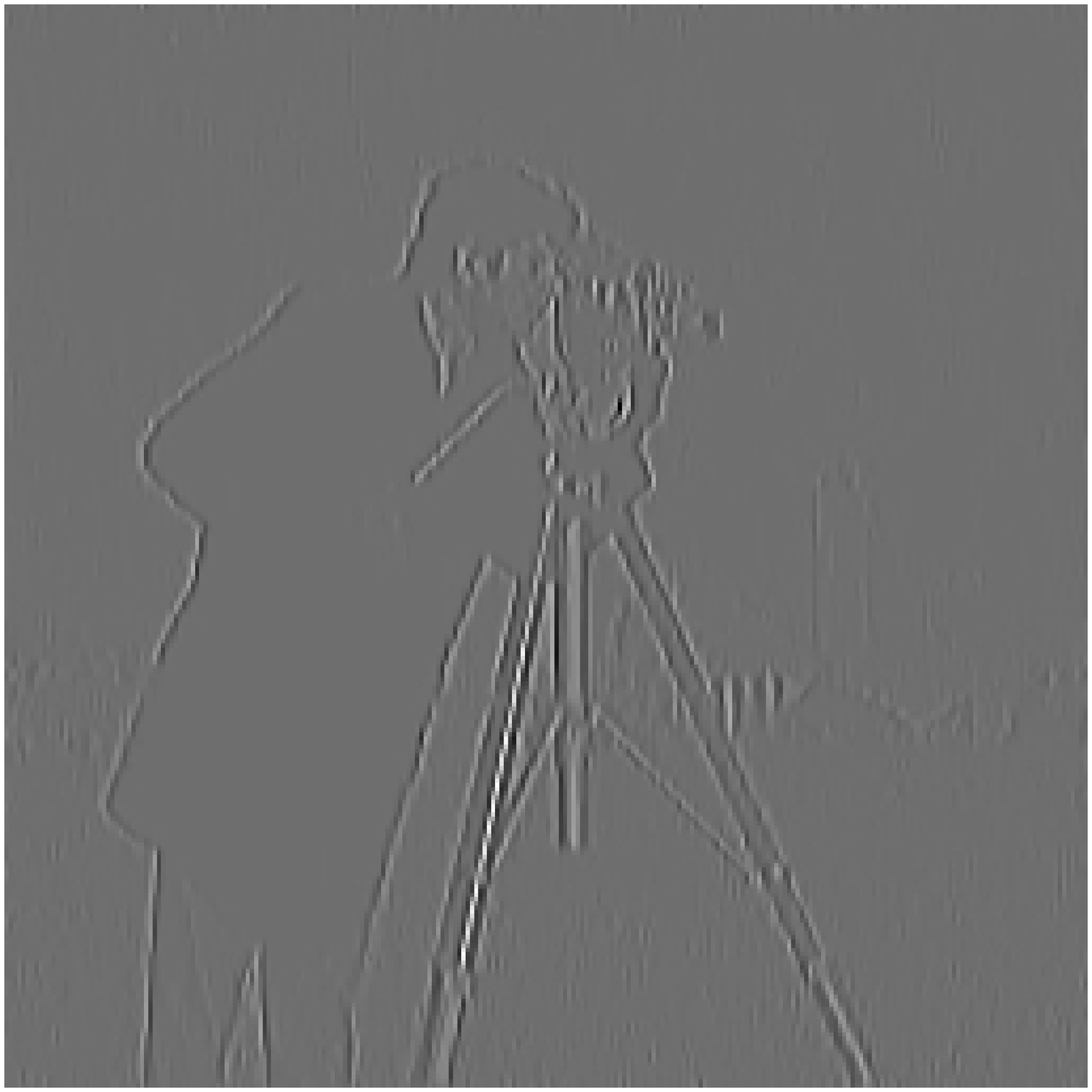}
	\includegraphics[width=0.19\textwidth]{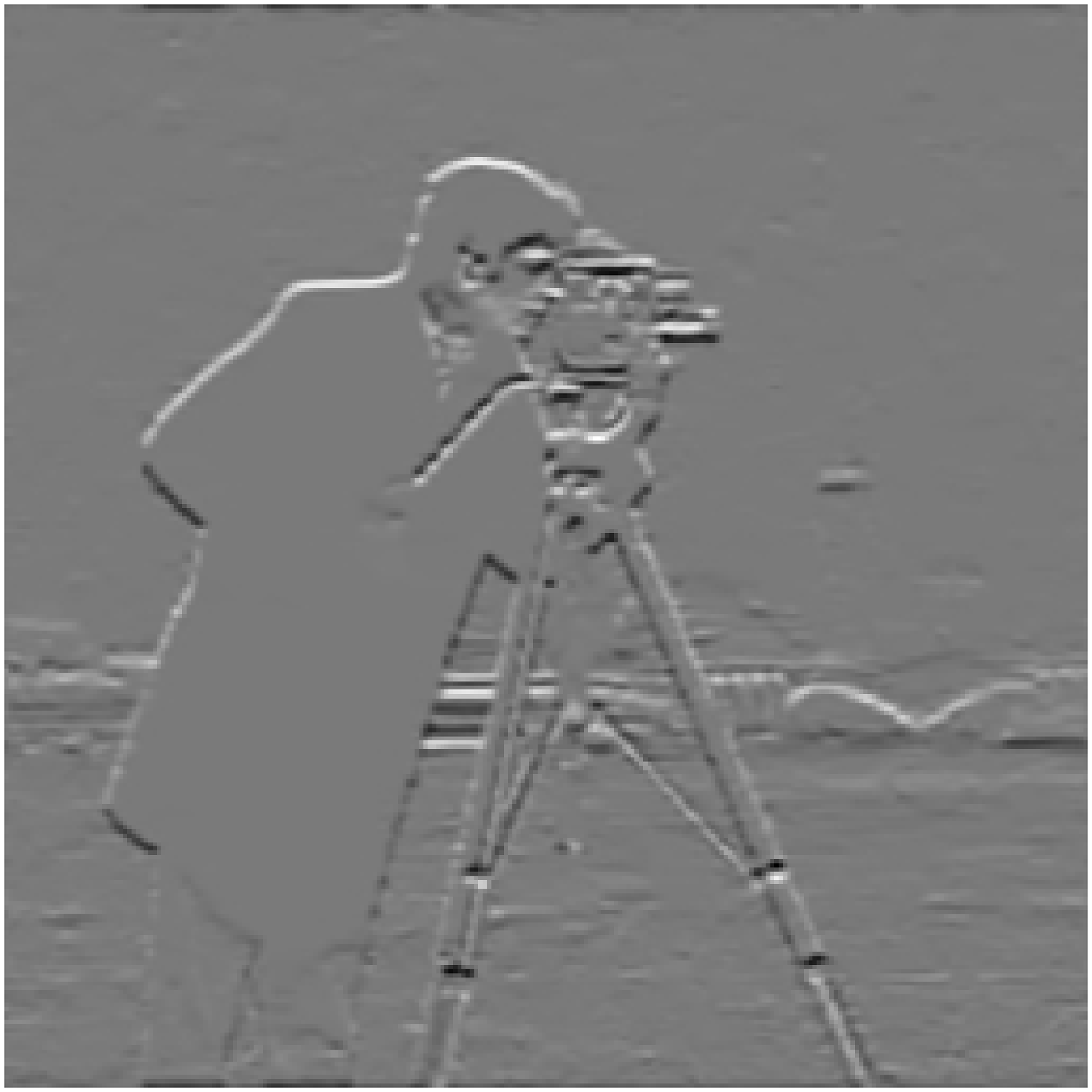}\\
	\includegraphics[width=0.19\textwidth]{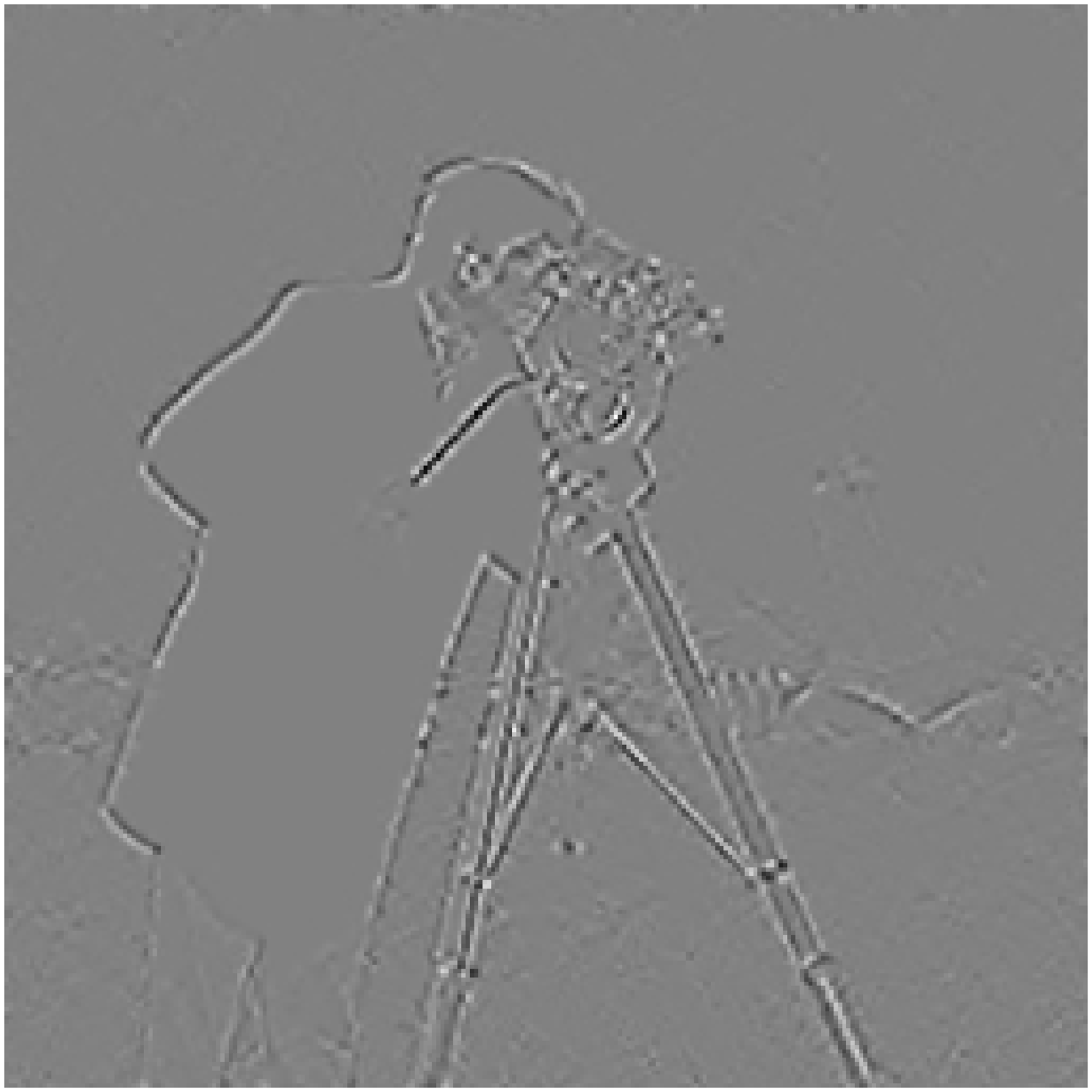}
	\includegraphics[width=0.19\textwidth]{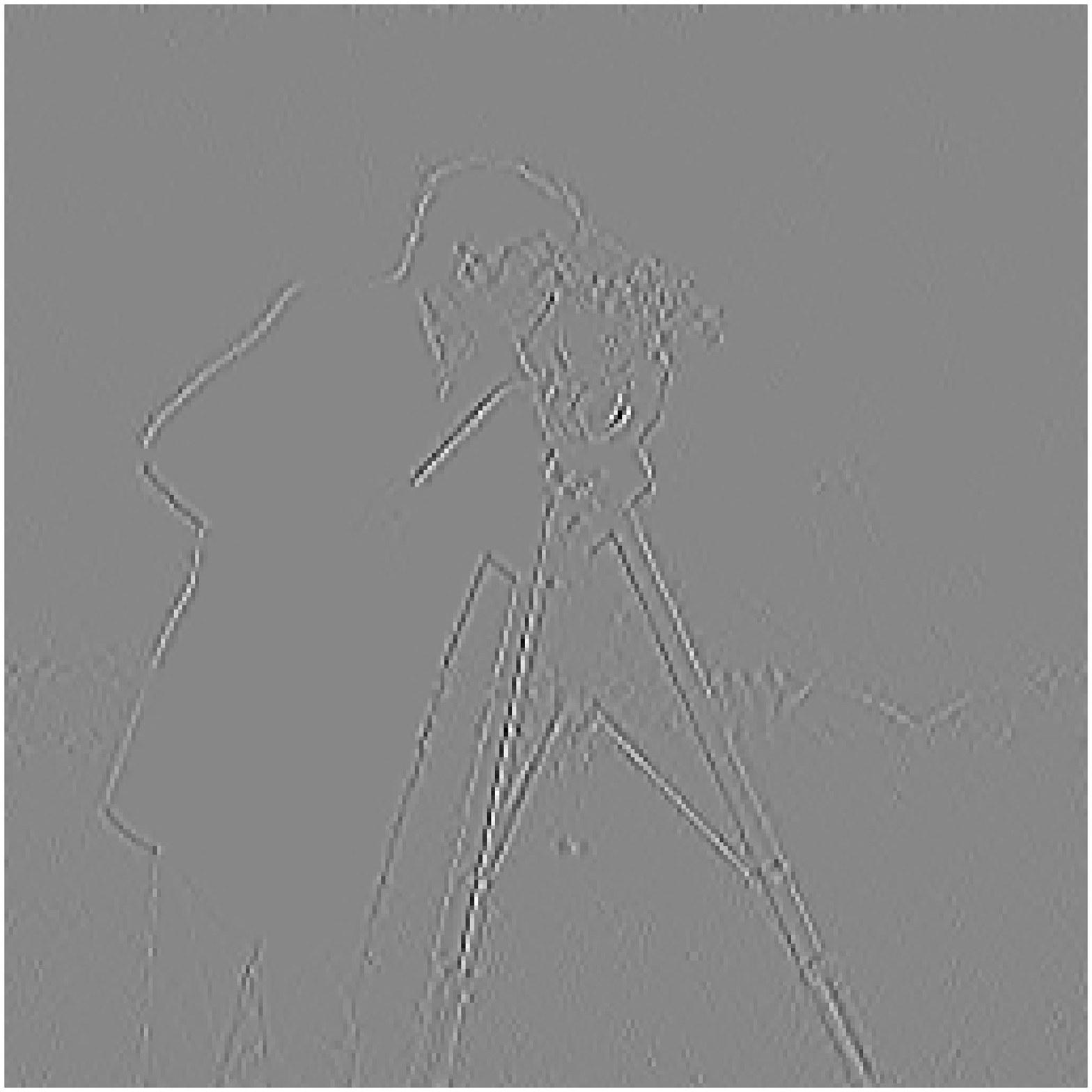}
	\includegraphics[width=0.19\textwidth]{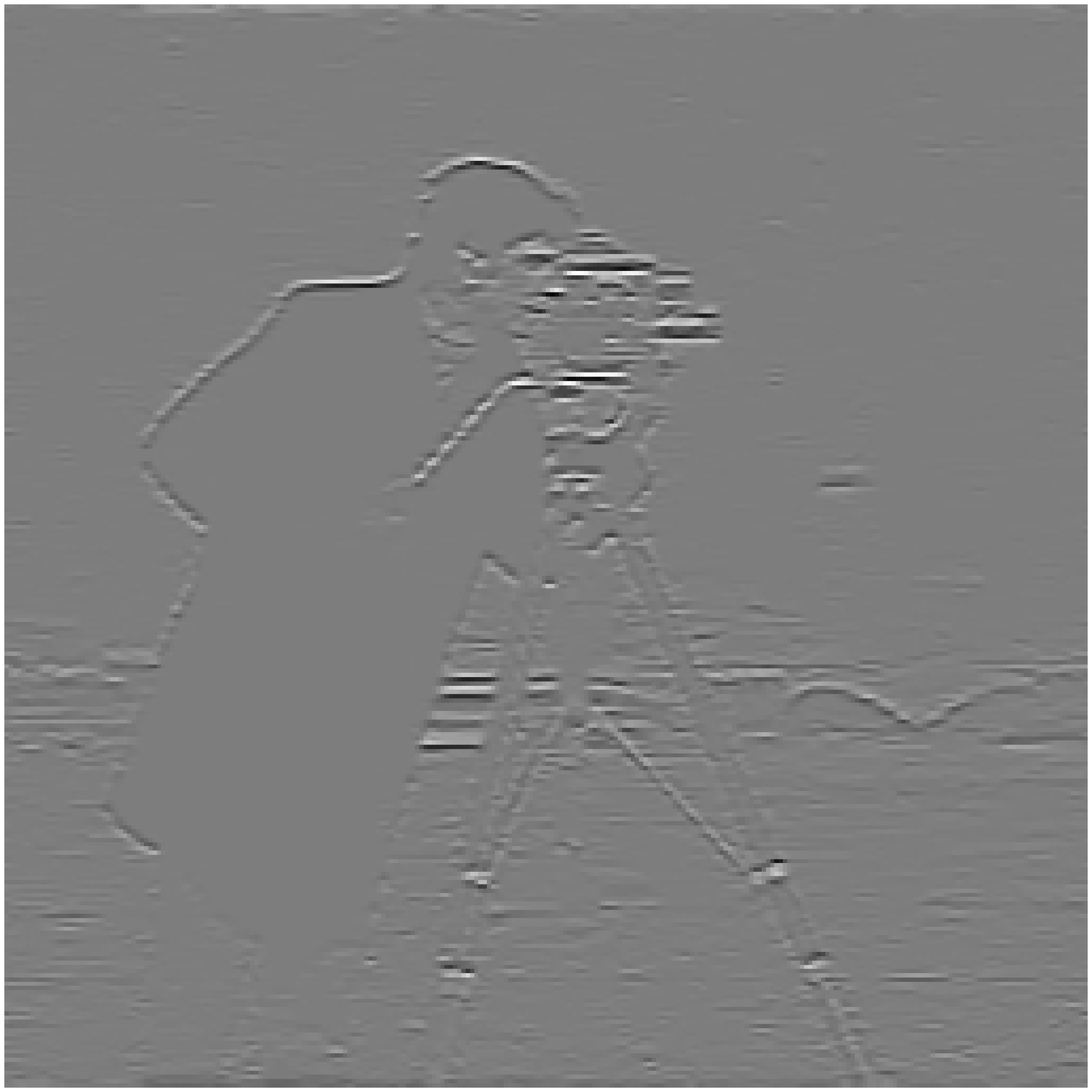}
	\includegraphics[width=0.19\textwidth]{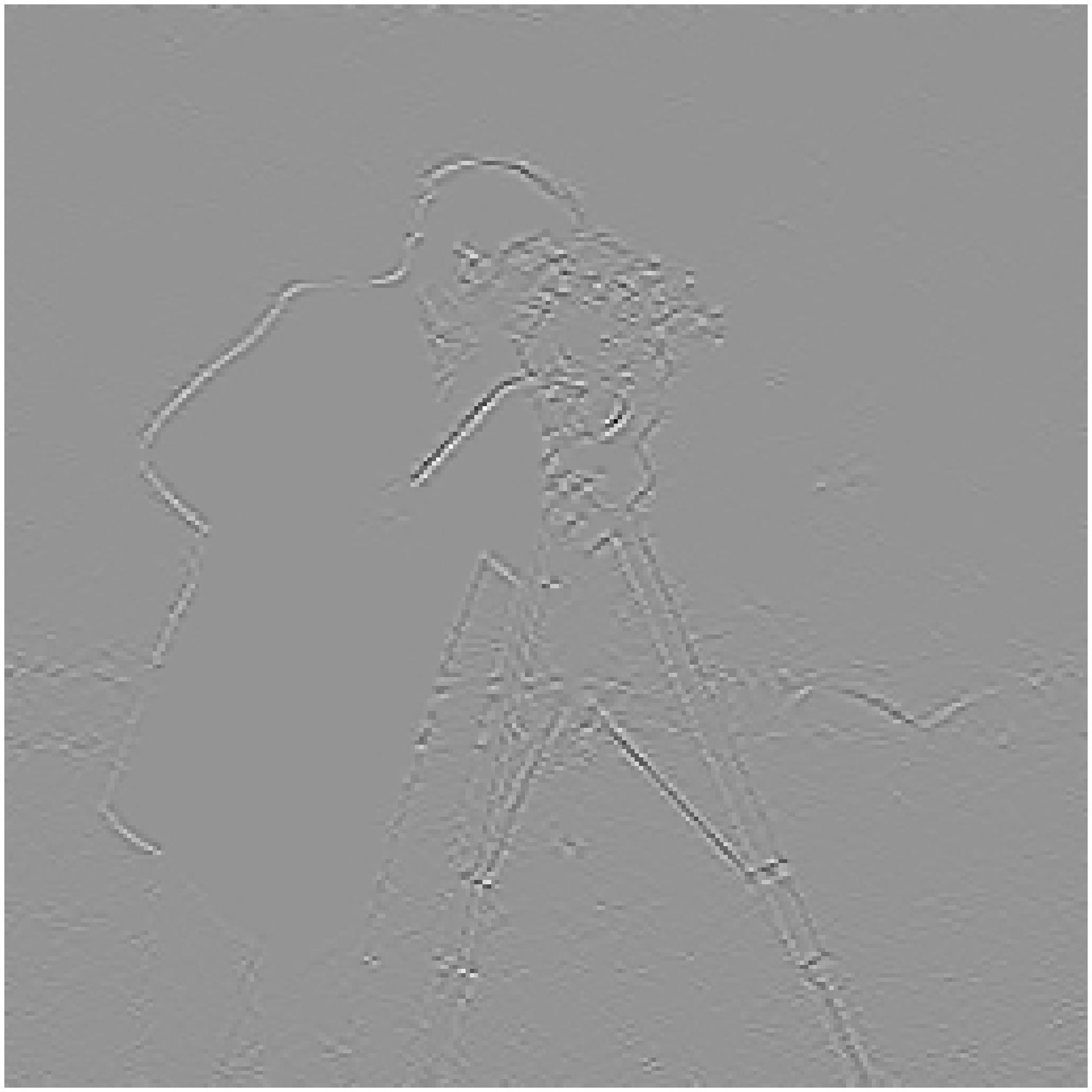}
	\includegraphics[width=0.19\textwidth]{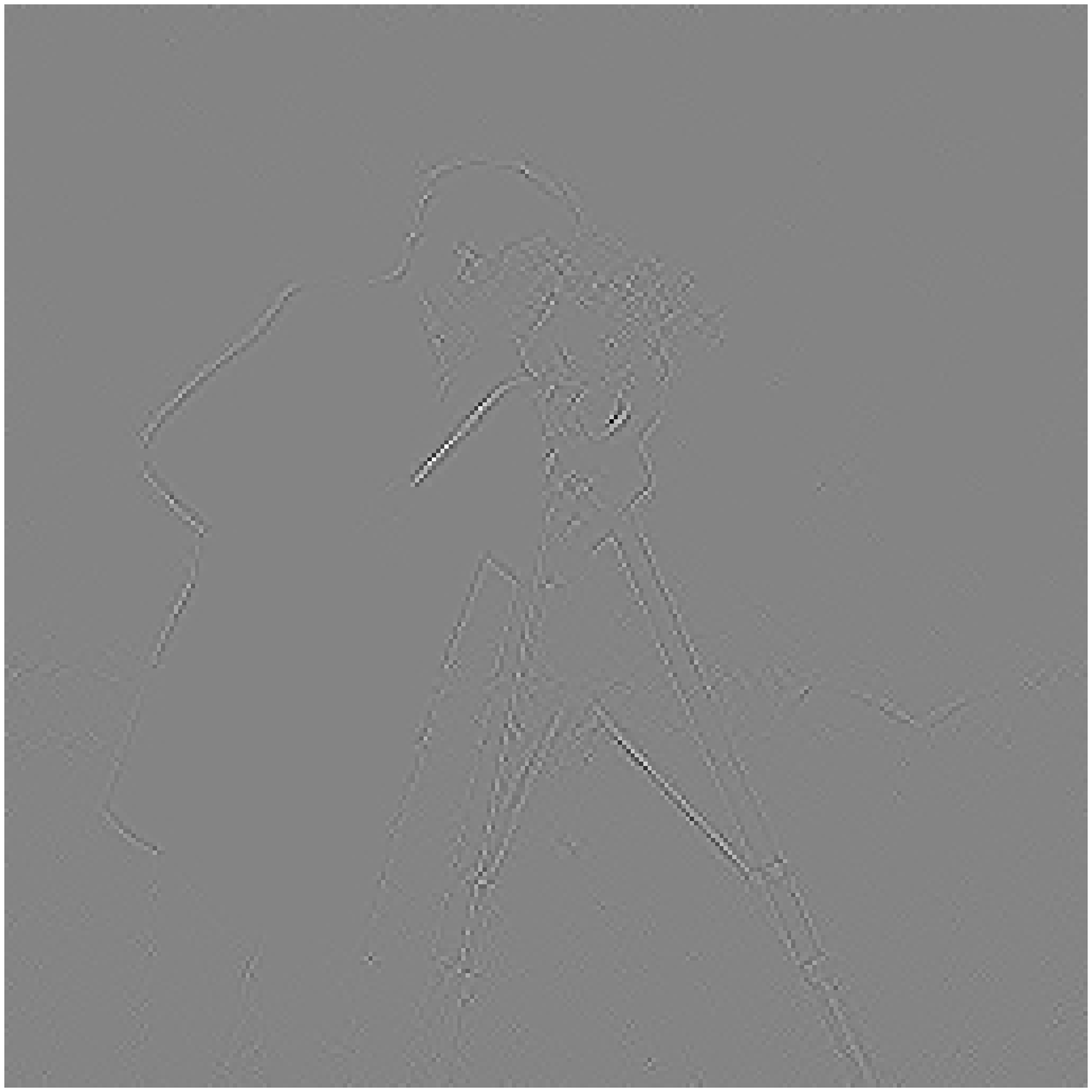}
	\caption{The  feature images of the noisy and denoised.}
	\label{F；analysis}
\end{figure}

In Figure \ref{F；D1} - Figure \ref{F；D6}, we show the visual comparisons of the results from the four compared methods. From the results of SB, we can see that the recovered images  have a lot of artifacts in the homogeneous regions. The reason is that the regularization parameter $\lambda$ is chosen optimal for PSNR value. In order to preserve the  sharp edges in images, a relatively small $\lambda$ should be chosen. However, it  leads to lack of smoothness in  homogeneous regions. This visual effect is especially clear in the `cameraman' image (see the sky and grass). From the results of DTCWT and LCHMM, we can observe that the restored images introduce too much artifacts. one of the reasons is that they both use the decimated wavelet transform. From the last three results, we can see that the denoised images by  ${l_o}$-WF and ours have similar visual effects, but our method   has a slightly higher PSNR and less runtime averagely which can be seen in Table \ref{T1}.

Table \ref{T1} lists the PSNR and  total run-time of the compared five methods for image denoising. Since the LCHMM is coded by C++ and the others by matlab, we don't list the runtime of LCHMM.  From Table \ref{T1}, we can see that our model outperforms the other image restoration methods in terms of PSNR, averagely. In addition, from Table \ref{T1}, we can see that the DTCWT is the fastest (It spends less than 1 second for the six test images). However, the quality of the DTCWT is the worst. 

To show the convergence of the our algorithm using equation (\ref{17}), we plot the relative error $\frac{||u^{n+1}-u^n||}{||u^n||}$ versus the iterations of the six test images in Figure \ref{F；error}. From Figure \ref{F；error}, we can observe  the relative error is monotonically decreasing with the iterations, which numerically proves that the  algorithm is convergent.

\begin{figure}[htbp]
	\centering
	\includegraphics[width=0.32\textwidth]{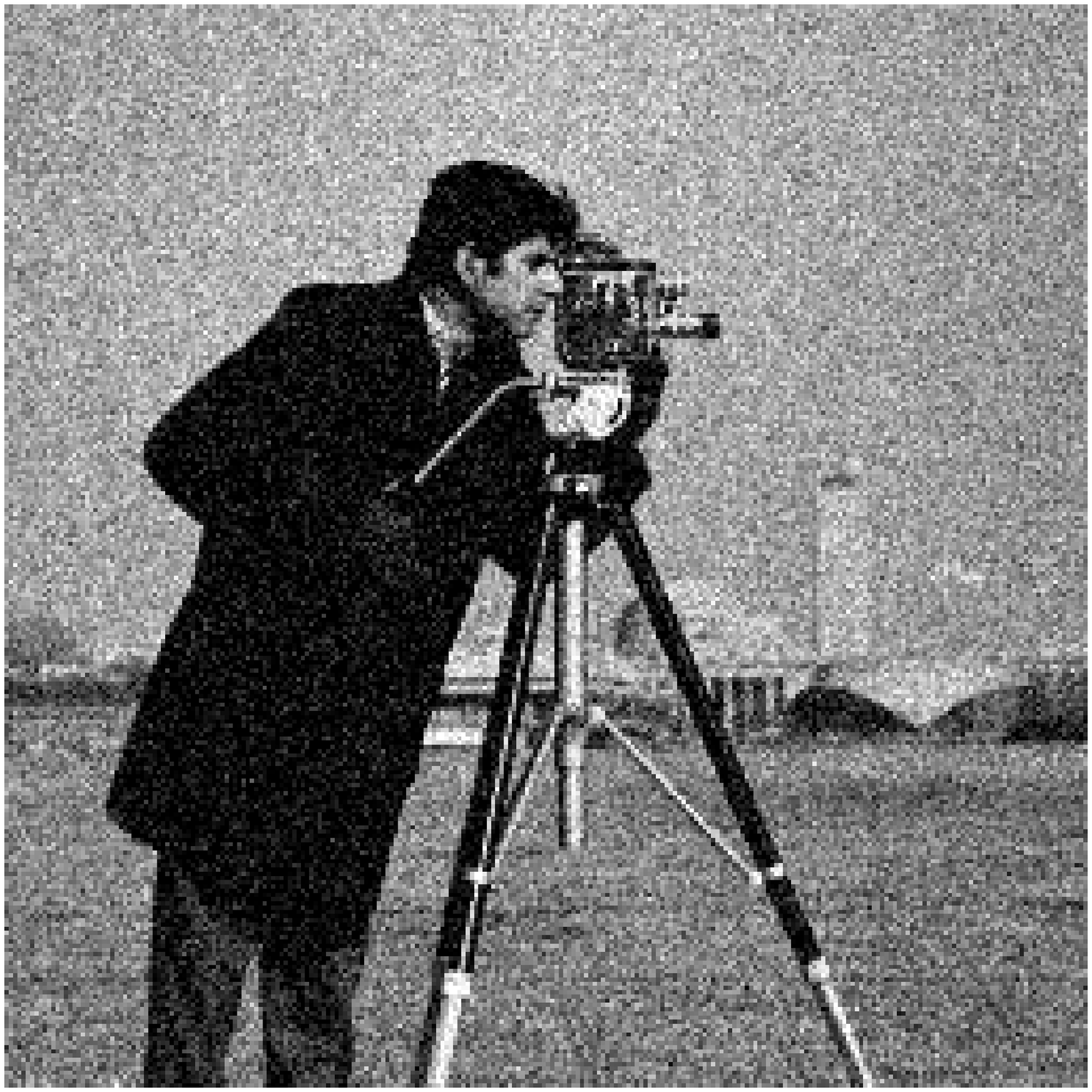}\\
	\includegraphics[width=0.32\textwidth]{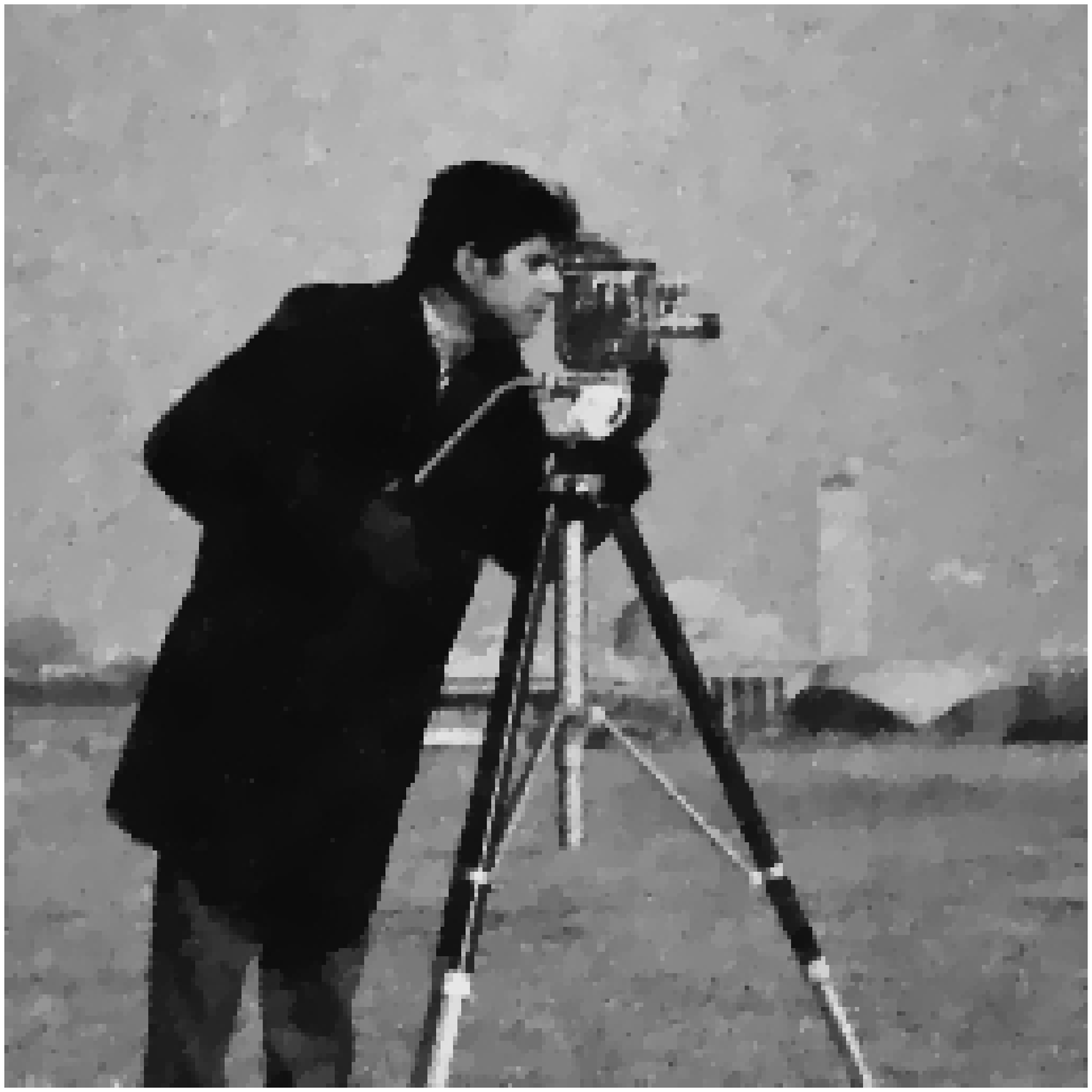}
	\includegraphics[width=0.32\textwidth]{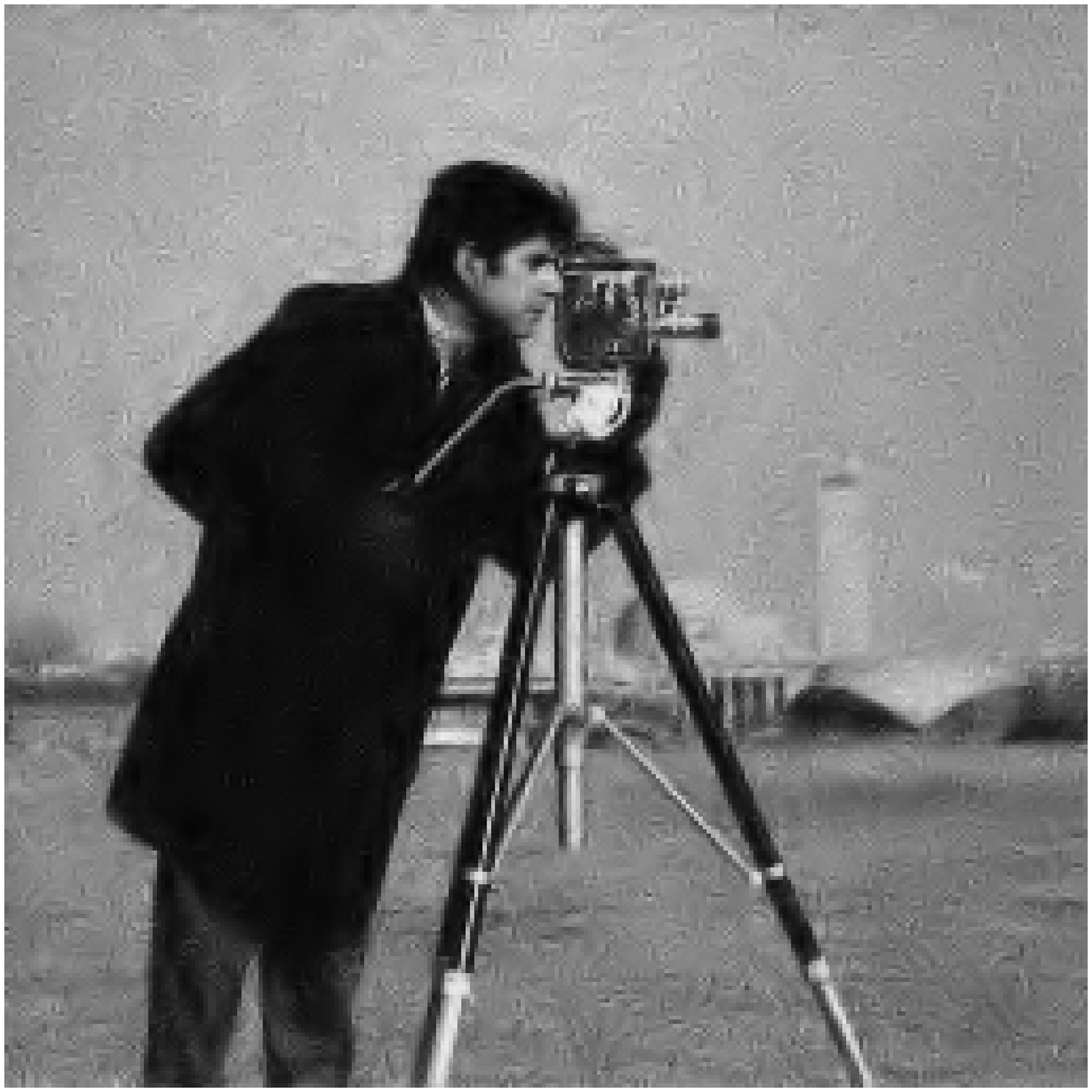}
	\includegraphics[width=0.32\textwidth]{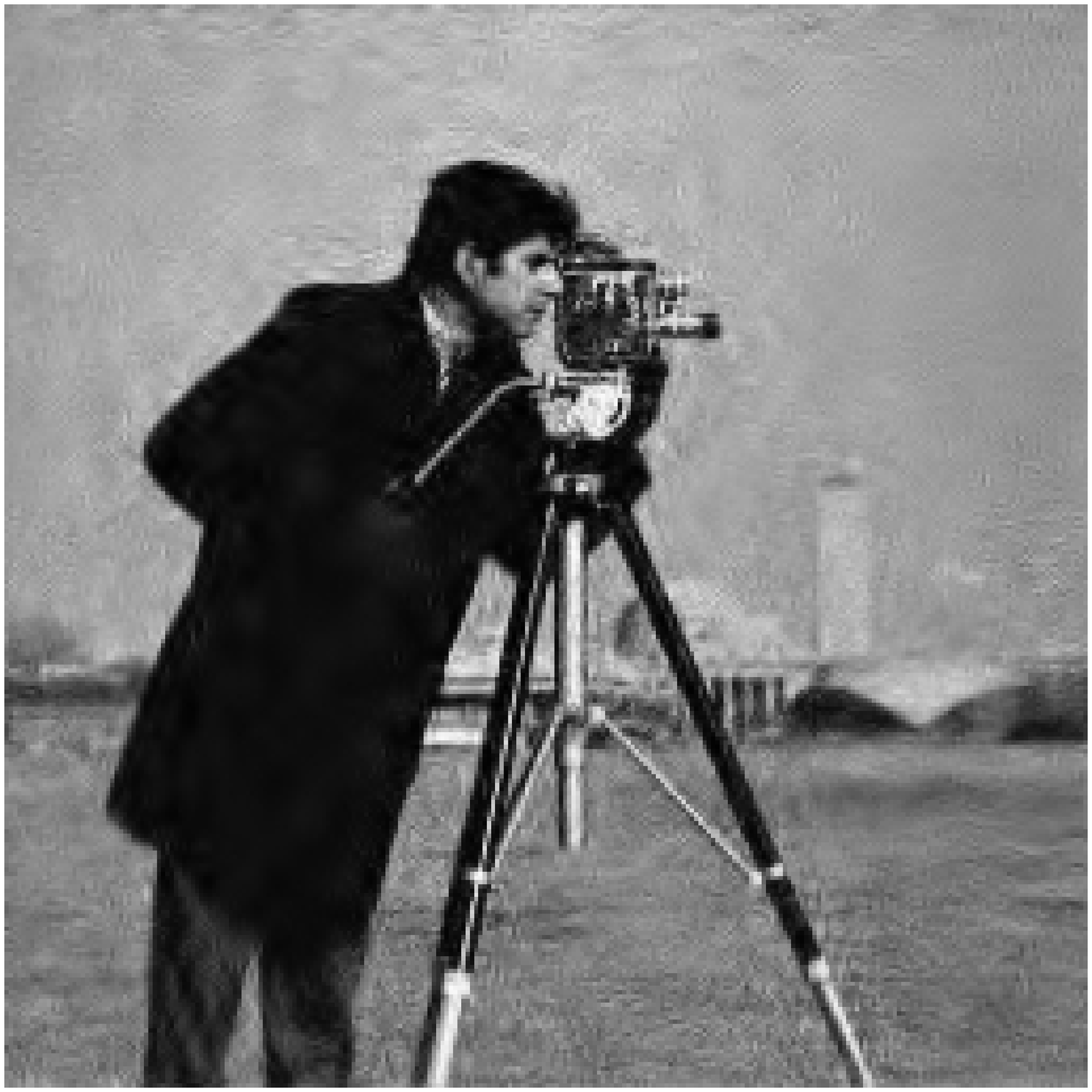}\\
	\includegraphics[width=0.32\textwidth]{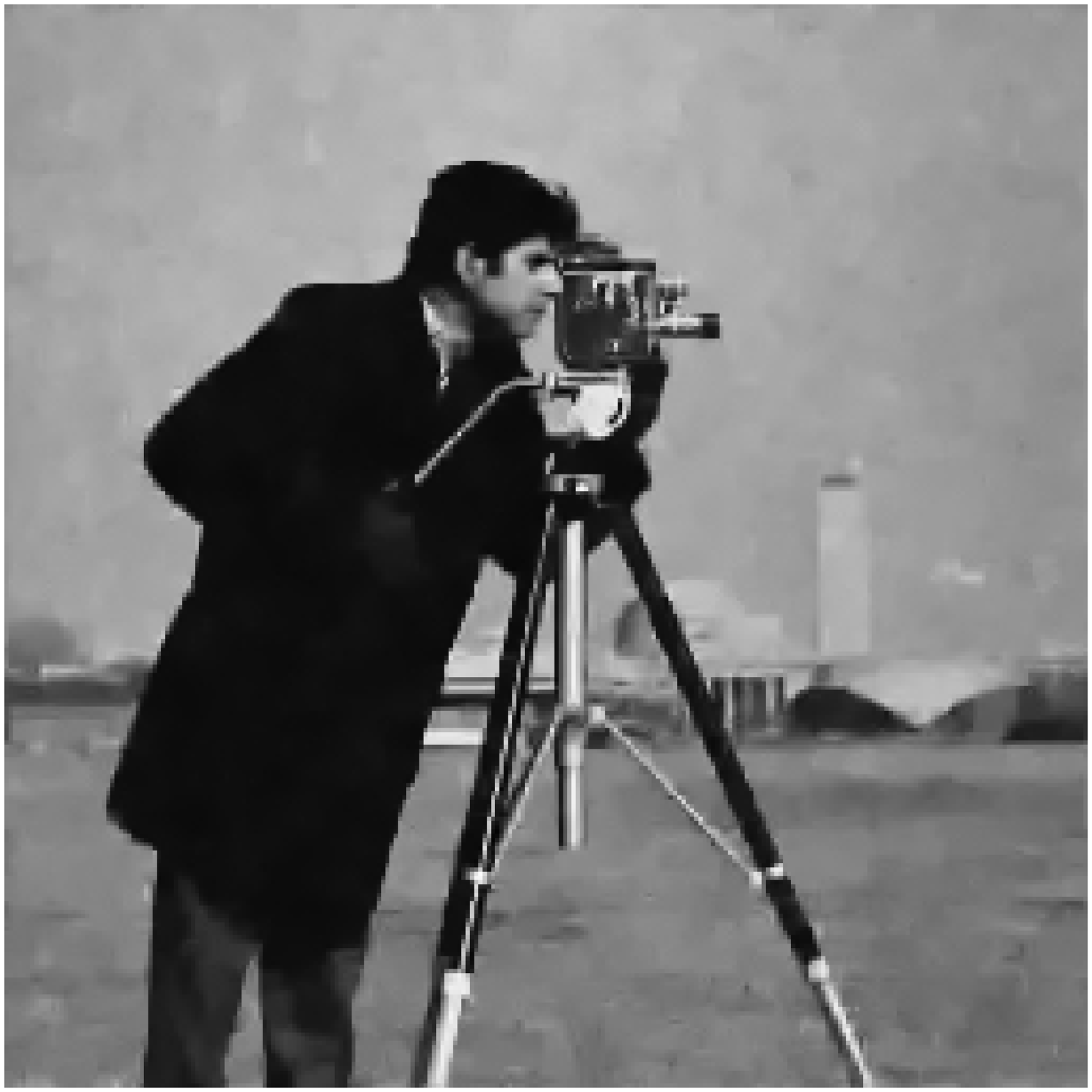}
	\includegraphics[width=0.32\textwidth]{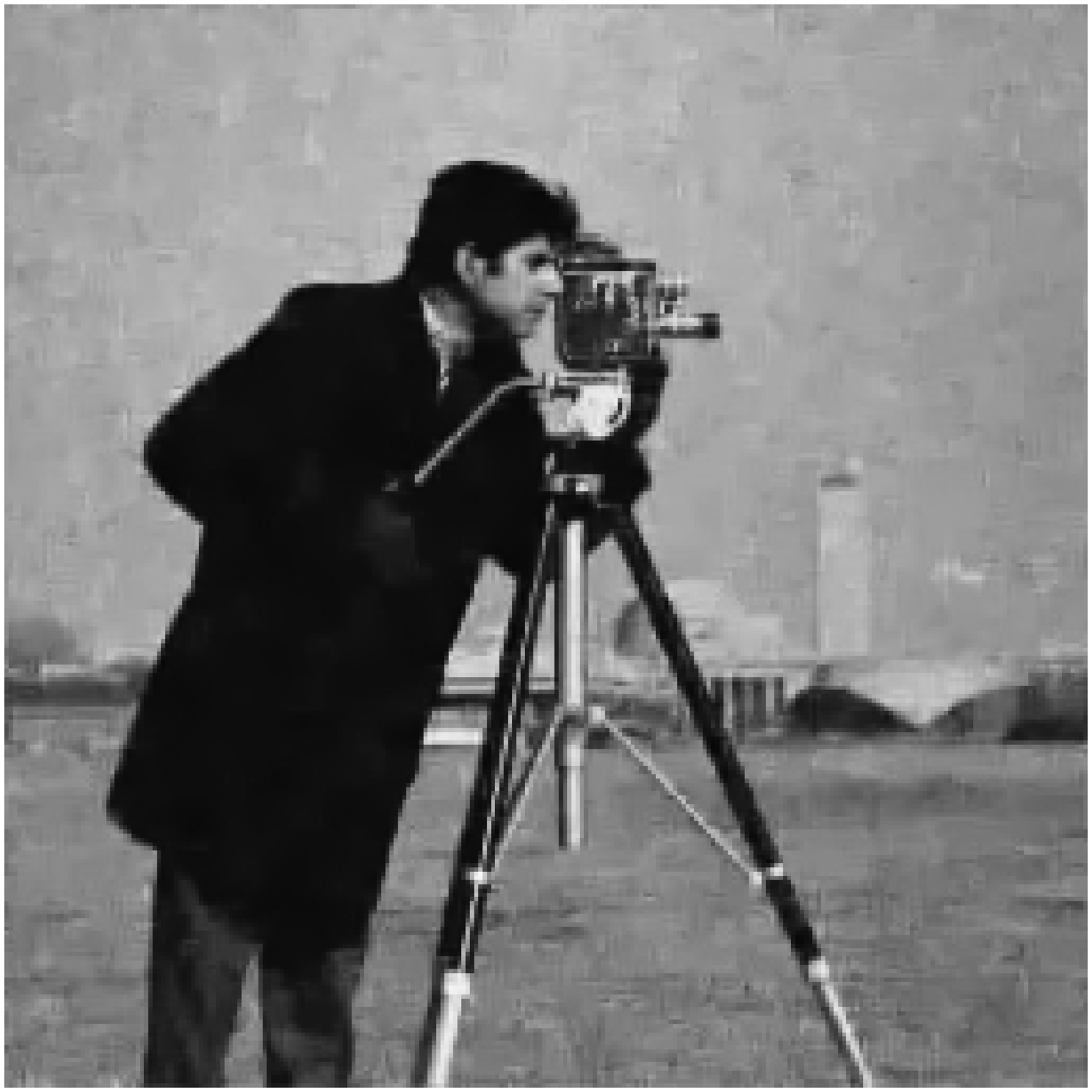}
	\includegraphics[width=0.32\textwidth]{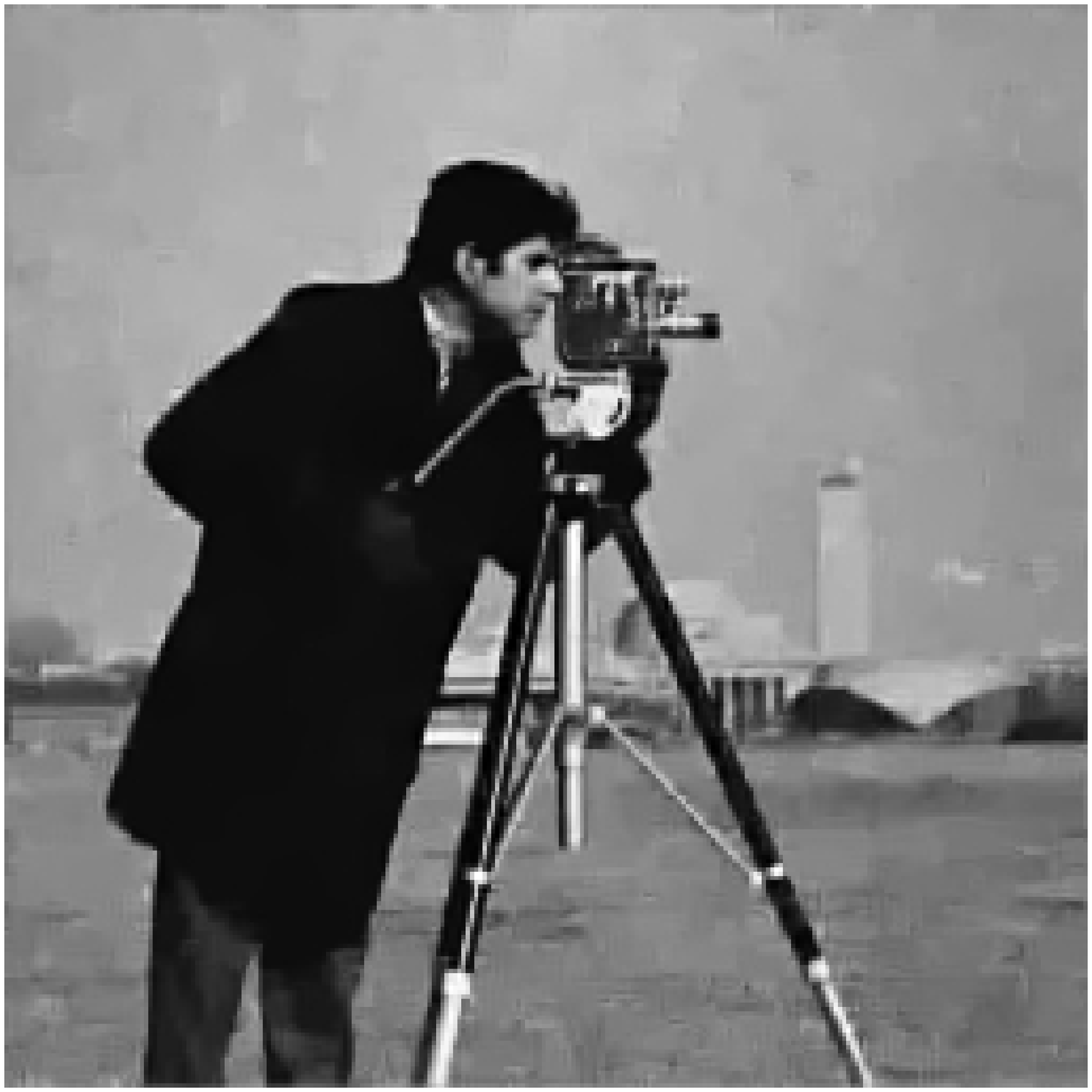}
		\caption{Comparison for image denoising. From left to right and up to down are: noisy images and  denoised by  SB, DTCWT, LCHMM, ${l_o}$-WF and ours using equations (\ref{13}) and (\ref{17}), respectively.}
	\label{F；D1}
\end{figure}

\begin{figure}[htbp]
	\centering	
	\includegraphics[width=0.32\textwidth]{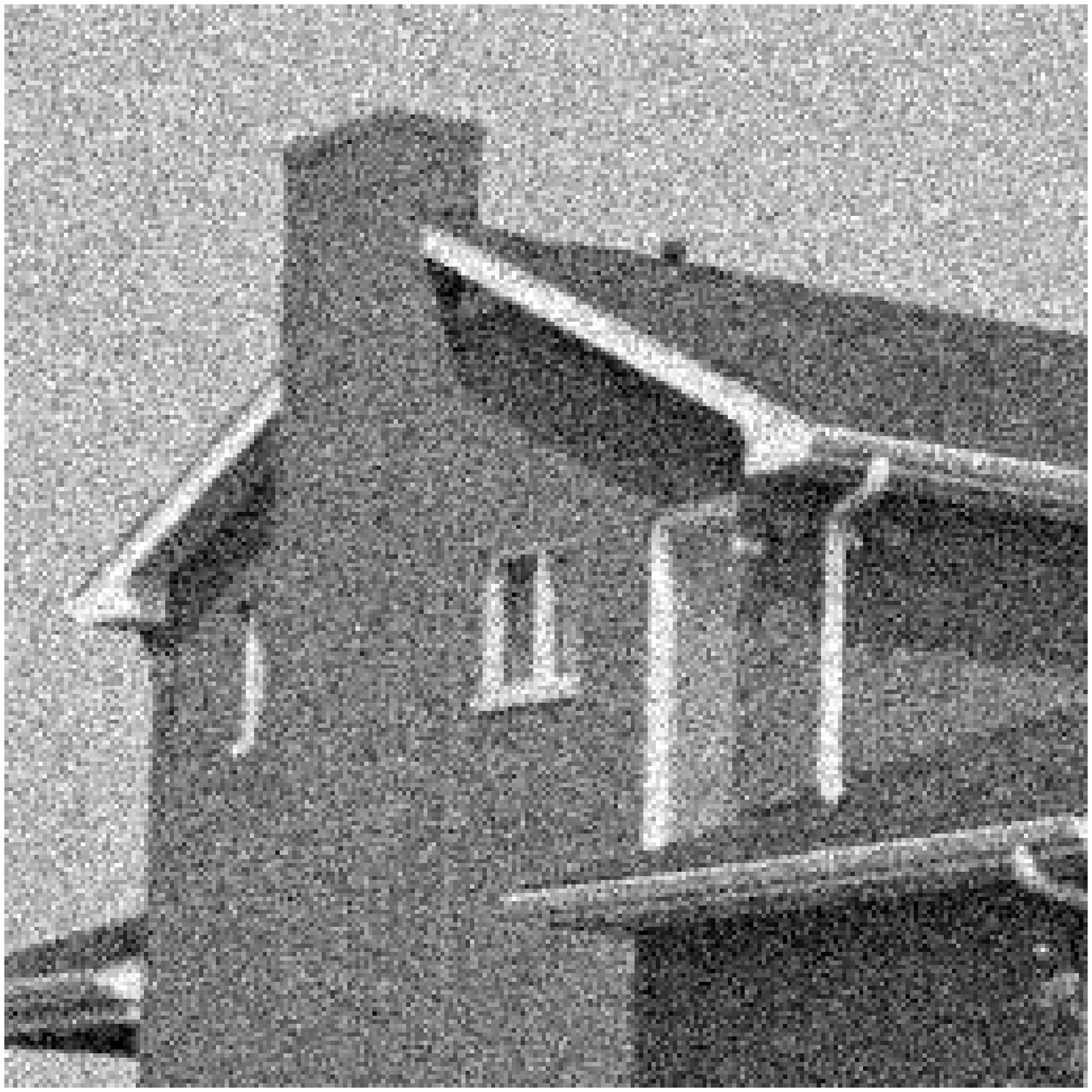}\\
	\includegraphics[width=0.32\textwidth]{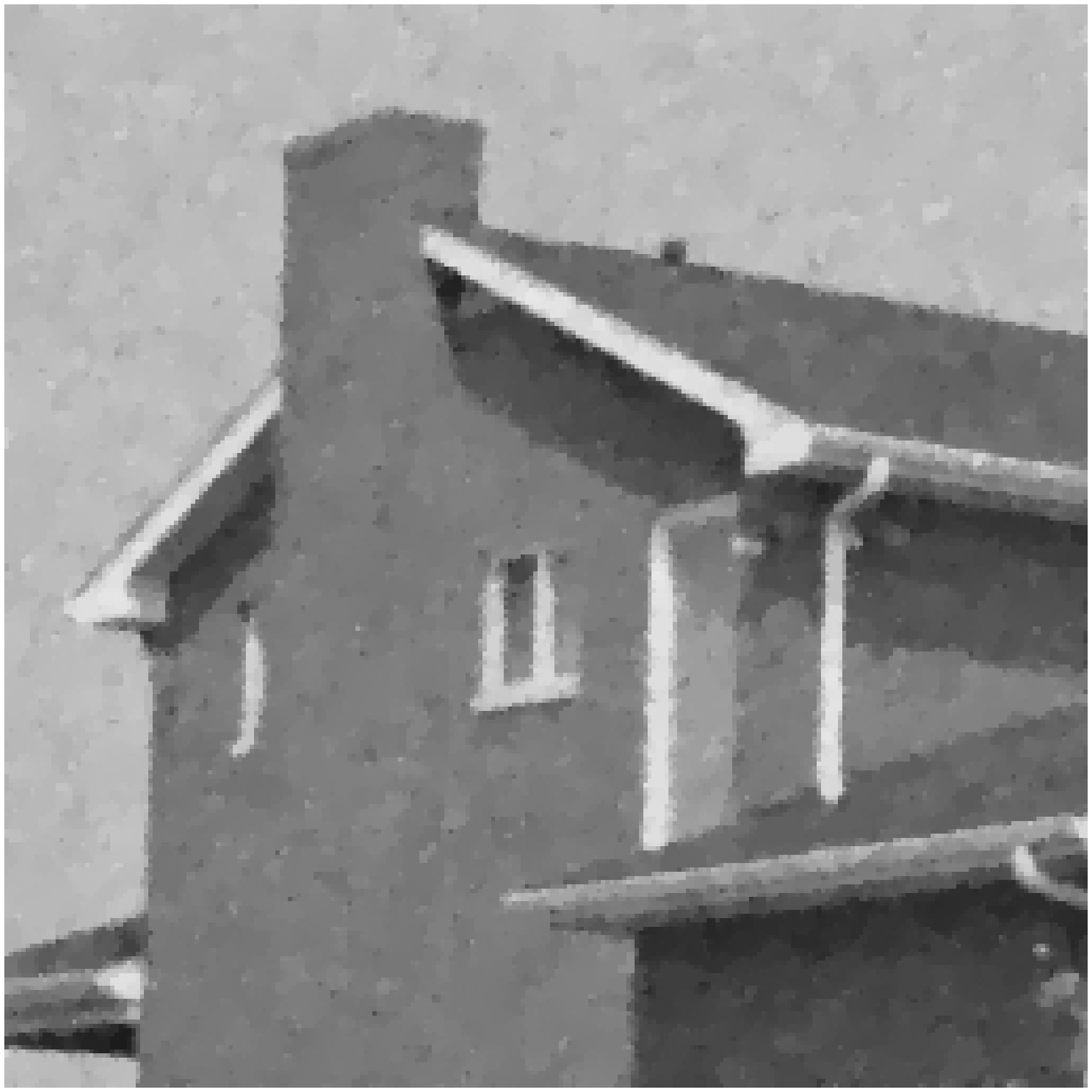}
	\includegraphics[width=0.32\textwidth]{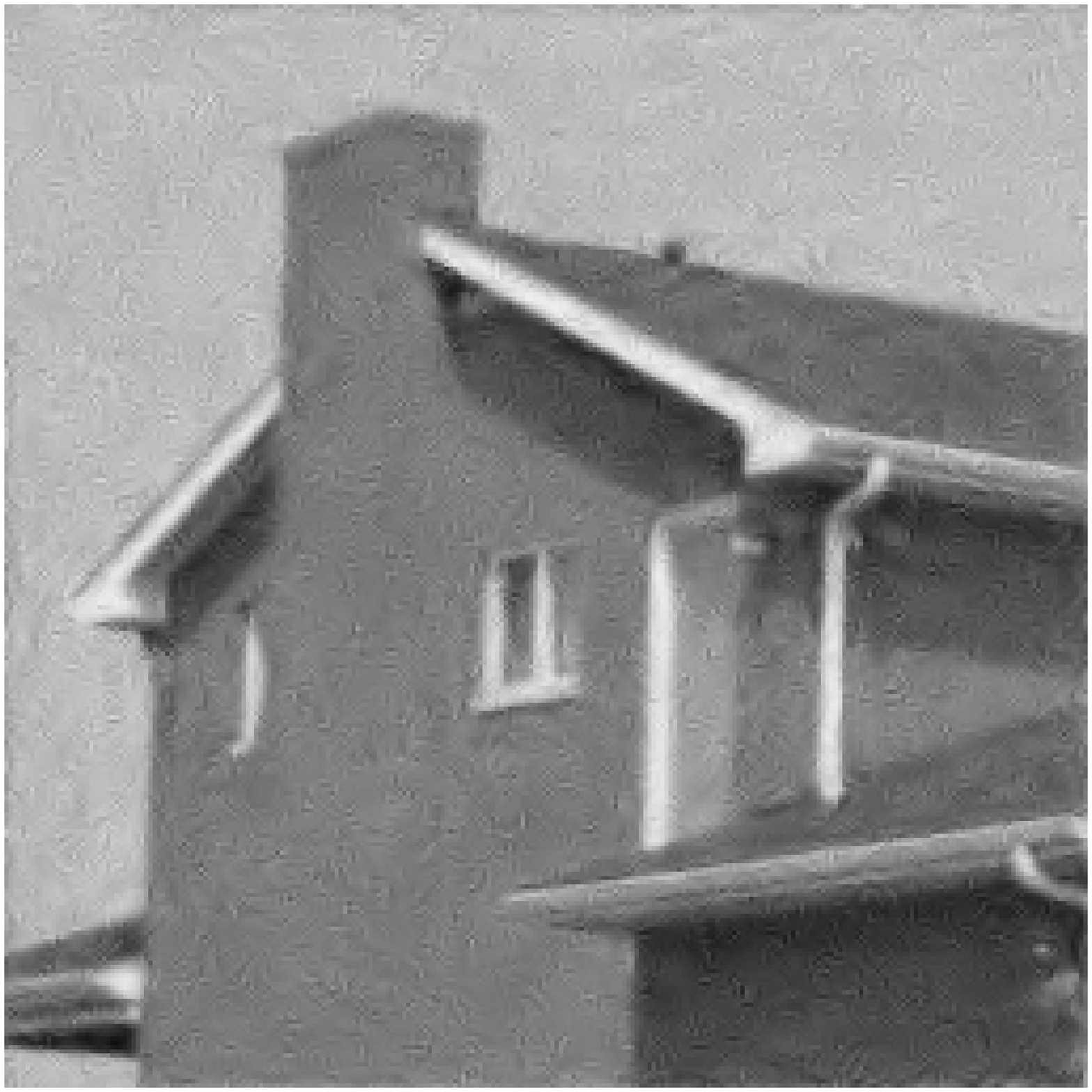}
	\includegraphics[width=0.32\textwidth]{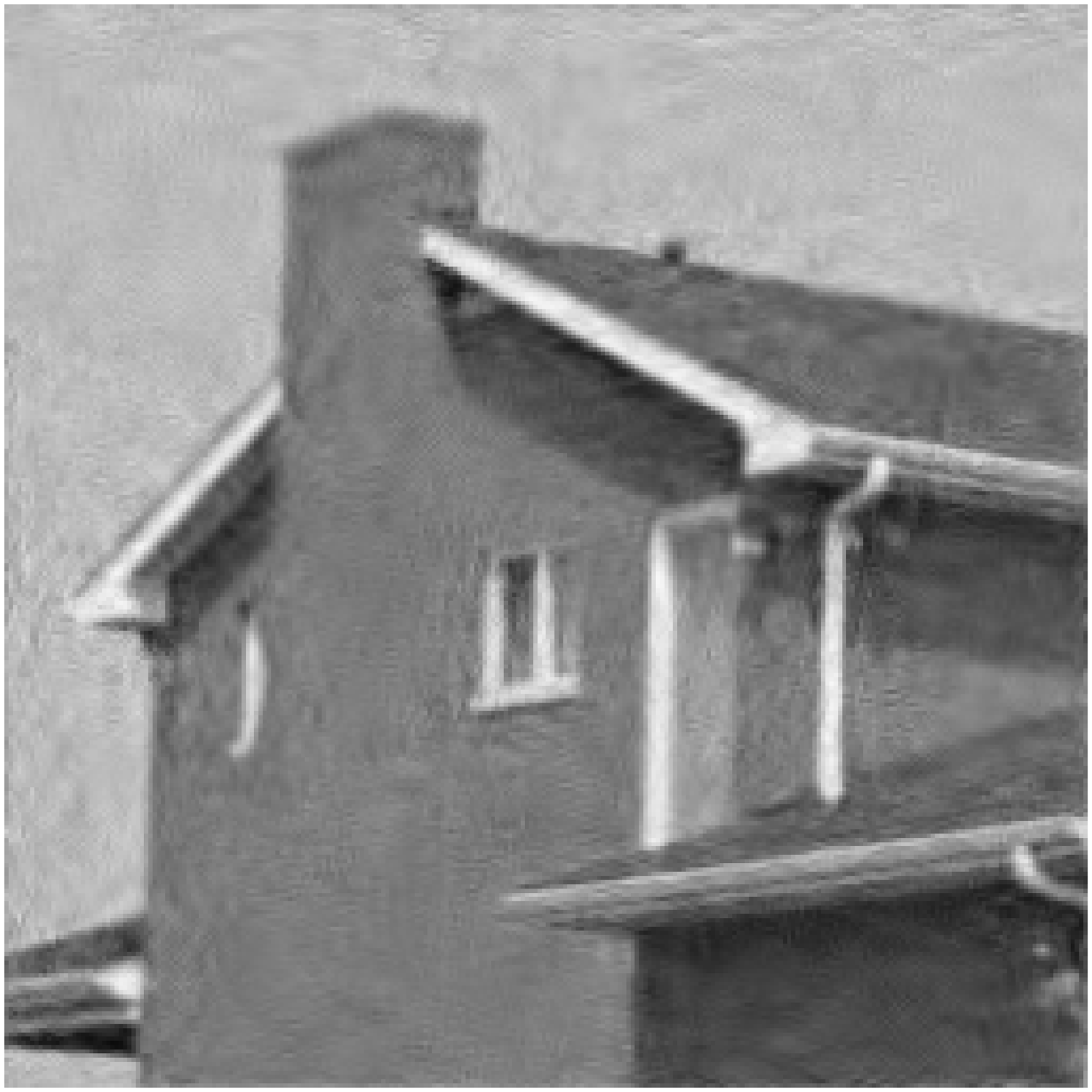}\\
	\includegraphics[width=0.32\textwidth]{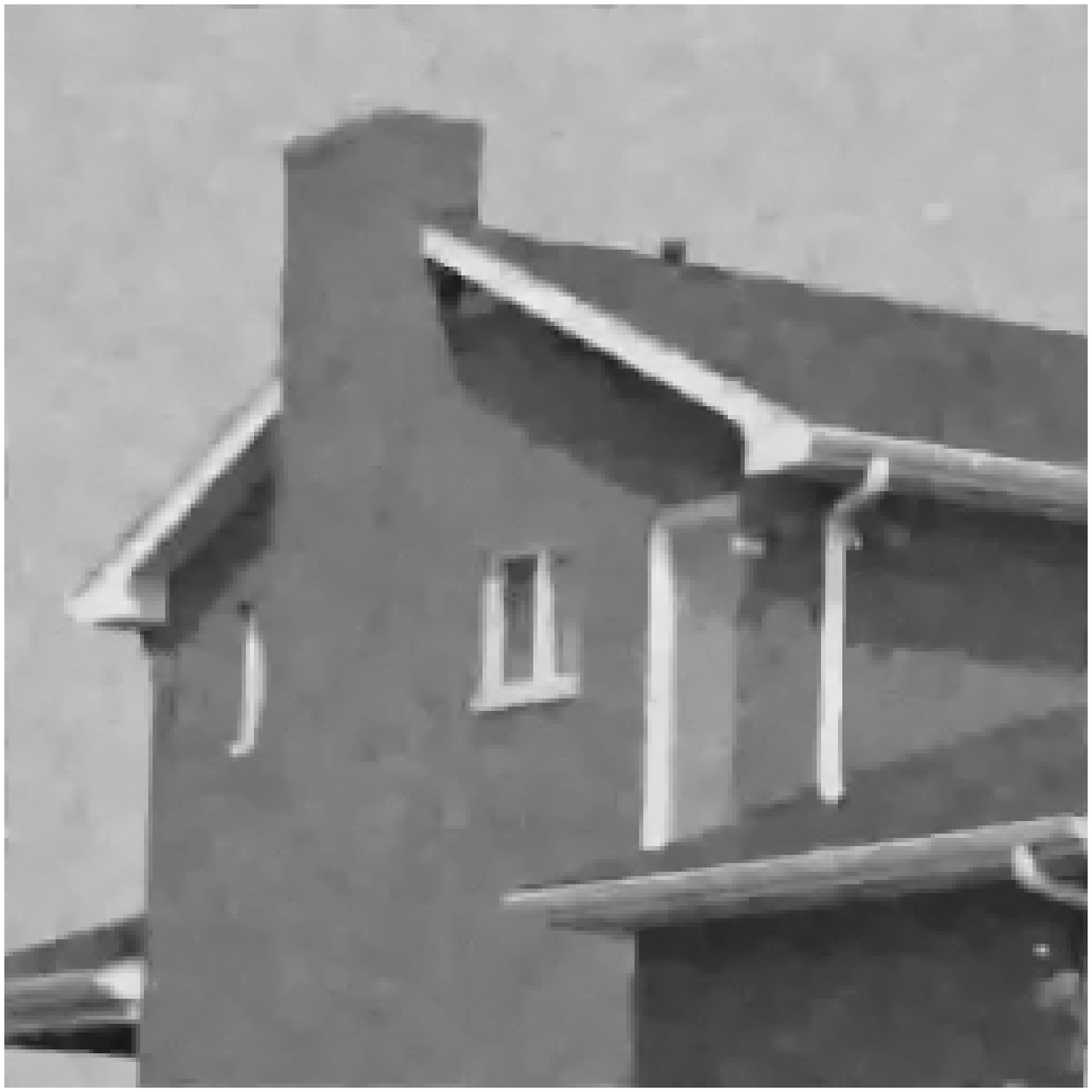}
	\includegraphics[width=0.32\textwidth]{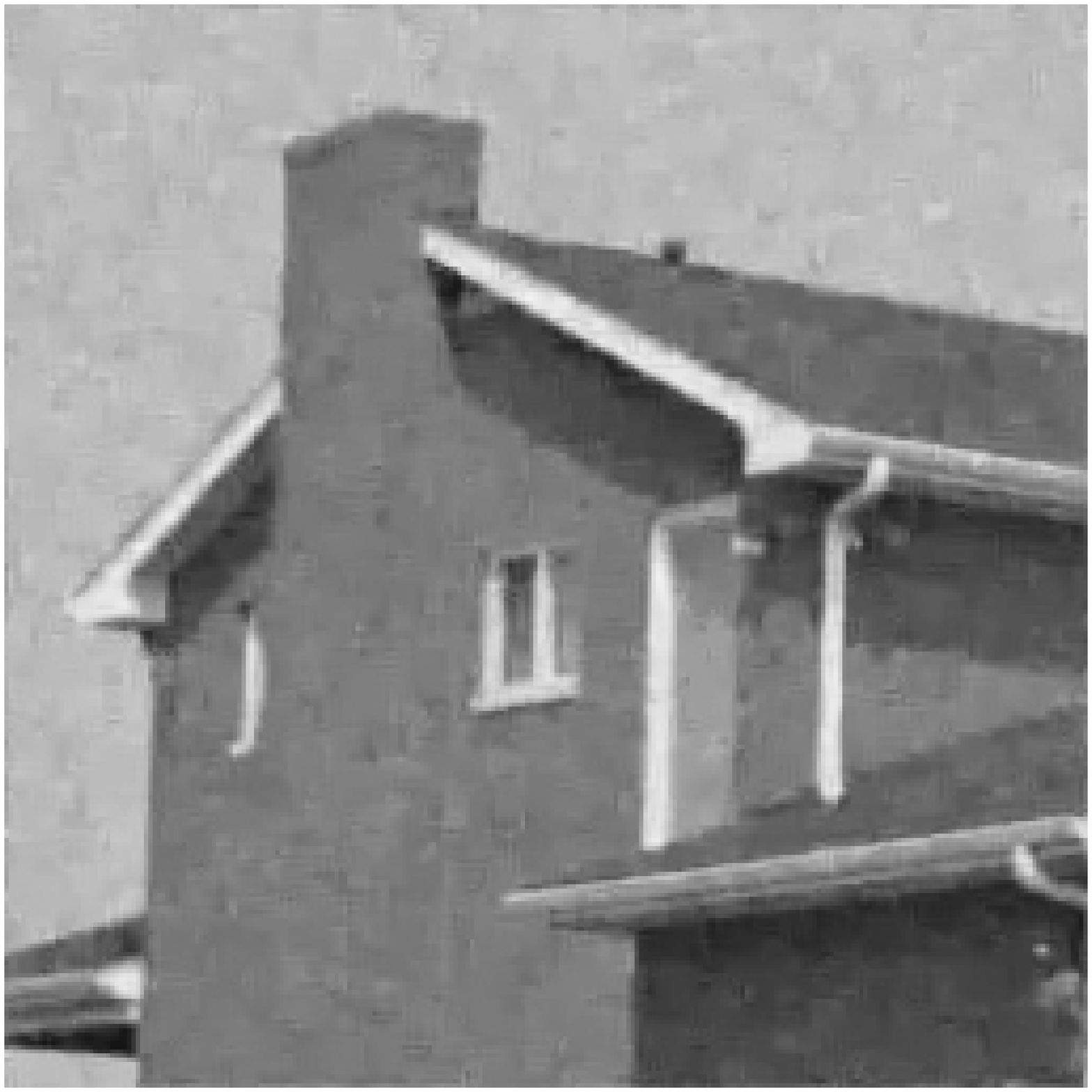}
	\includegraphics[width=0.32\textwidth]{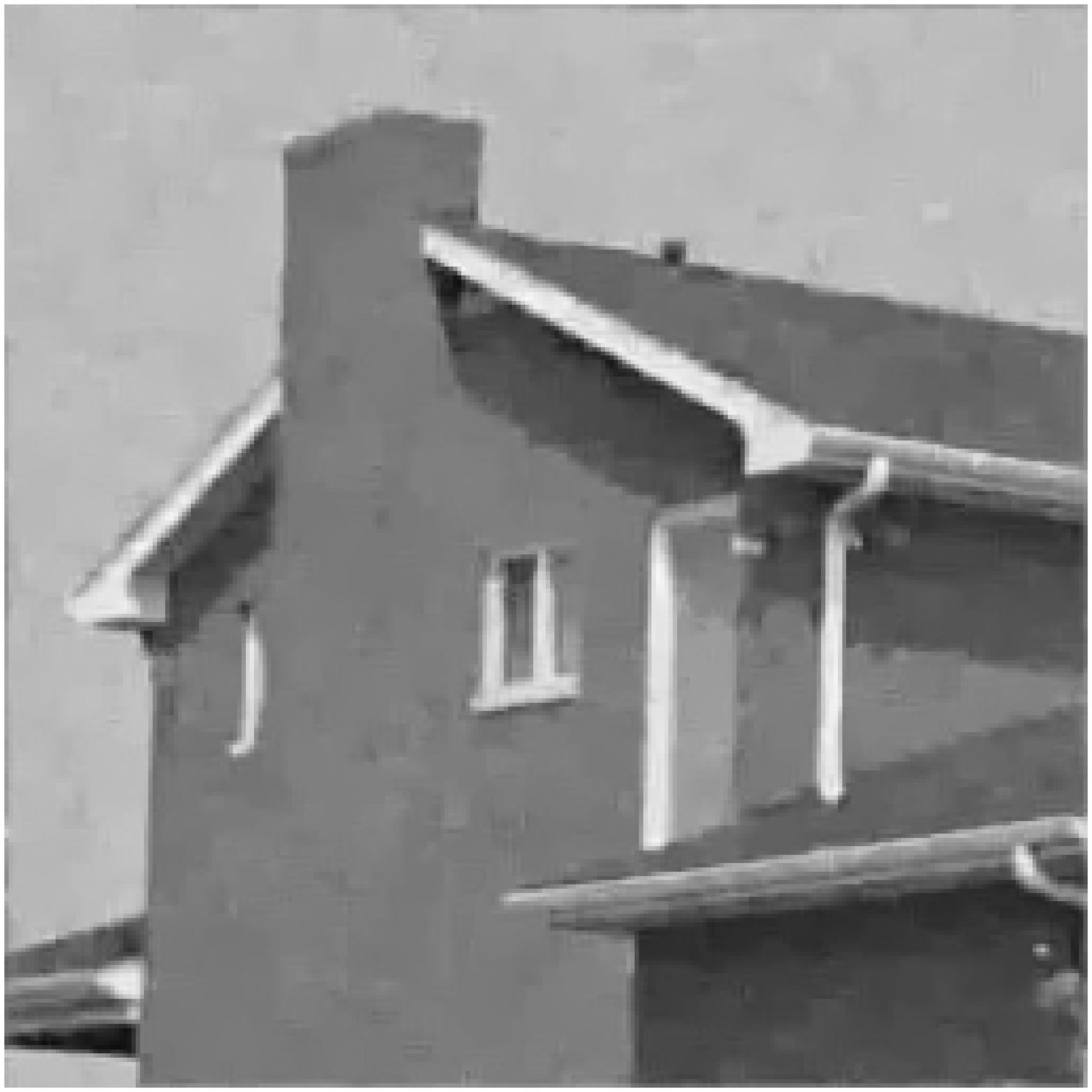}\\
	\caption{Comparison for image denoising. From left to right and up to down are: noisy images and  denoised by  SB, DTCWT, LCHMM, ${l_o}$-WF and ours using equations (\ref{13}) and (\ref{17}), respectively.}
	\label{F；D2}
\end{figure}

\begin{figure}[htbp]
	\centering
	\includegraphics[width=0.32\textwidth]{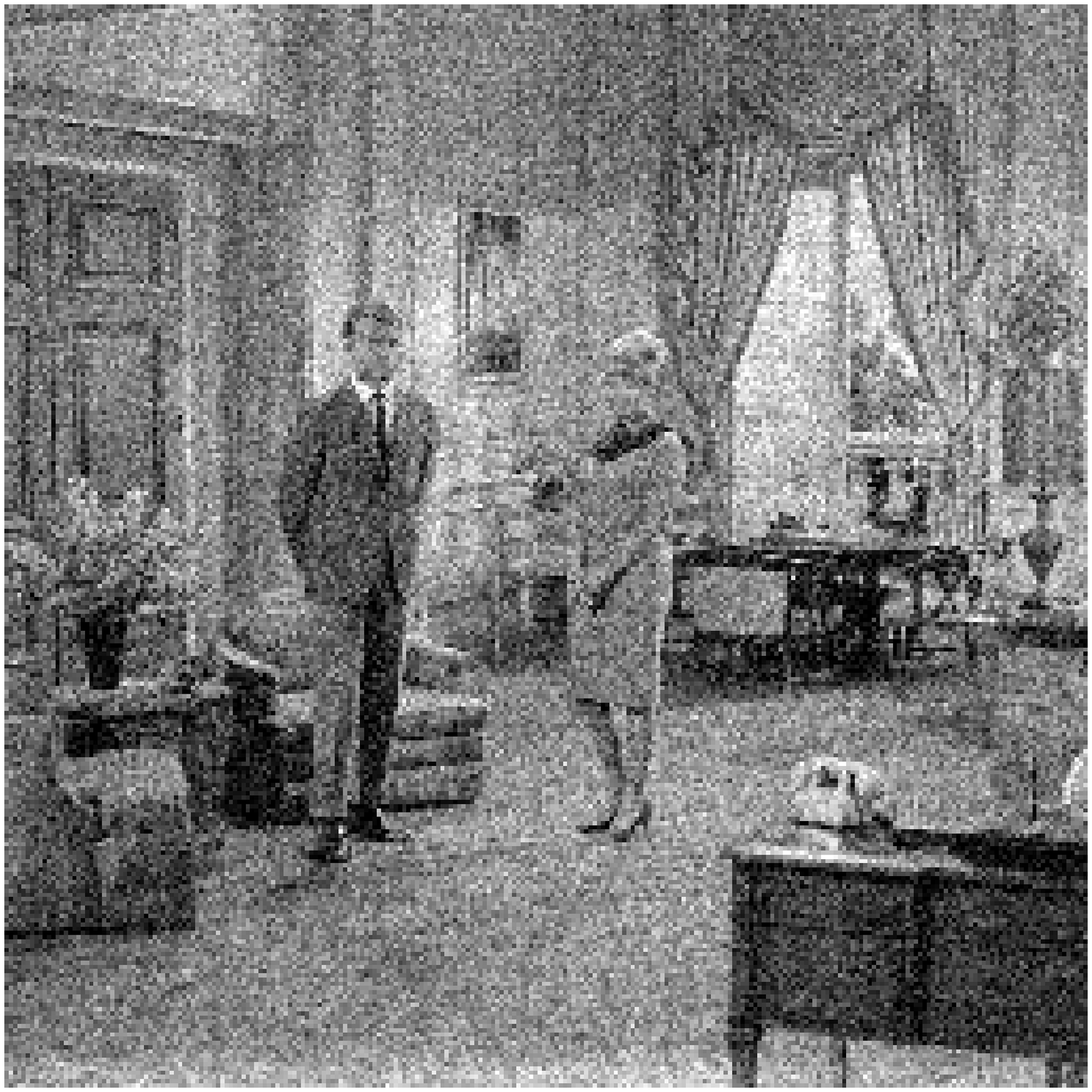}\\
	\includegraphics[width=0.32\textwidth]{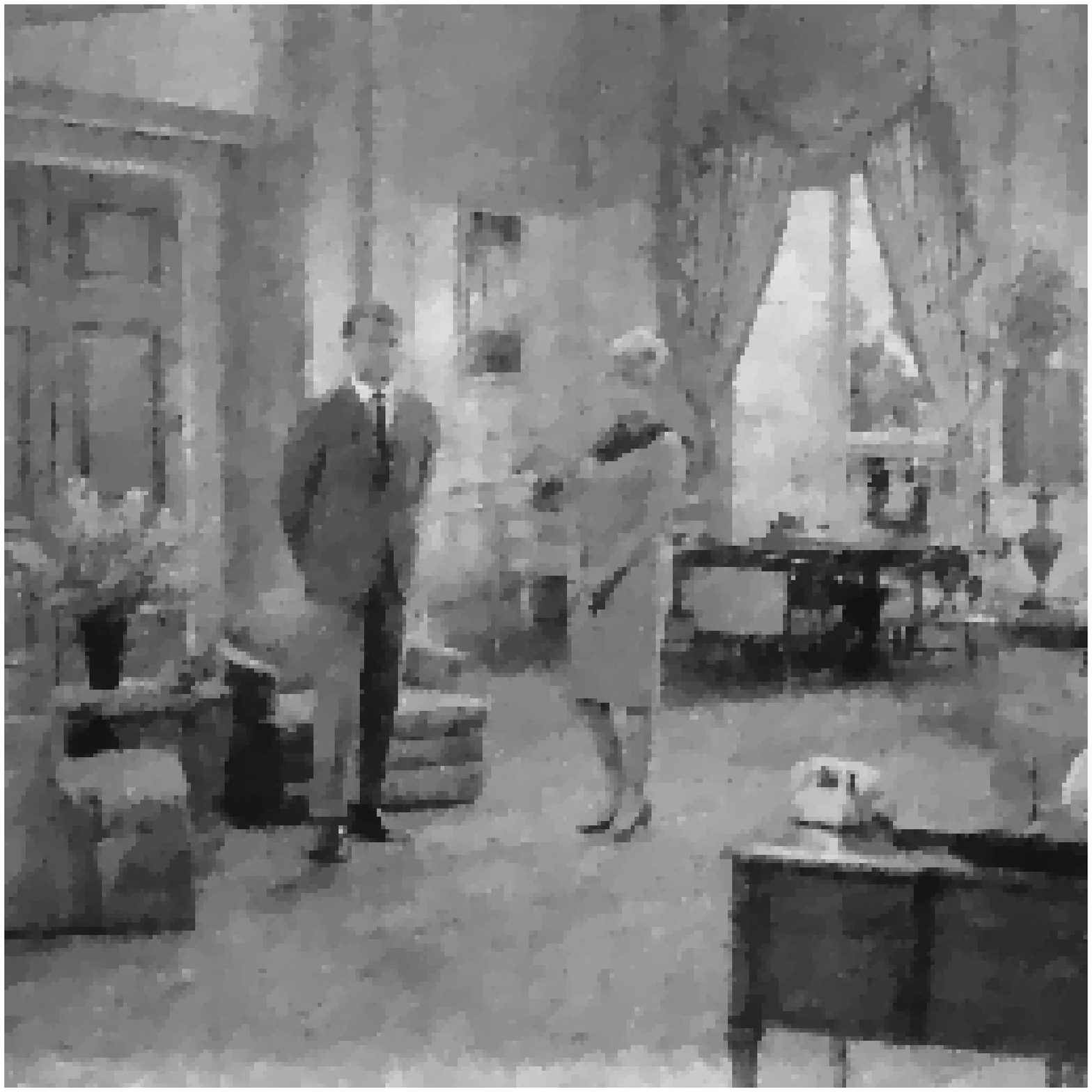}
	\includegraphics[width=0.32\textwidth]{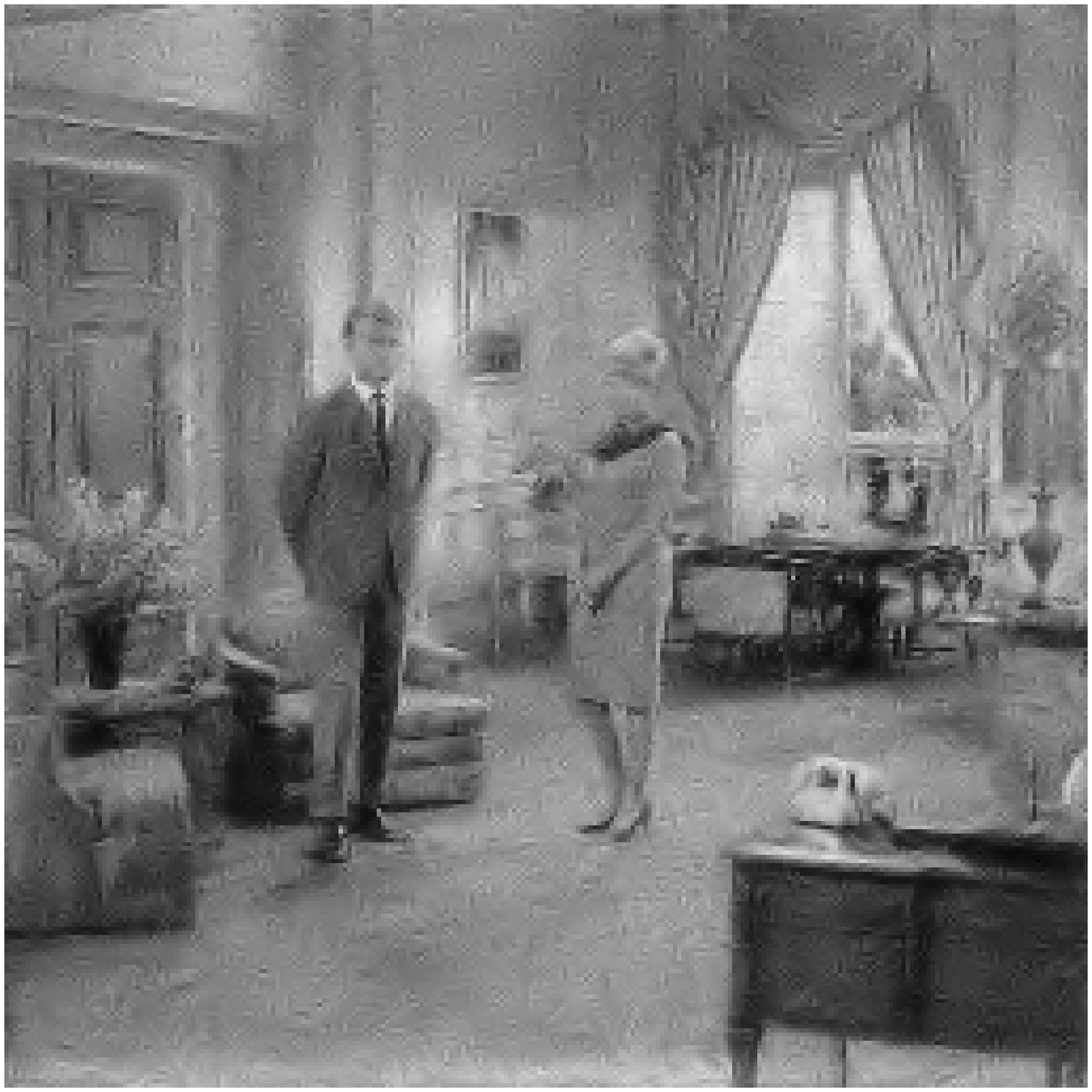}
	\includegraphics[width=0.32\textwidth]{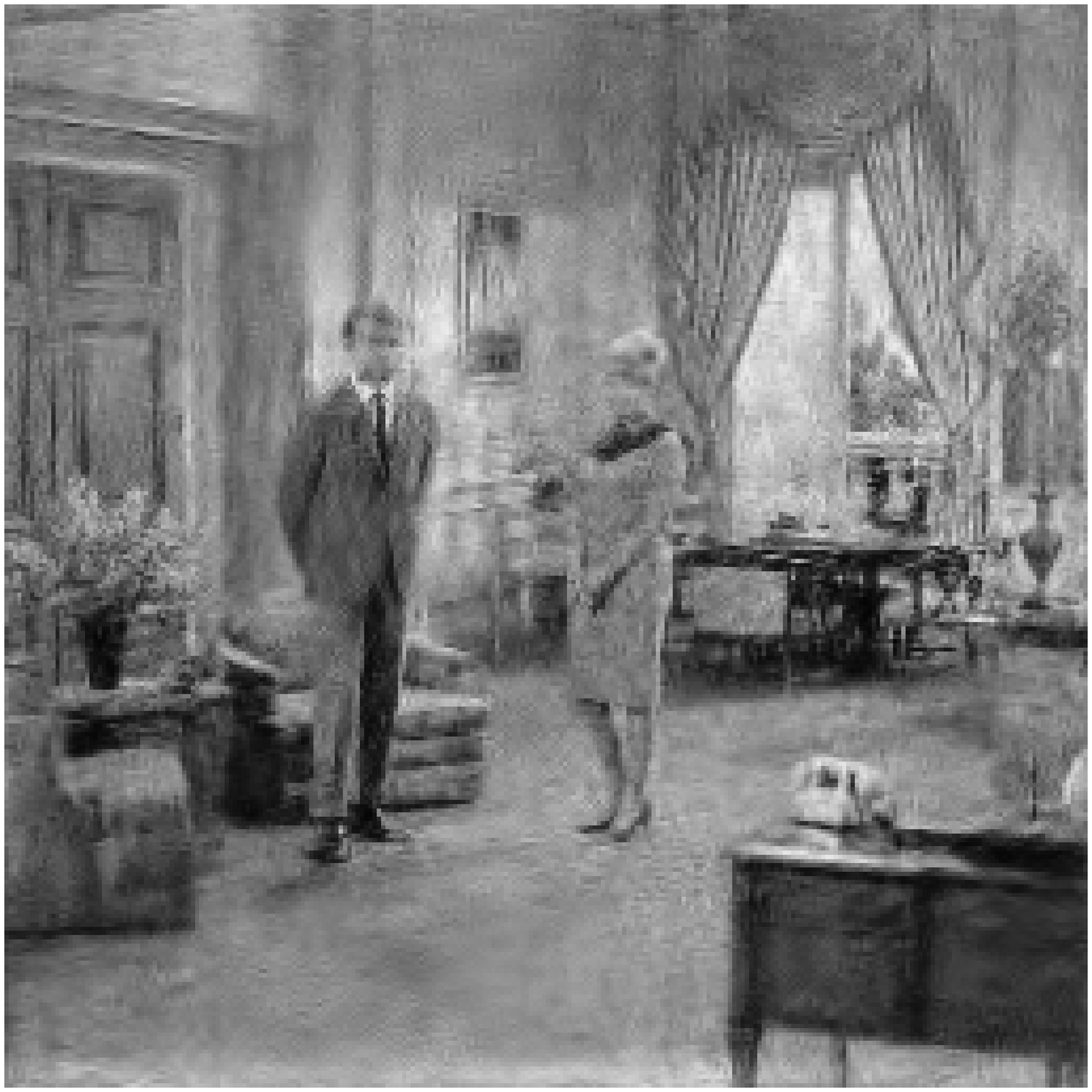}\\
	\includegraphics[width=0.32\textwidth]{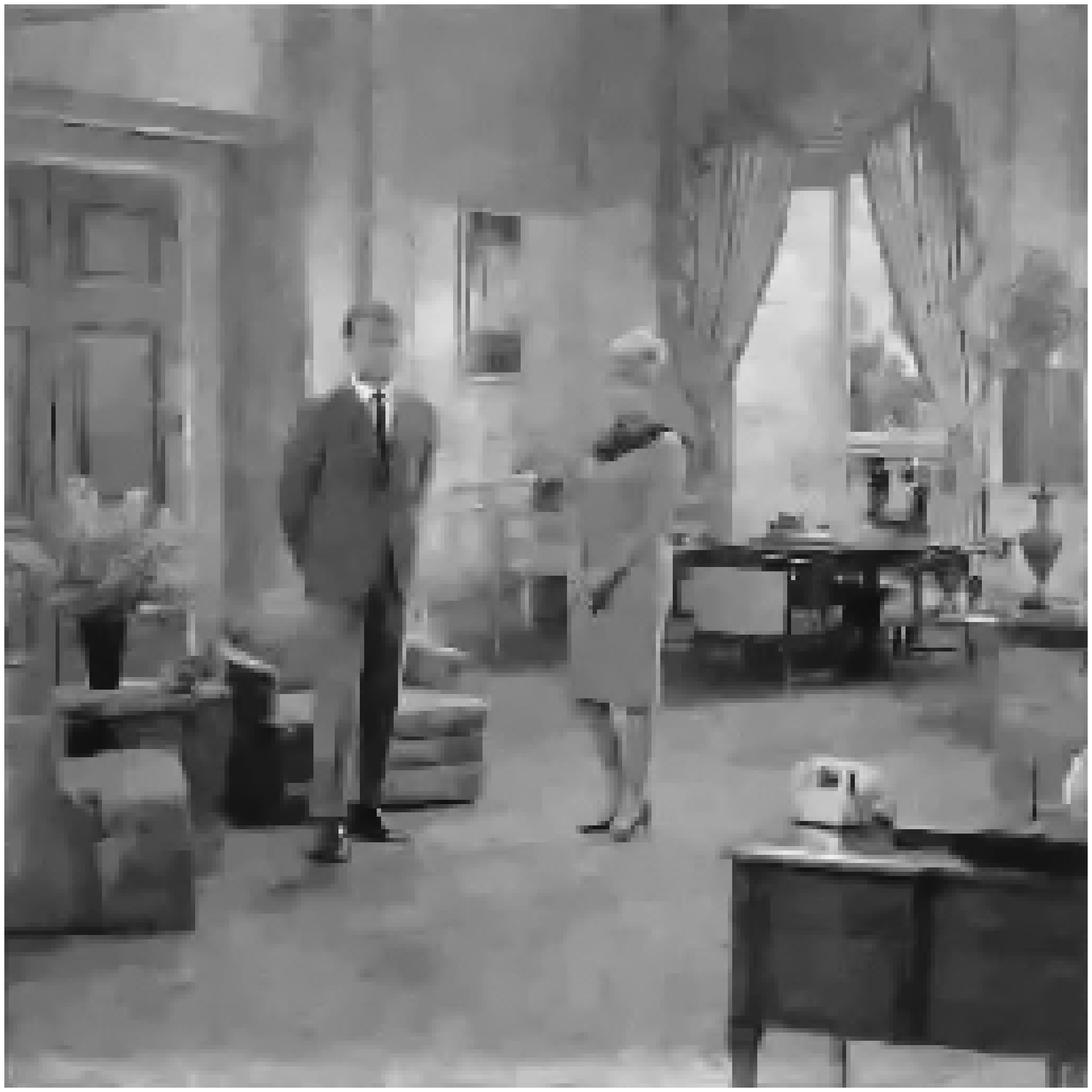}
	\includegraphics[width=0.32\textwidth]{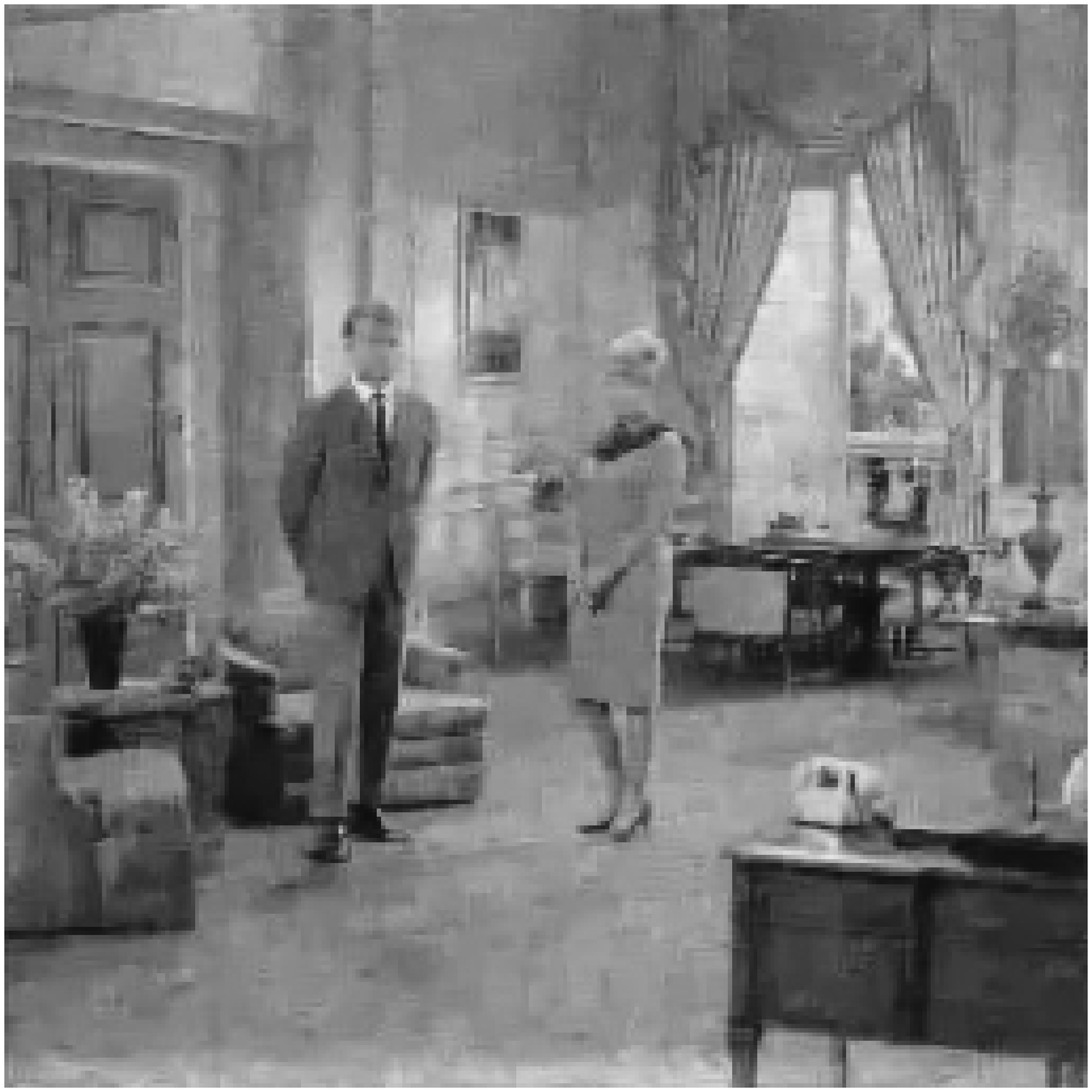}
	\includegraphics[width=0.32\textwidth]{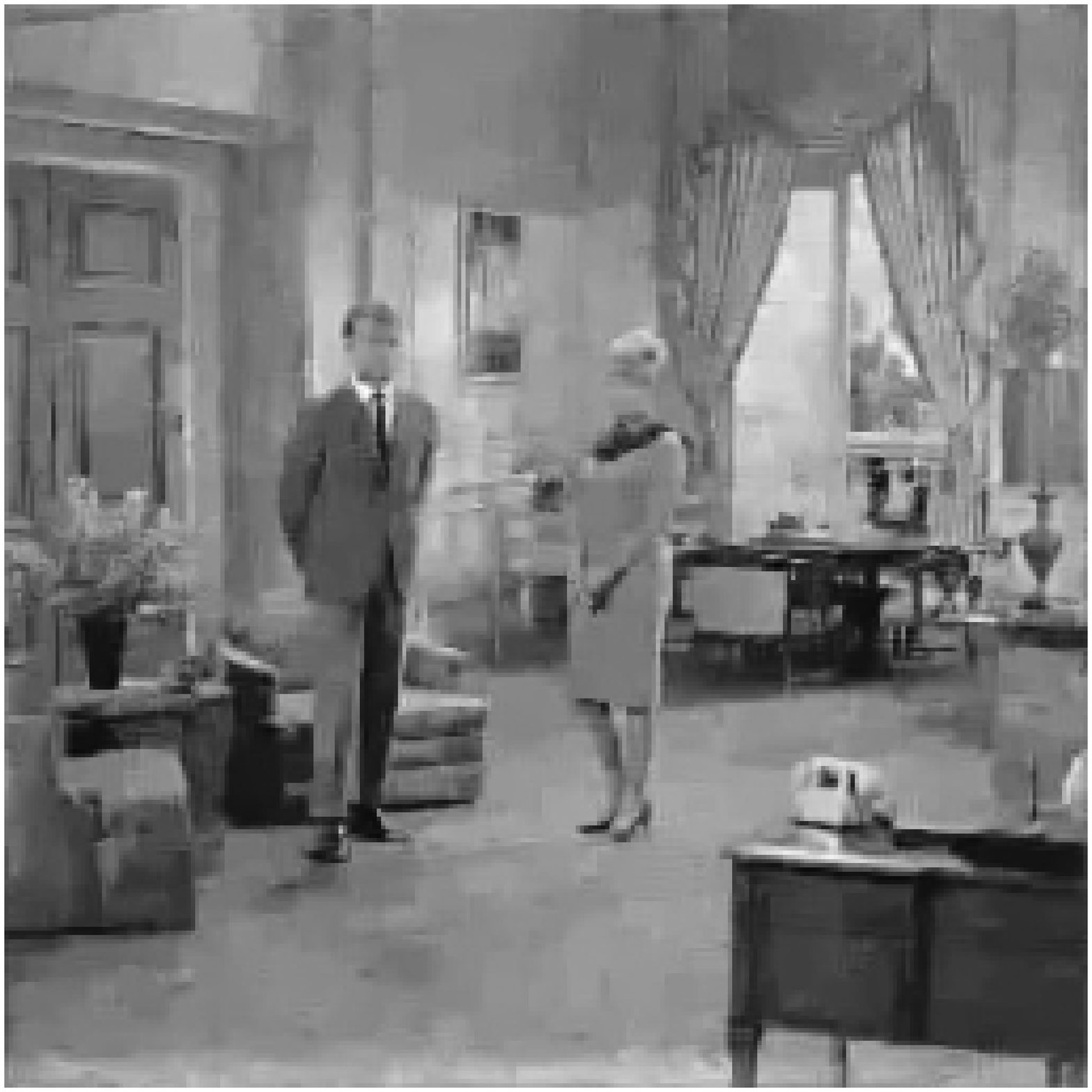}\\
		\caption{Comparison for image denoising. From left to right and up to down are: noisy images and  denoised by  SB, DTCWT, LCHMM, ${l_o}$-WF and ours using equations (\ref{13}) and (\ref{17}), respectively.}
	\label{F；D3}
\end{figure}

\begin{figure}[htbp]
	\centering	
	\includegraphics[width=0.32\textwidth]{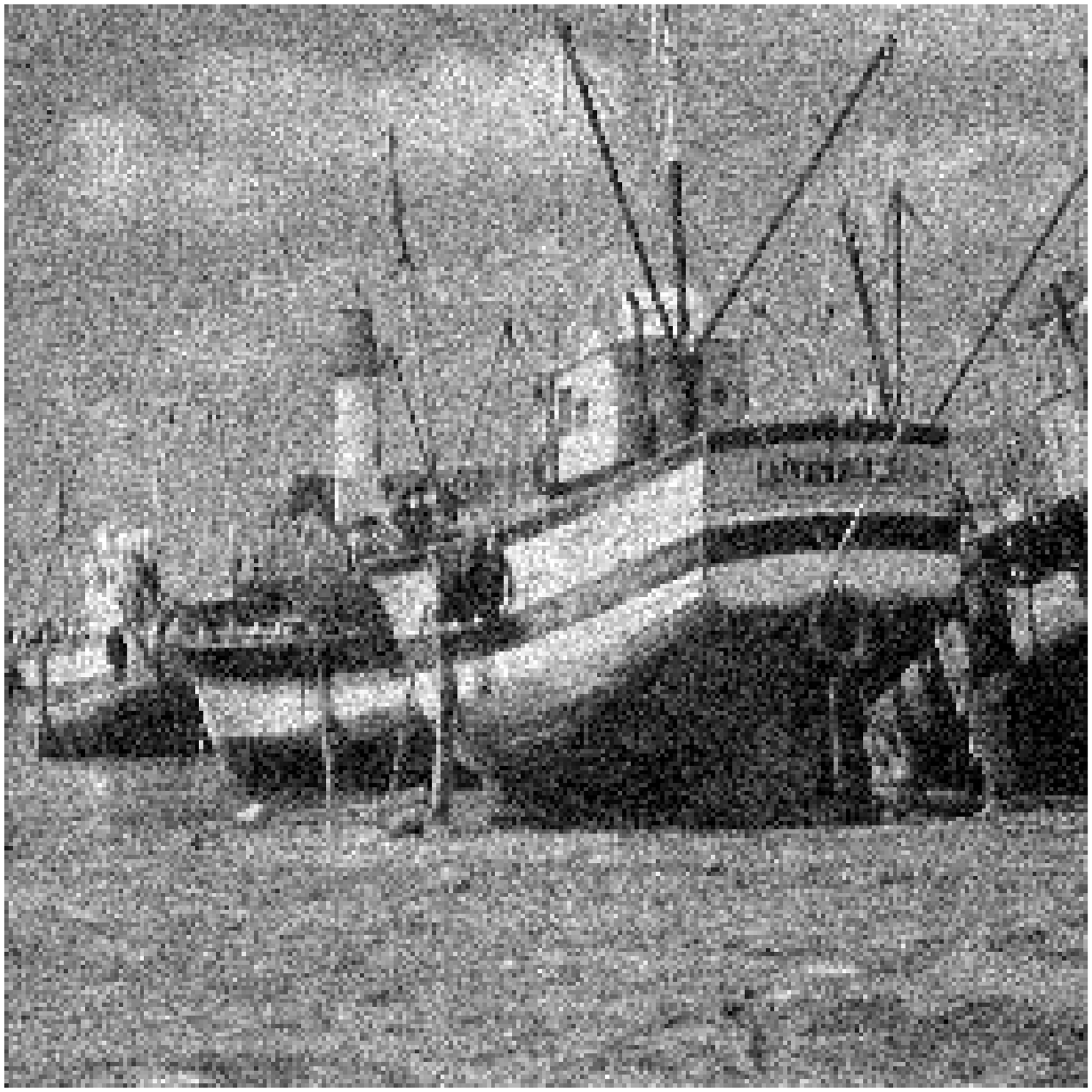}\\
	\includegraphics[width=0.32\textwidth]{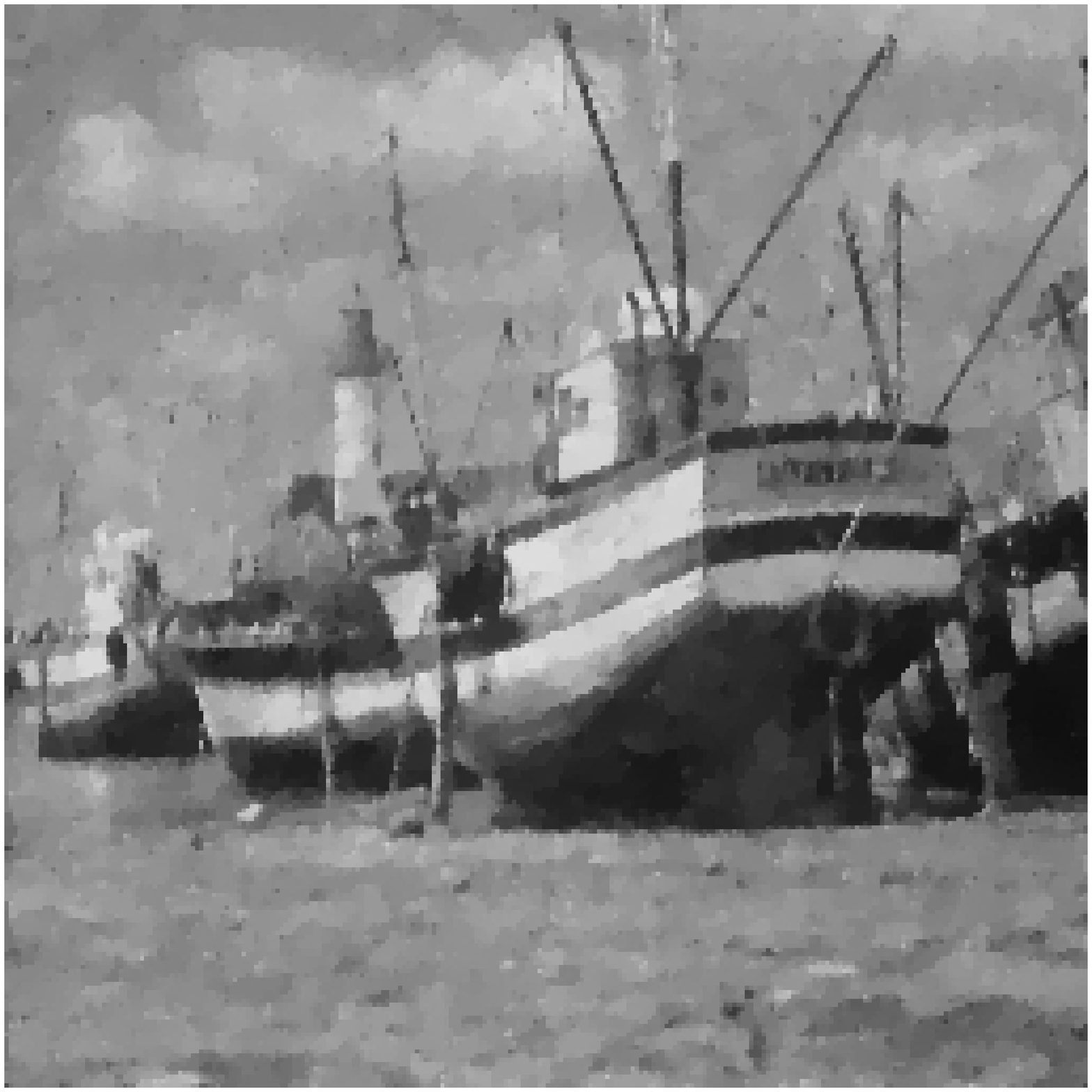}
	\includegraphics[width=0.32\textwidth]{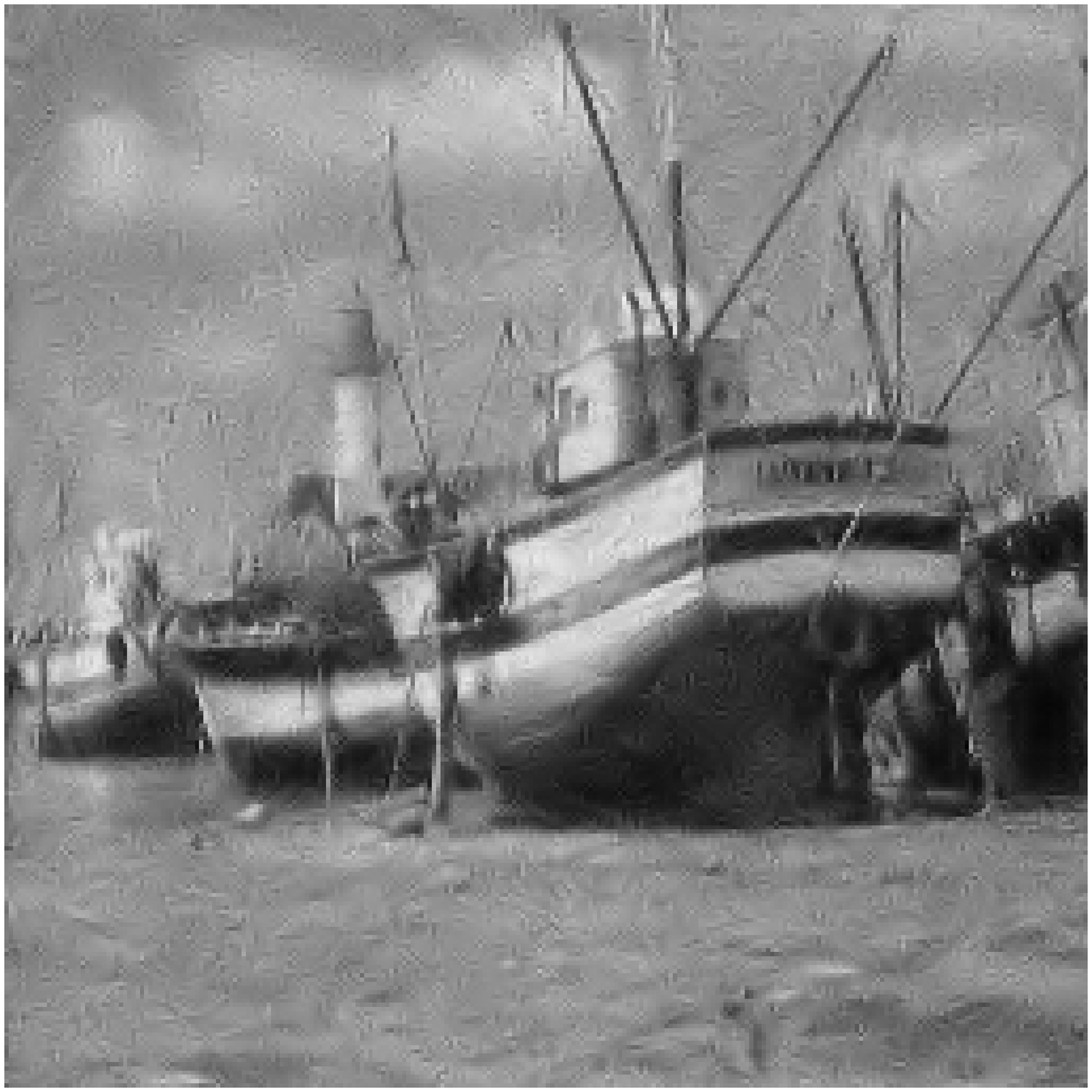}
	\includegraphics[width=0.32\textwidth]{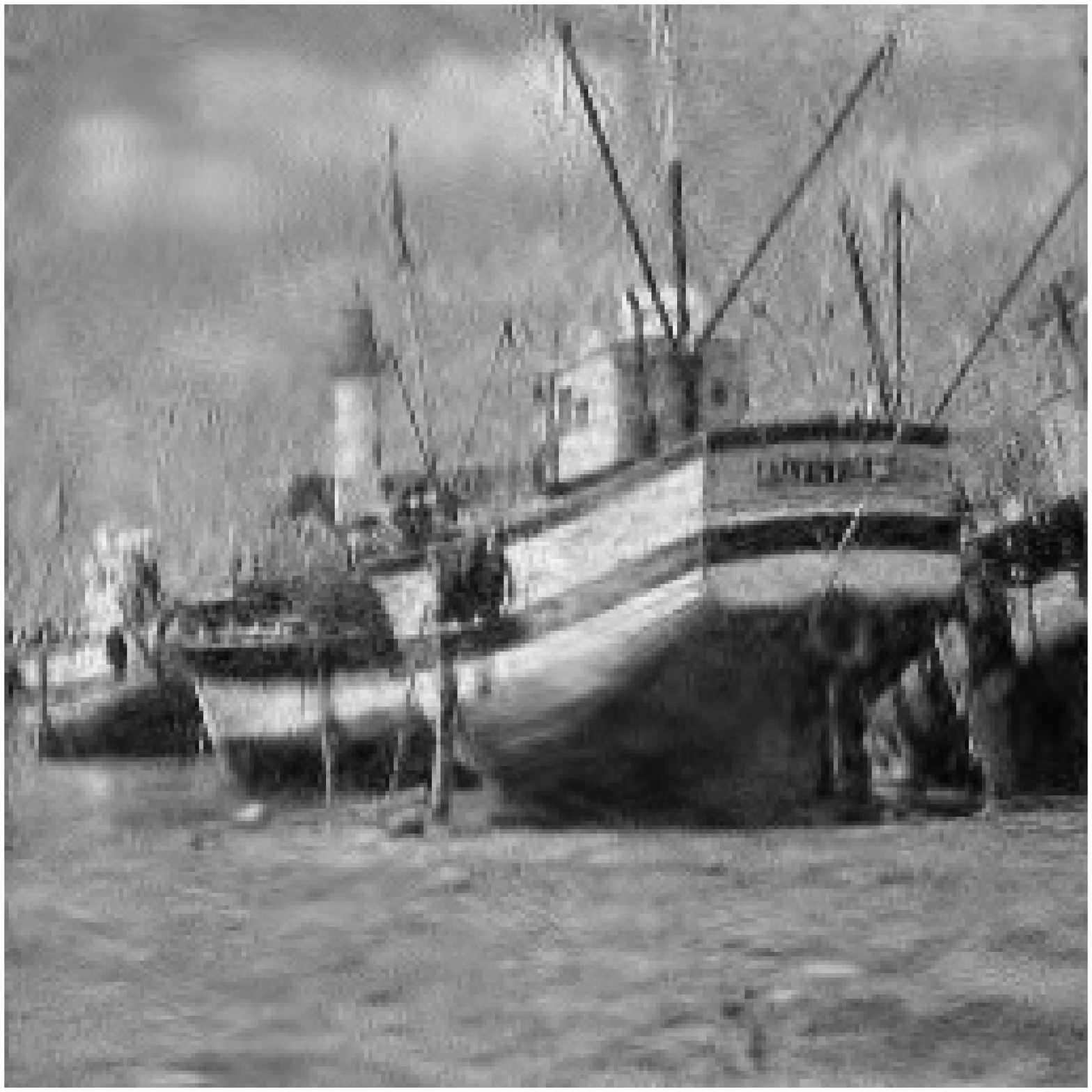}\\
	\includegraphics[width=0.32\textwidth]{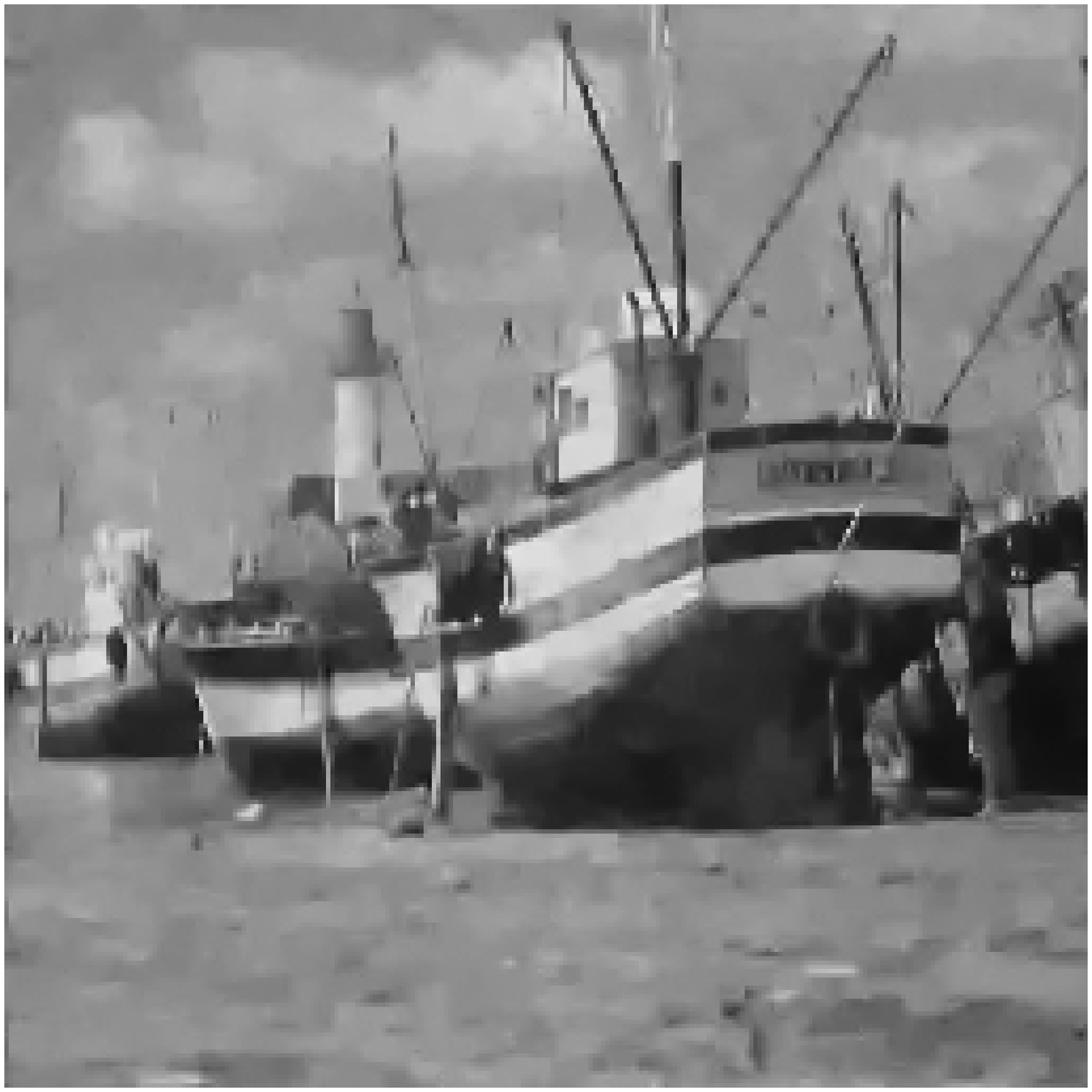}
	\includegraphics[width=0.32\textwidth]{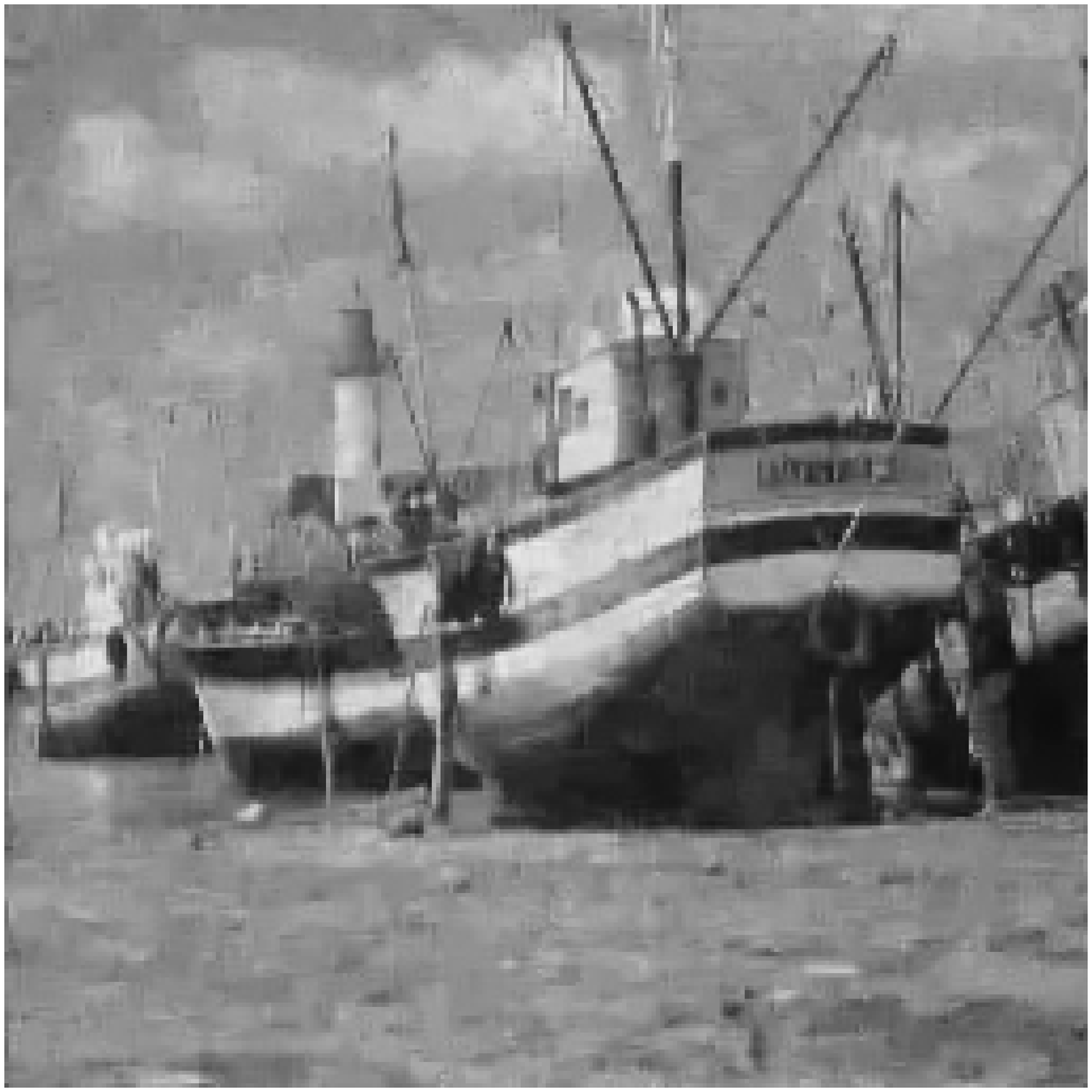}
	\includegraphics[width=0.32\textwidth]{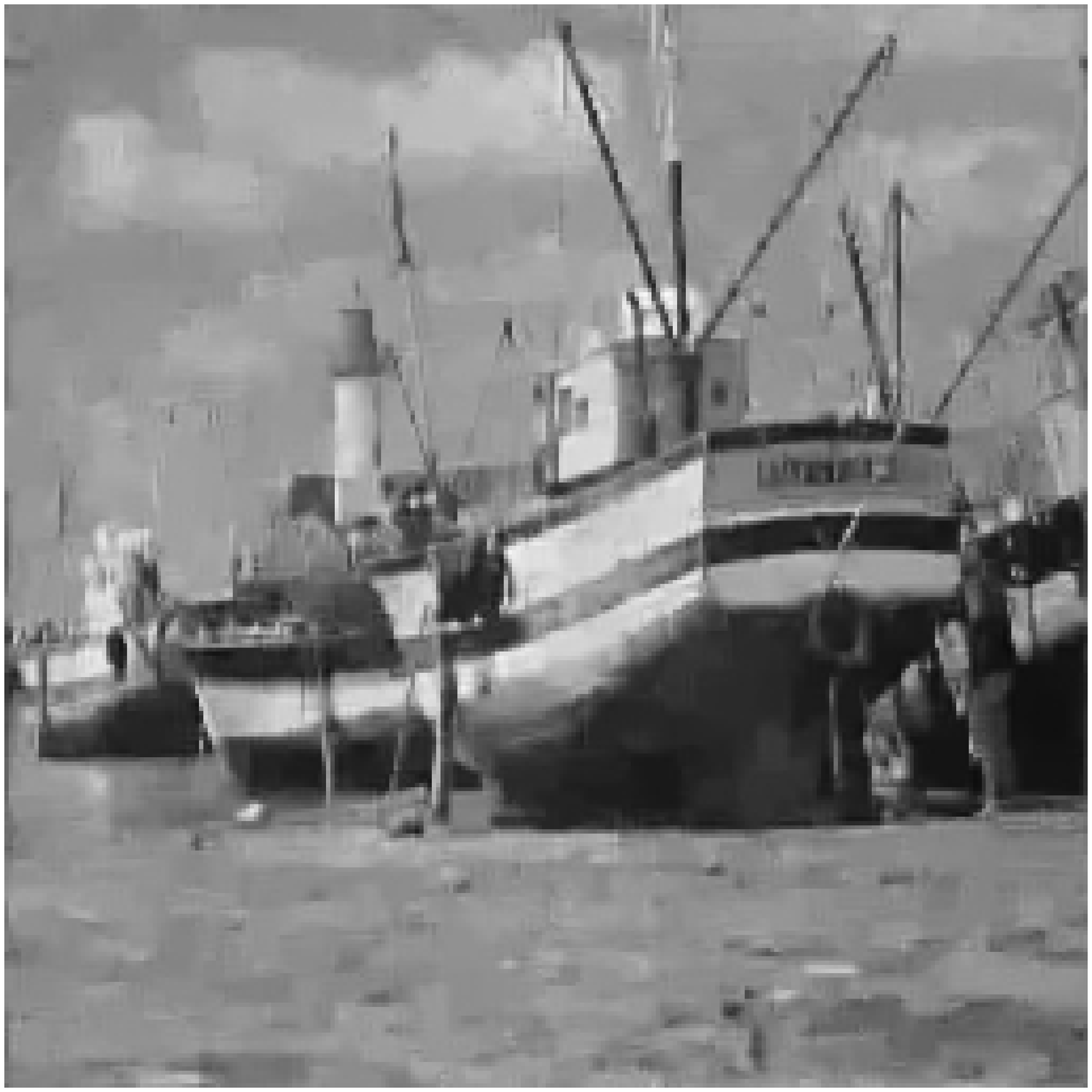}\\
	\caption{Comparison for image denoising. From left to right and up to down are: noisy images and  denoised by  SB, DTCWT, LCHMM, ${l_o}$-WF and ours using equations (\ref{13}) and (\ref{17}), respectively.}
	\label{F；D4}
\end{figure}

\begin{figure}[htbp]
	\centering
	\includegraphics[width=0.32\textwidth]{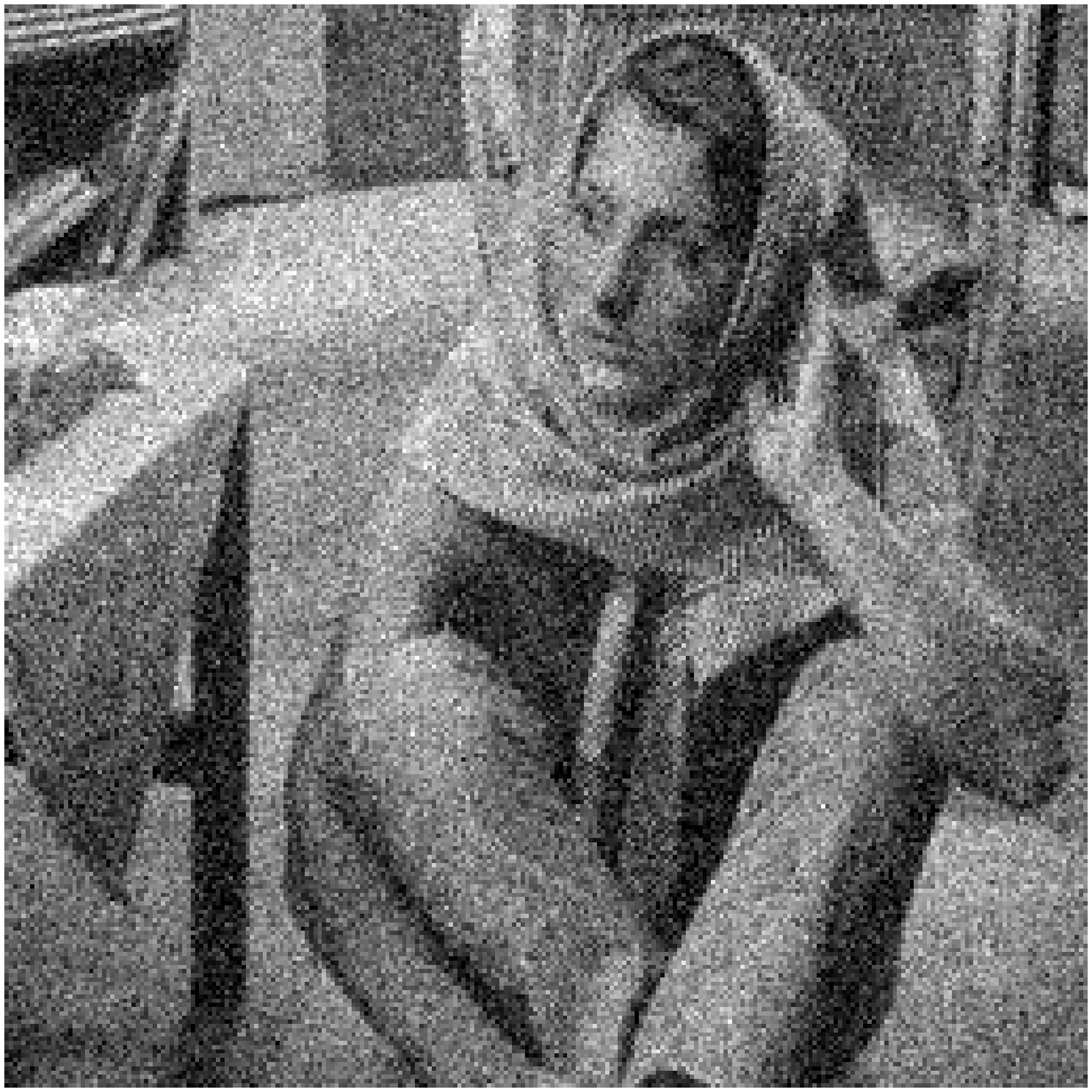}\\
	\includegraphics[width=0.32\textwidth]{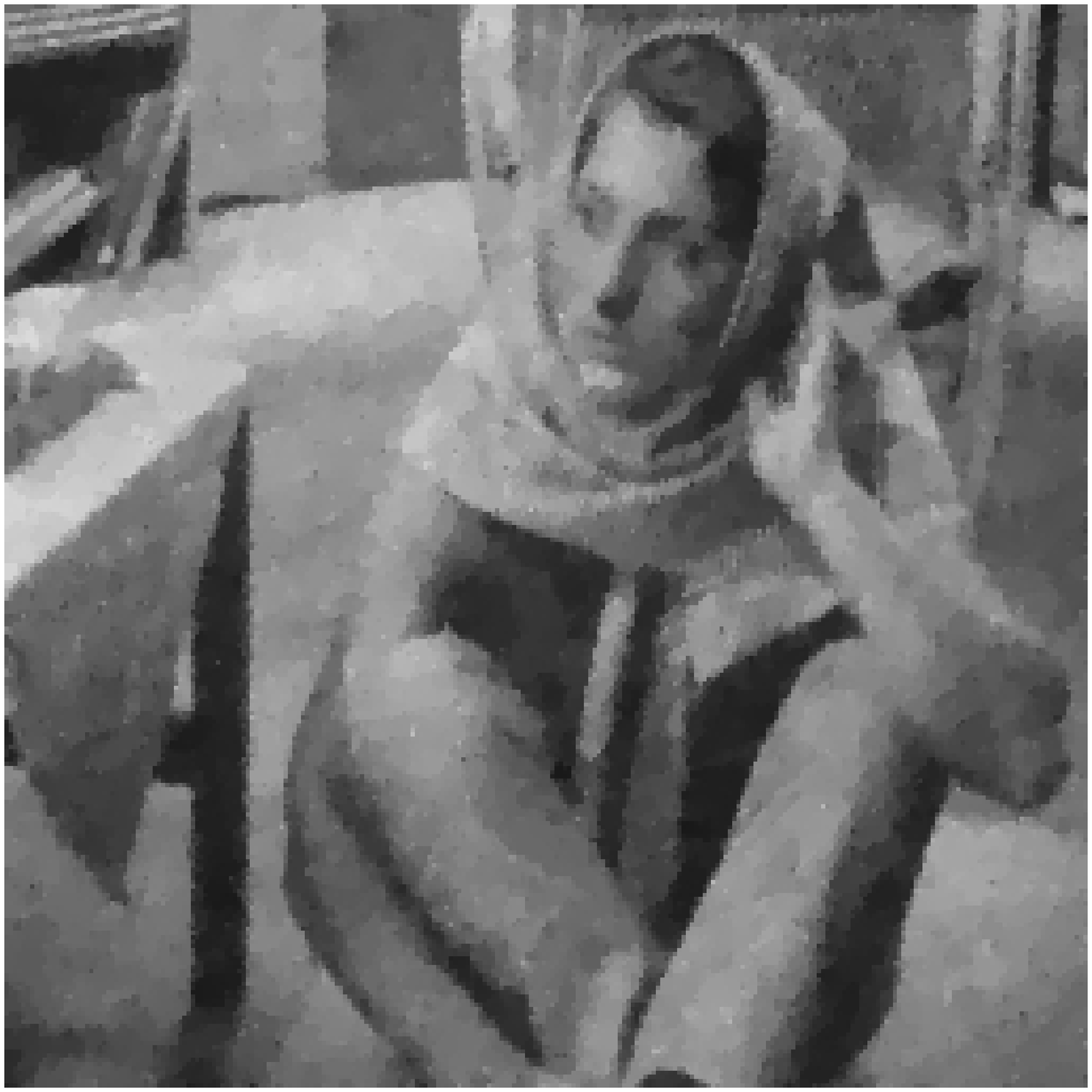}
	\includegraphics[width=0.32\textwidth]{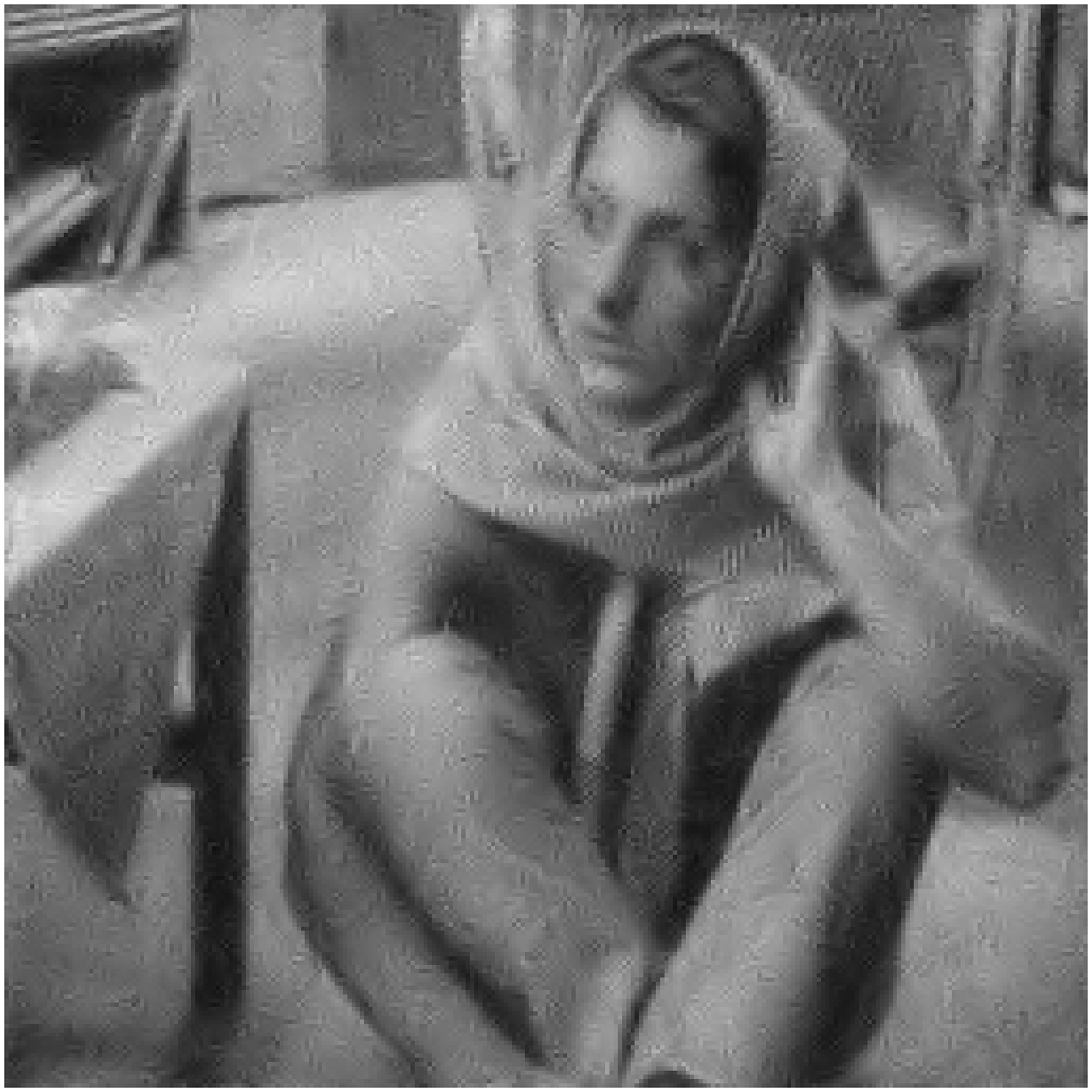}
	\includegraphics[width=0.32\textwidth]{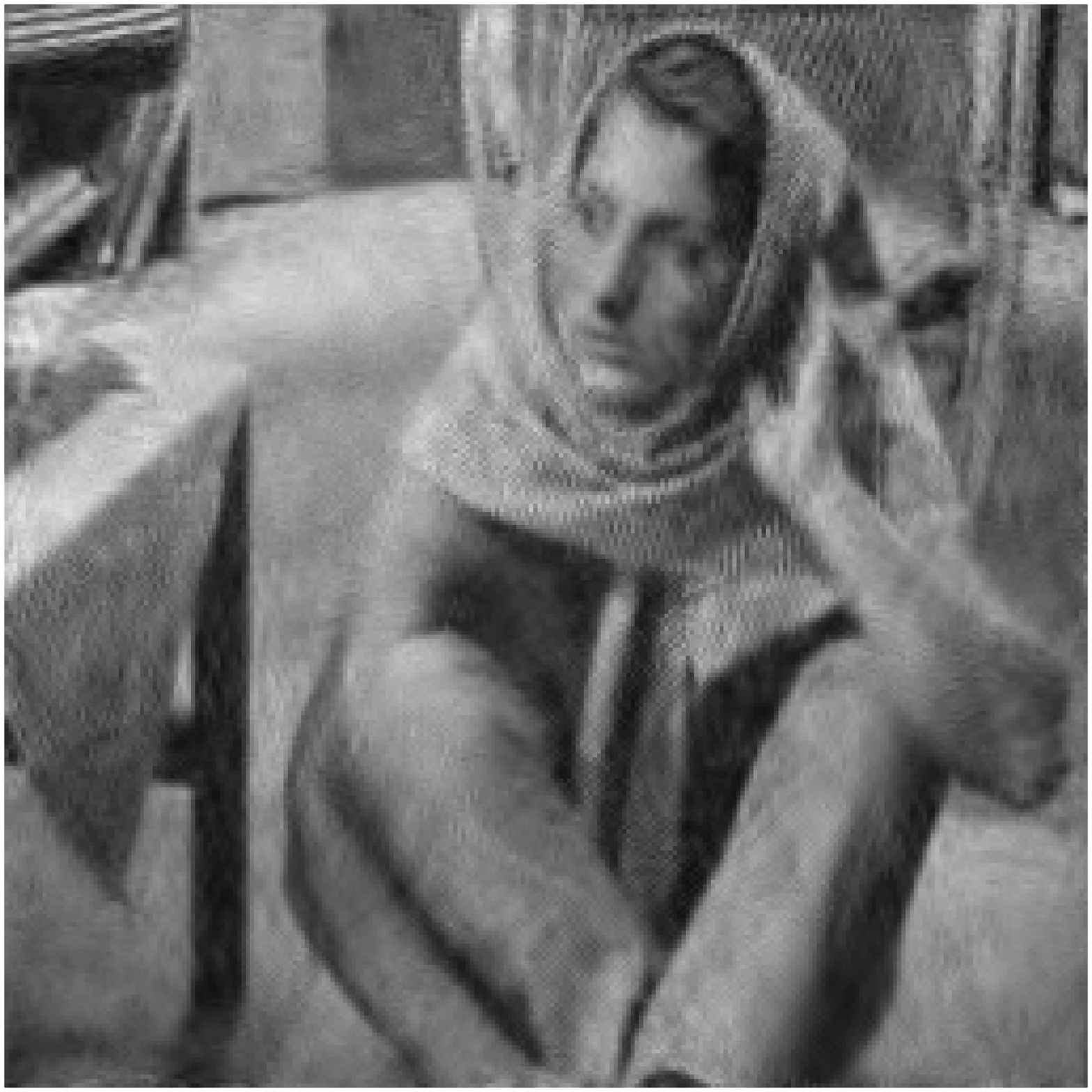}\\
	\includegraphics[width=0.32\textwidth]{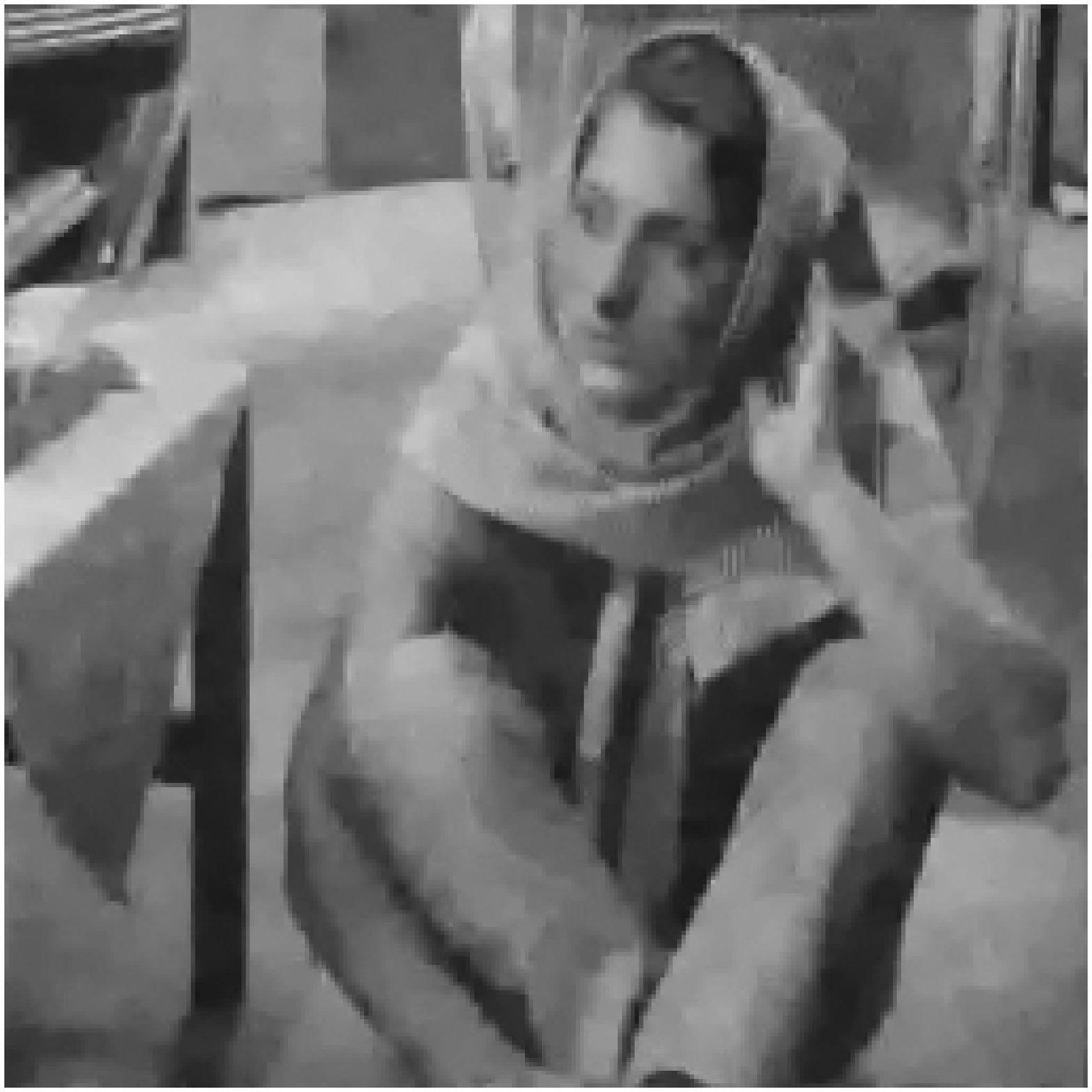}
	\includegraphics[width=0.32\textwidth]{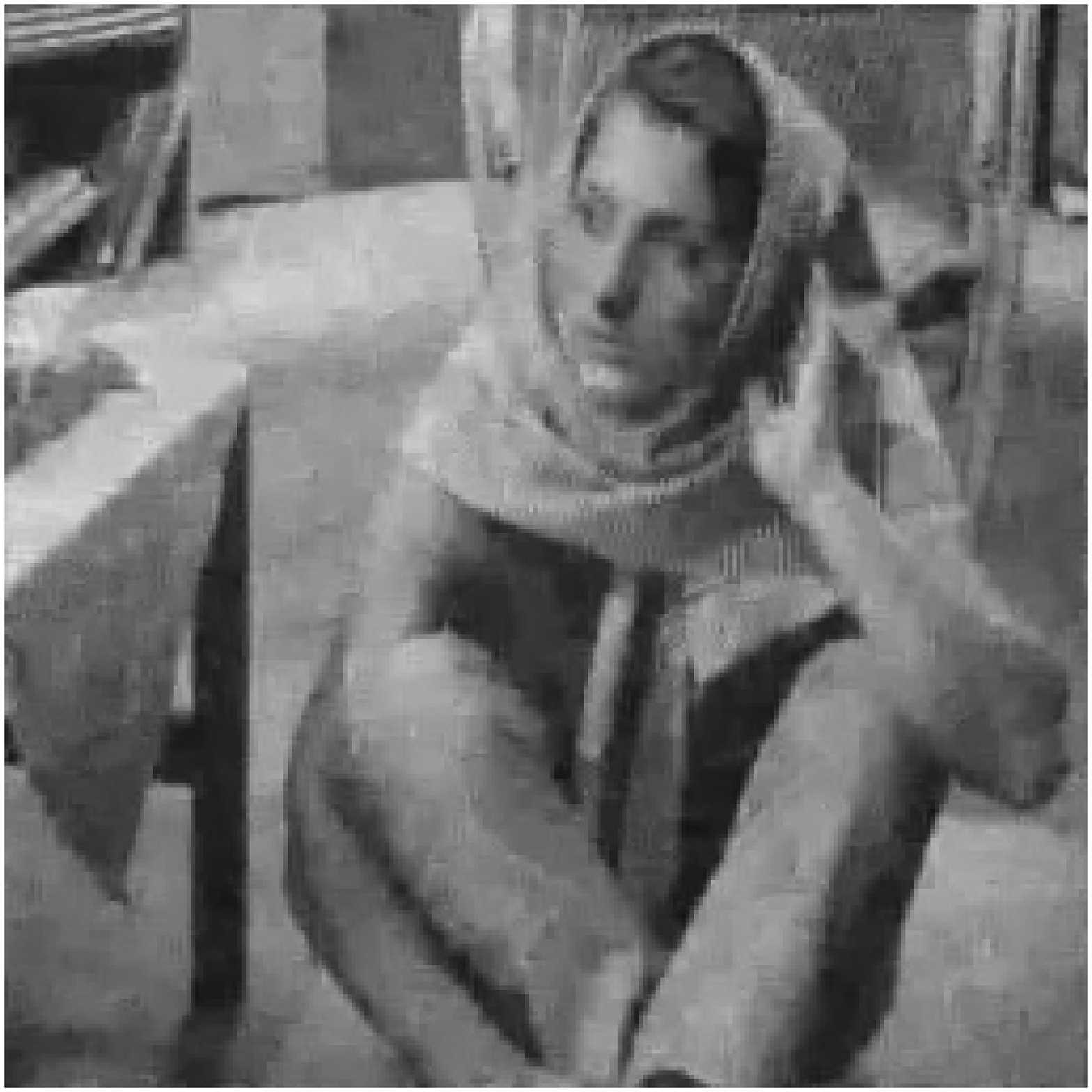}
	\includegraphics[width=0.32\textwidth]{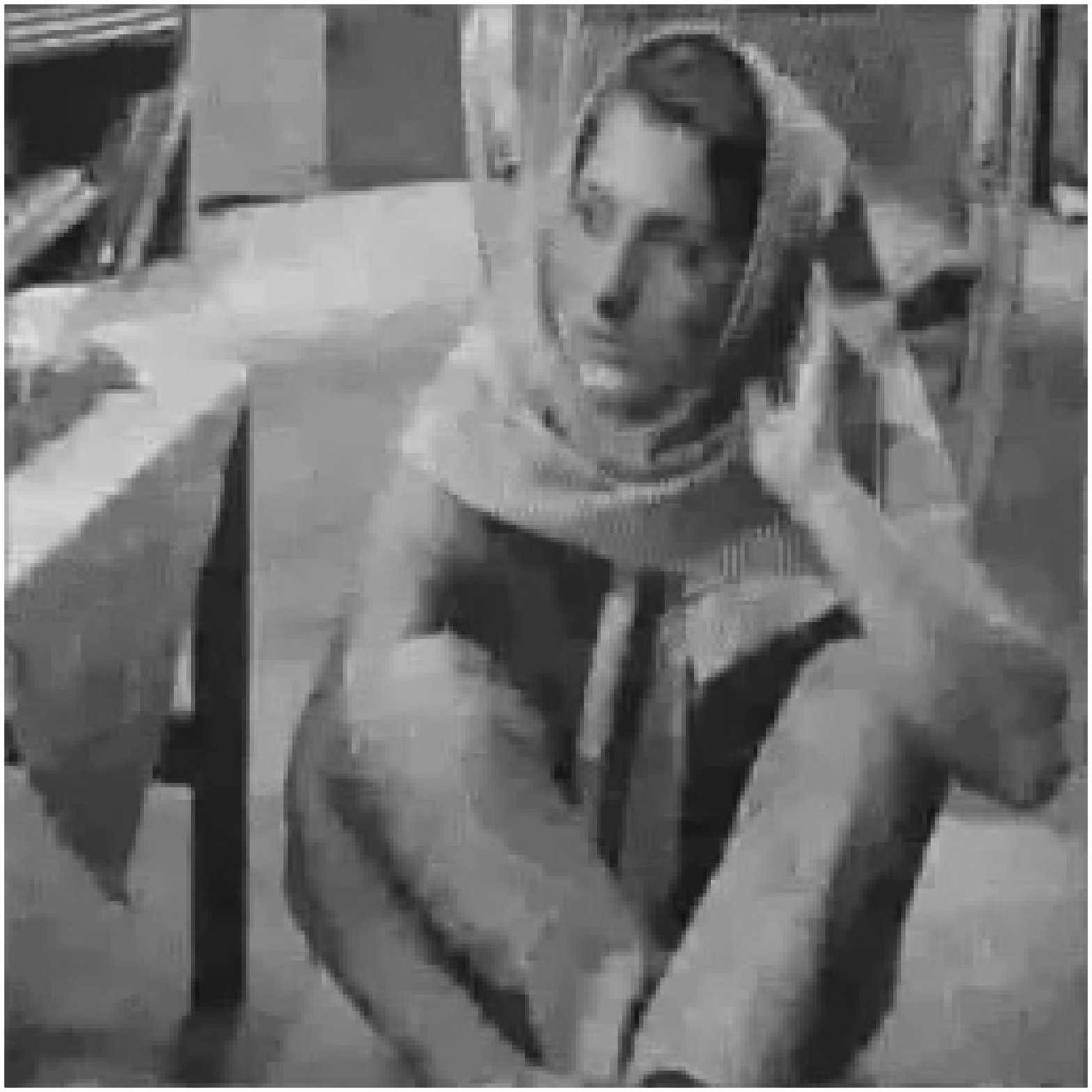}\\
		\caption{Comparison for image denoising. From left to right and up to down are: noisy images and  denoised by  SB, DTCWT, LCHMM, ${l_o}$-WF and ours using equations (\ref{13}) and (\ref{17}), respectively.}
	\label{F；D5}
\end{figure}
	
\begin{figure}[htbp]
	\centering	
	\includegraphics[width=0.32\textwidth]{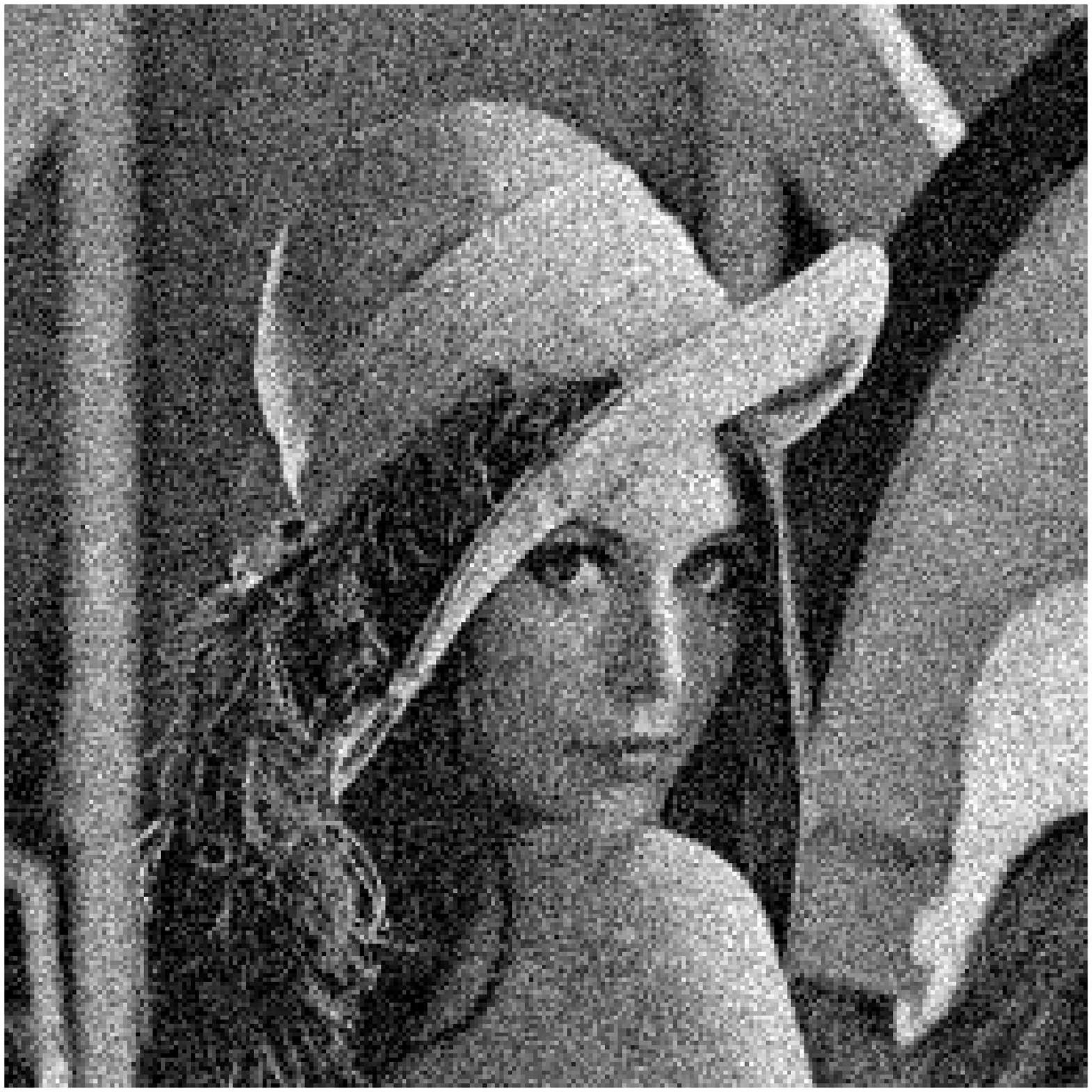}\\
	\includegraphics[width=0.32\textwidth]{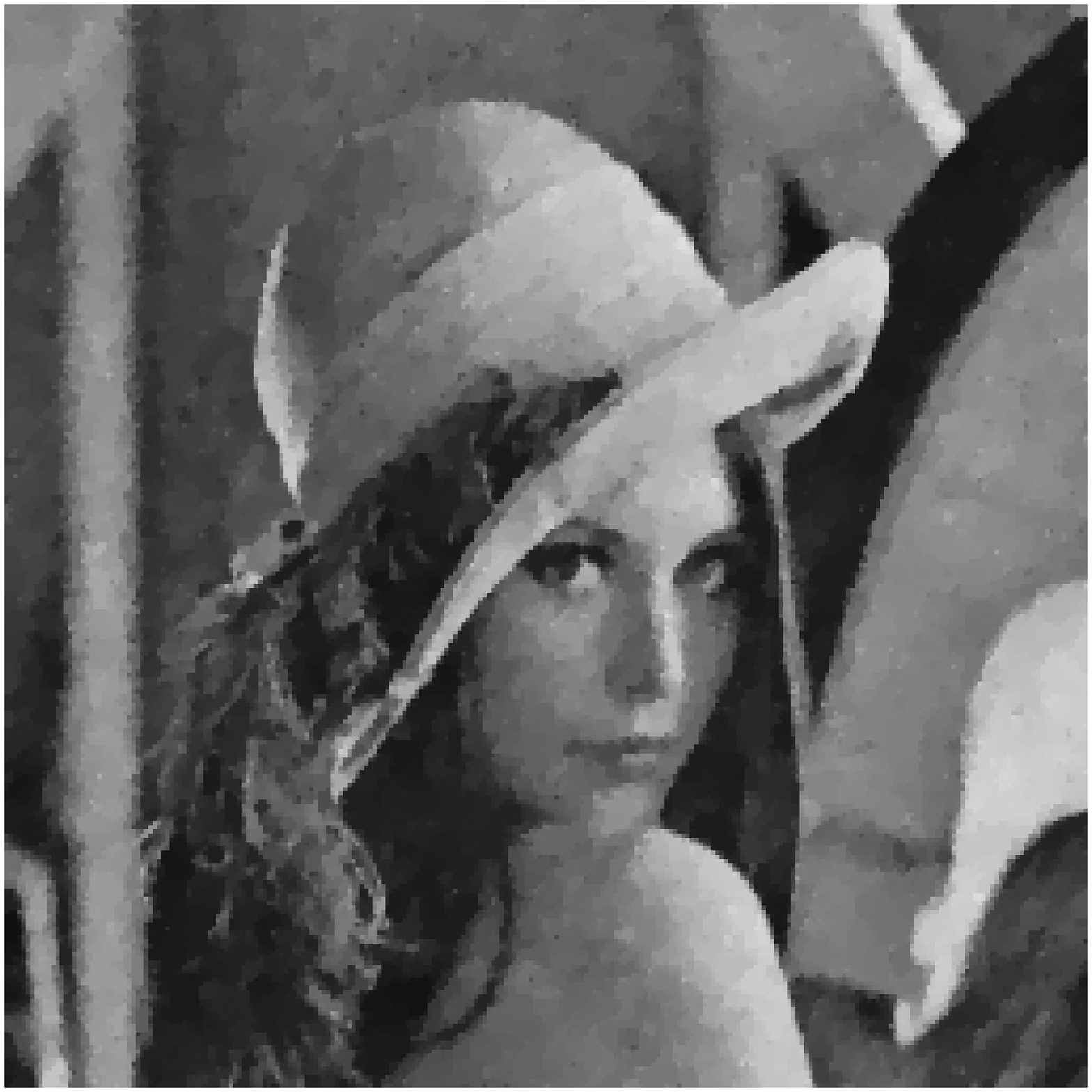}
	\includegraphics[width=0.32\textwidth]{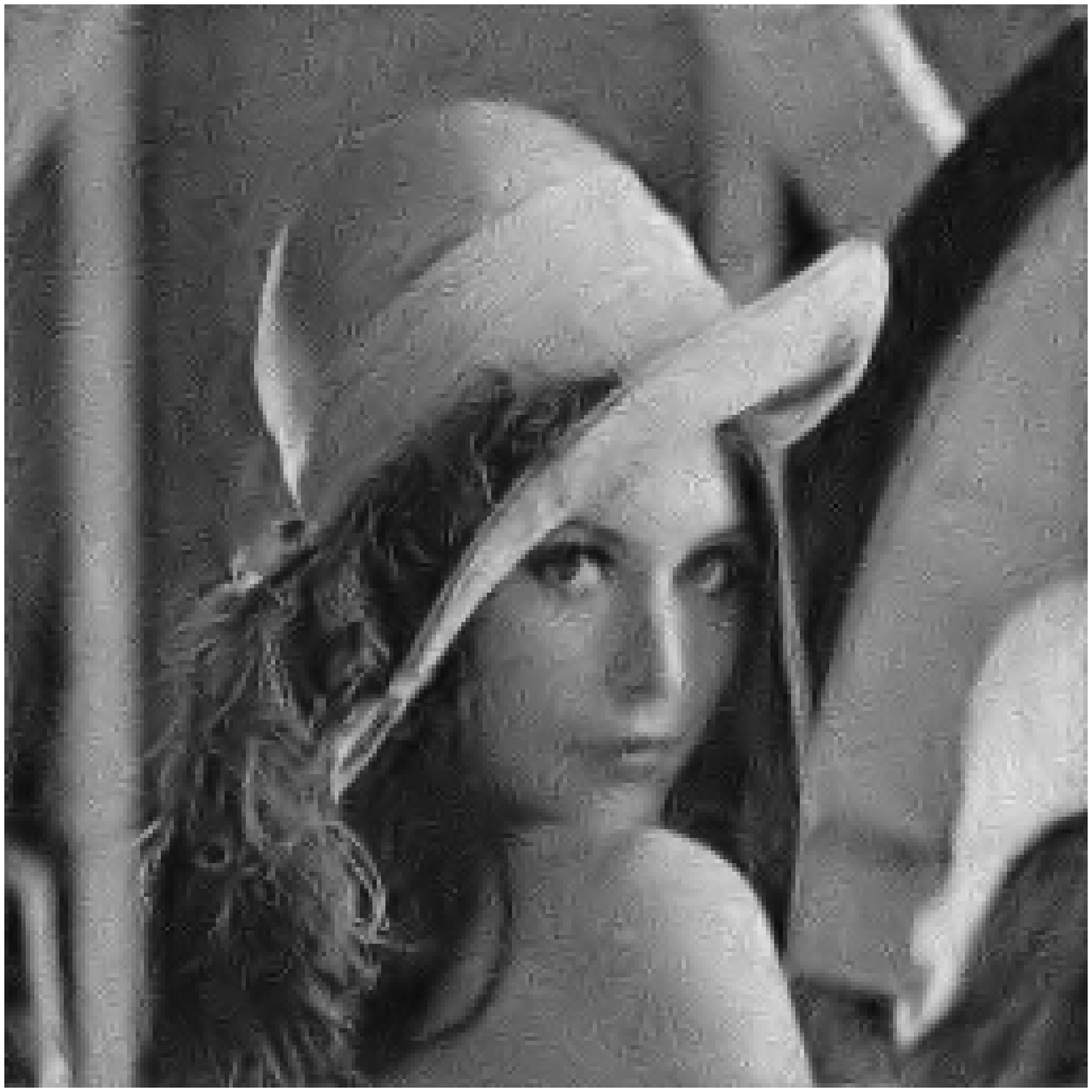}
	\includegraphics[width=0.32\textwidth]{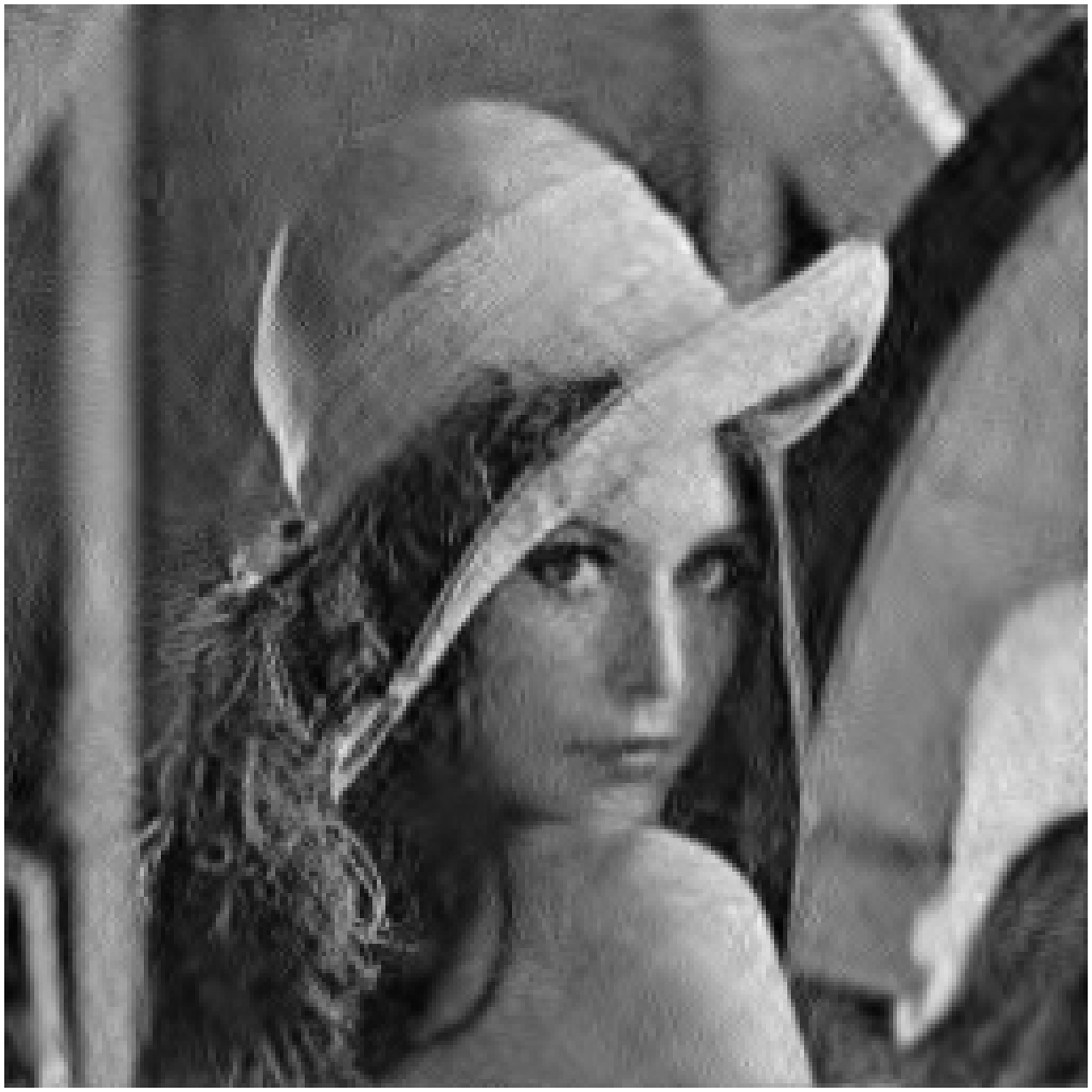}\\
	\includegraphics[width=0.32\textwidth]{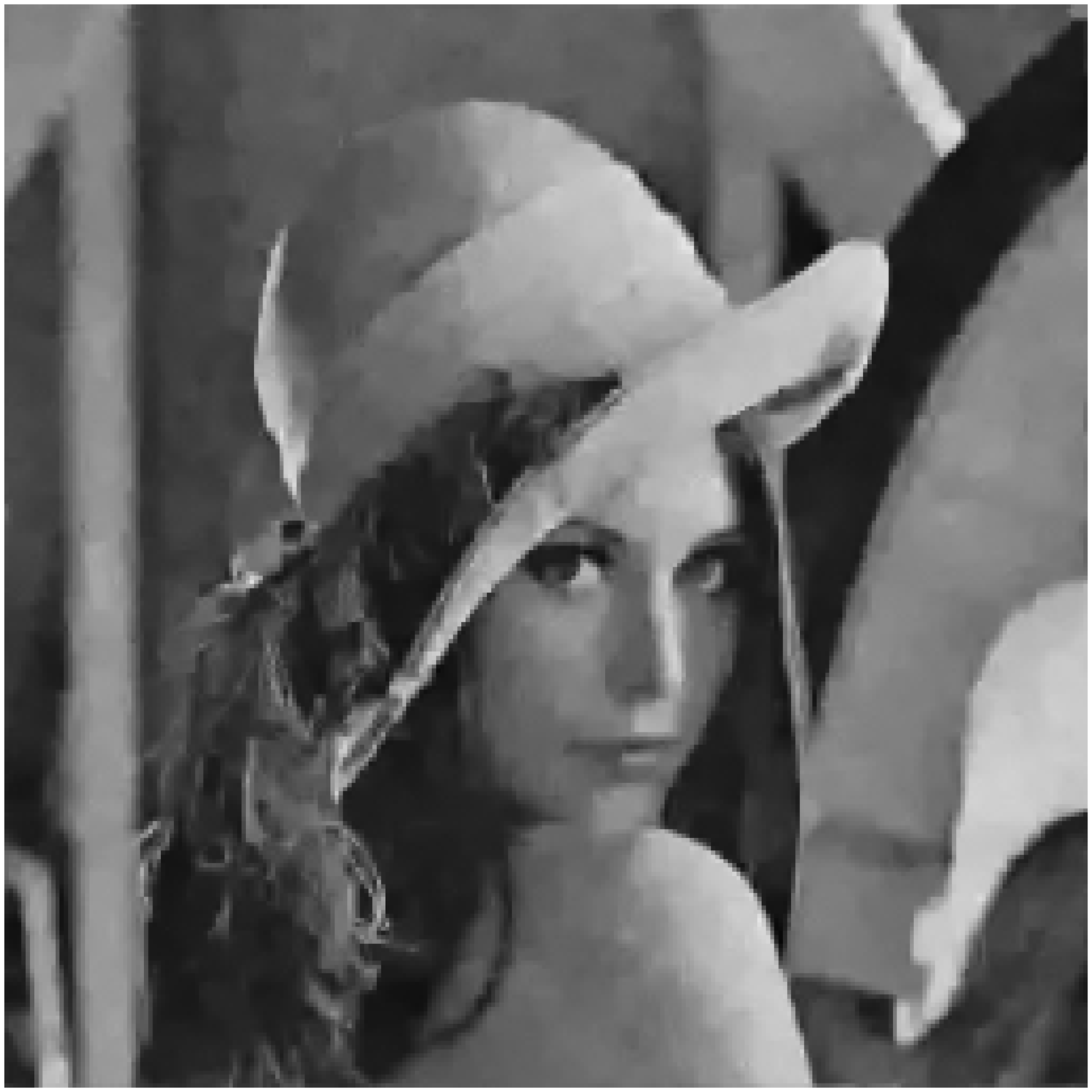}
	\includegraphics[width=0.32\textwidth]{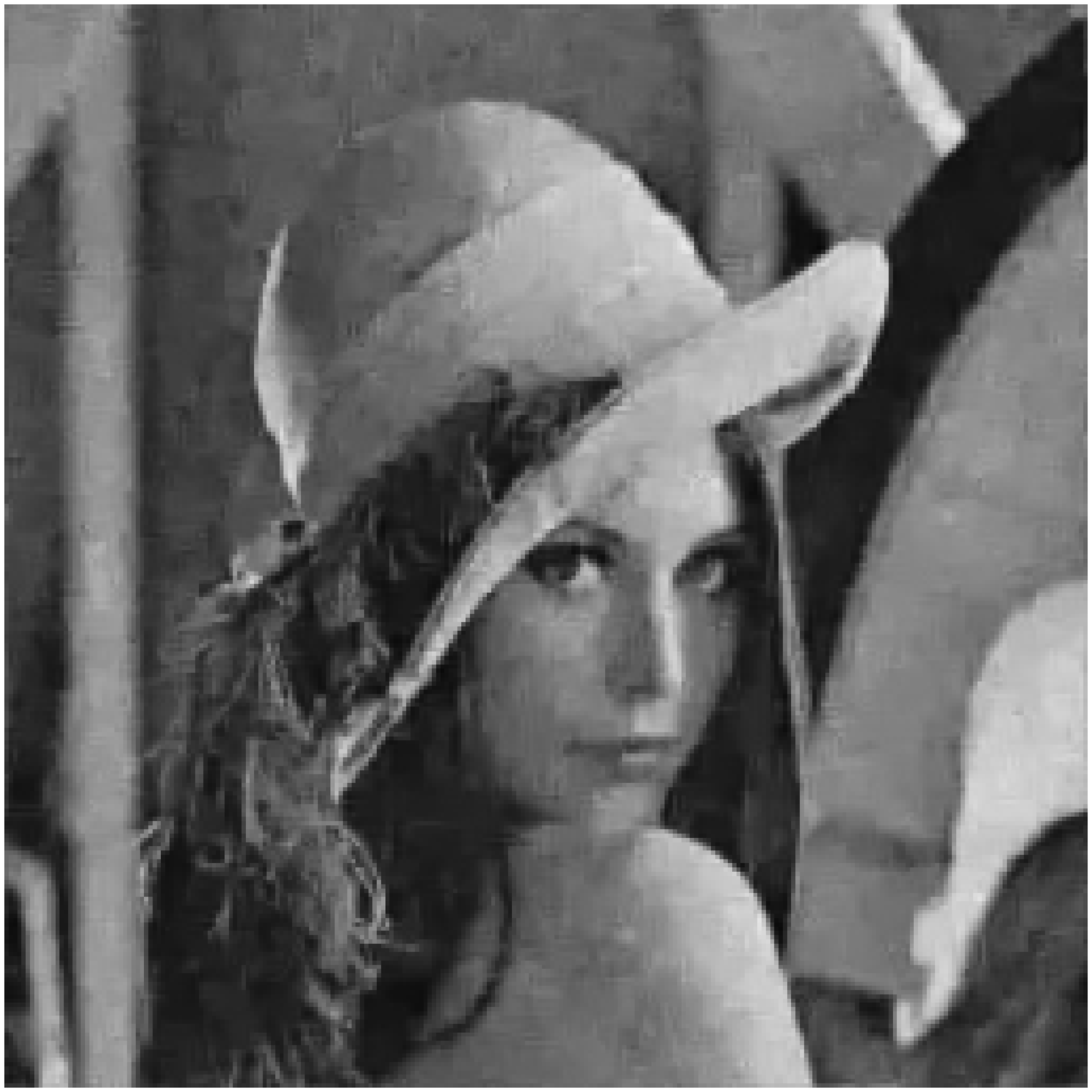}
	\includegraphics[width=0.32\textwidth]{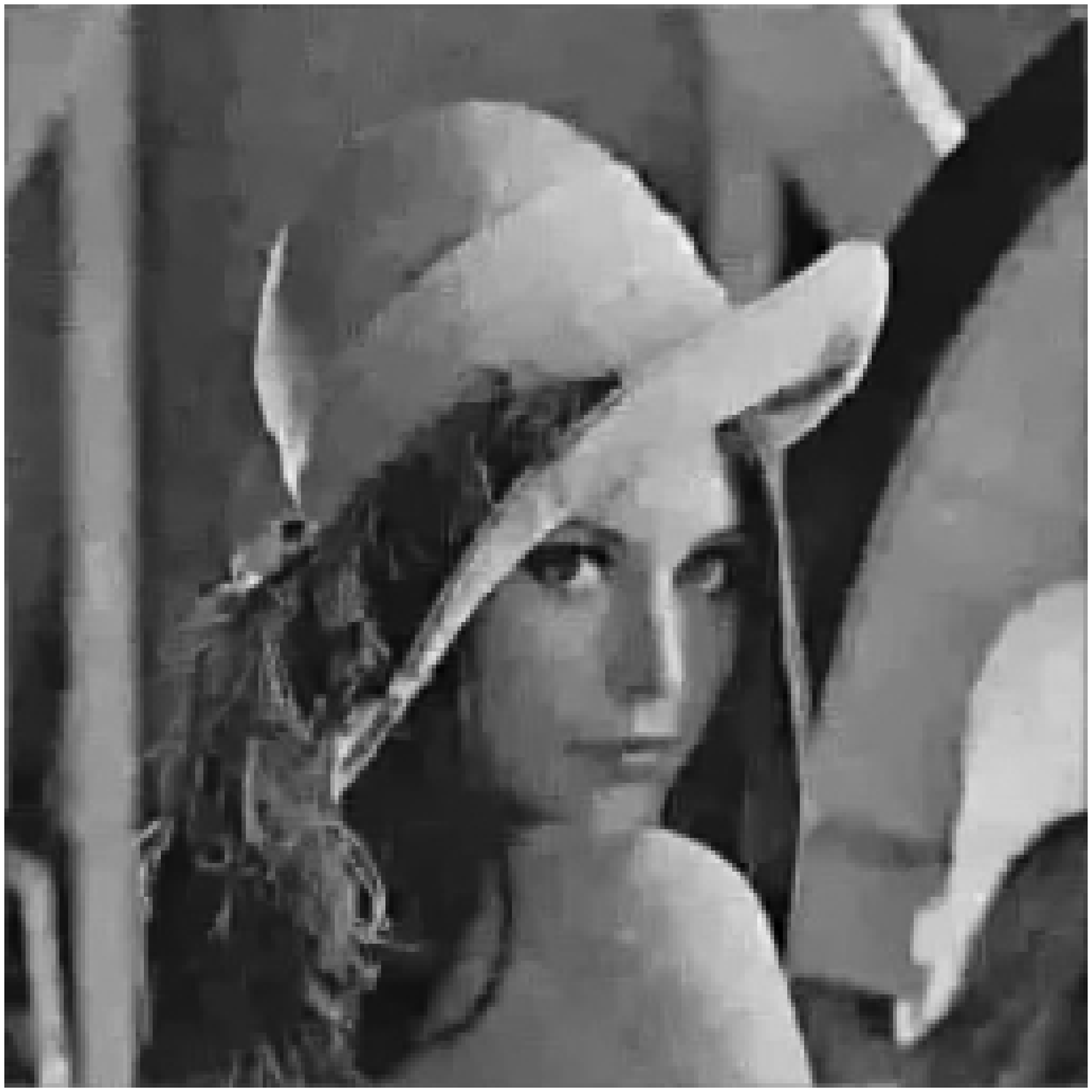}
	\caption{Comparison for image denoising. From left to right and up to down are: noisy images and  denoised by  SB, DTCWT, LCHMM, ${l_o}$-WF and ours using equations (\ref{13}) and (\ref{17}), respectively.}
	\label{F；D6}
\end{figure}

\begin{figure}[htbp]
	\centering
	\includegraphics[width=0.9\textwidth]{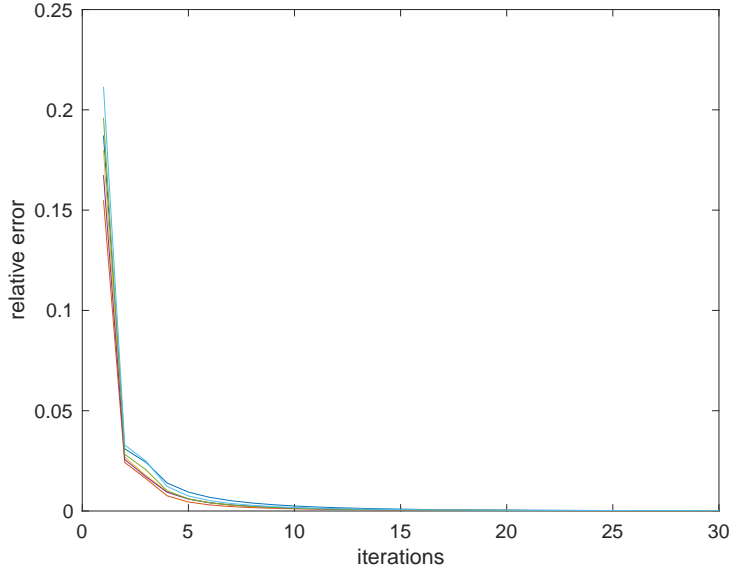}
	\caption{Relative error $\frac{||u^{n+1}-u^n||}{||u^n||}$ versus iterations of six test images.}
	\label{F；error}
\end{figure}

\begin{table*}[htbp]
	\centering
	\caption{PSNR  of the compared four methods for image denoising}
	\centering
	\begin{tabular}{c|c|c|c|c|c|c|c}
		\hline
		Image &	original&	SB&	DTCWT&LCHMM	& ${l_o}$-WF& \tabincell{c}{ours\\ using (\ref{13})}&  \tabincell{c}{ours\\ using (\ref{17})}\\
		\hline
		1&	20.45&	27.05&	26.21&27.14&	27.63&27.35&	\textbf{27.81}\\
		2&	20.04&	30.20&	28.45&29.78&	\textbf{31.17}&30.63&	30.99\\
		3&	20.07&	25.64&	25.28&26.00&	25.88&26.38&	\textbf{26.36}\\
		4&	20.15&	26.16&	25.78&26.47&	26.40&26.66&	\textbf{26.69}\\
		5&	20.06&	27.64&	27.09&27.98&	27.98&28.12& \textbf{28.18}\\
		6&	20.17&	27.89&	27.17&28.08&	28.20&28.23&	\textbf{28.33}\\
		average&	20.17&	27.43&26.66&	27.57&27.88&27.89 &	\textbf{28.06}\\
		Time(s)&	-&	3.09&	\textbf{0.49}&-&	12.40&10.00&	4.02\\
		\hline
	\end{tabular}%
\label{T1}
\end{table*}%

\subsection{Image debluring}
\label{se3.2}
In this subsection, we apply the proposed model to image deblurring and compare the results with those of  ${l_o}$-WF and EDWF. We also test the same six images for image deblurring. The blurred images are generated by convolution with the blur kernel $A$  and added by a Gaussian noise  $\eta $, where $A$  and $\eta $  are generated by the MATLAB command `$A = fspecial({\rm{'motion', 9, 0}})$ ' and `$\eta  = 5*randn(size(u))$ ', respectively. 

In the implementation of image deblurring, the parameters of our model using equation (\ref{13}) are set as:  ${\lambda _1} =0.006$, $\lambda _2==...=\lambda _9 = 0.004$,  ${\gamma _1} = 0.4$,  ${\gamma _2}=\gamma _3=...=\gamma _9 = 0.1$,  $tol = 5e - 4$; the parameters of our algorithm using equation (\ref{17}) are set as:  ${\lambda _1} = 0.004$,  ${\lambda _2}=\lambda _3=...=\lambda _9 = 0.002$,  ${\gamma _1} = 0.1$,  ${\gamma _2}=\gamma _3=...=\gamma _9 = 0.4$,  $tol = 5e - 4$. For  ${l_o}$-WF, the piecewise linear B-spline wavelet frame is also used and the decomposition level   $L = 1$,  $\lambda  = 30$,  $\mu  = 0.01$,  $\gamma  = 0.003$,  $tol = 5e - 4$. For  EDWF, $\lambda=1, \gamma=1.5, \rho=0.2, v=0.15, \mu_1=1, \mu_2=1$.

In Figure \ref{F；DCONv1}, the visual results of the three models are presented and Table \ref{T2}, summarizes the PSNR and  total run-time of the three methods. From Table \ref{T2}, we can see that the PSNR of our model using equation (\ref{17}) is averagely the largest. 

\begin{figure}[htbp]
	\centering
	\includegraphics[width=0.19\textwidth]{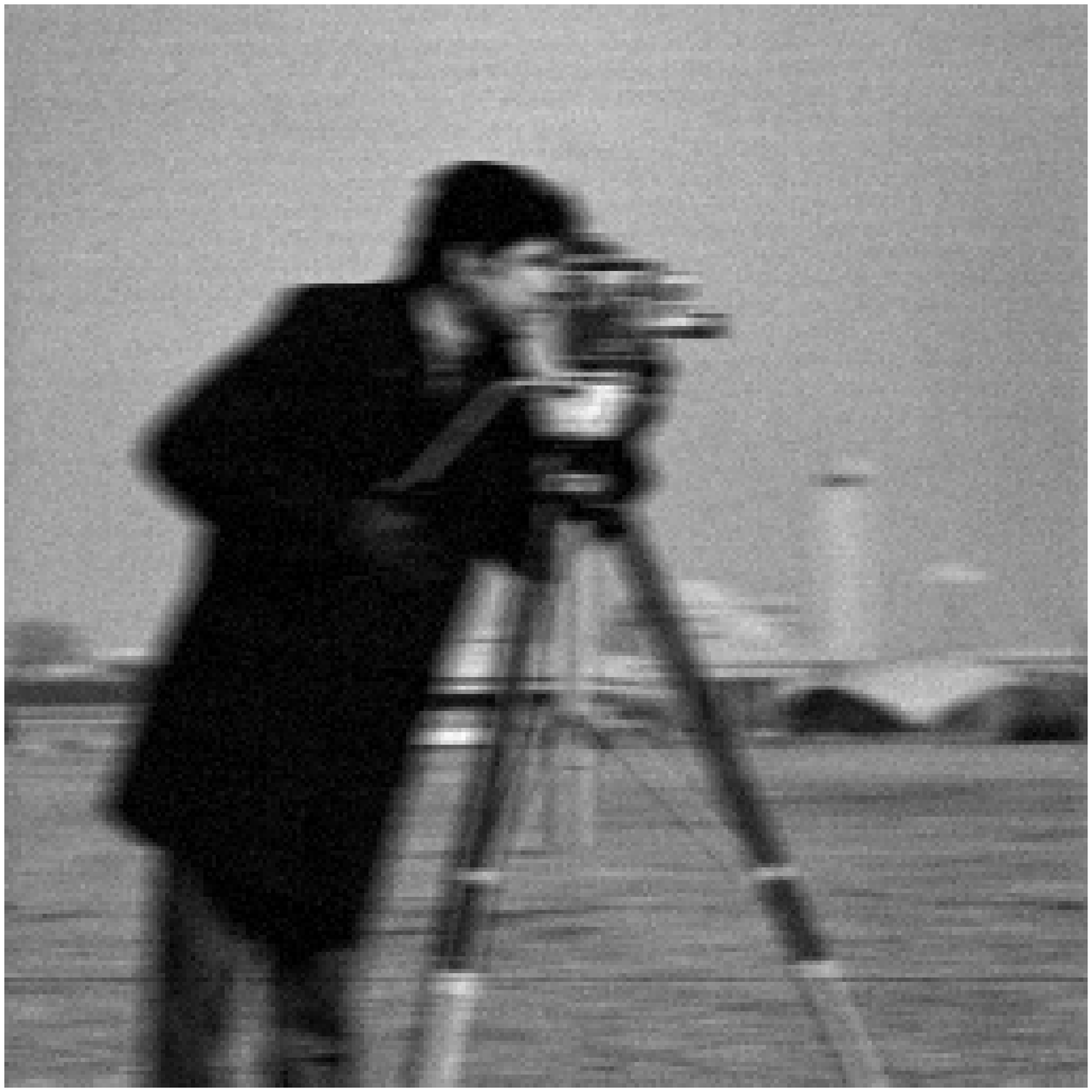}
	\includegraphics[width=0.19\textwidth]{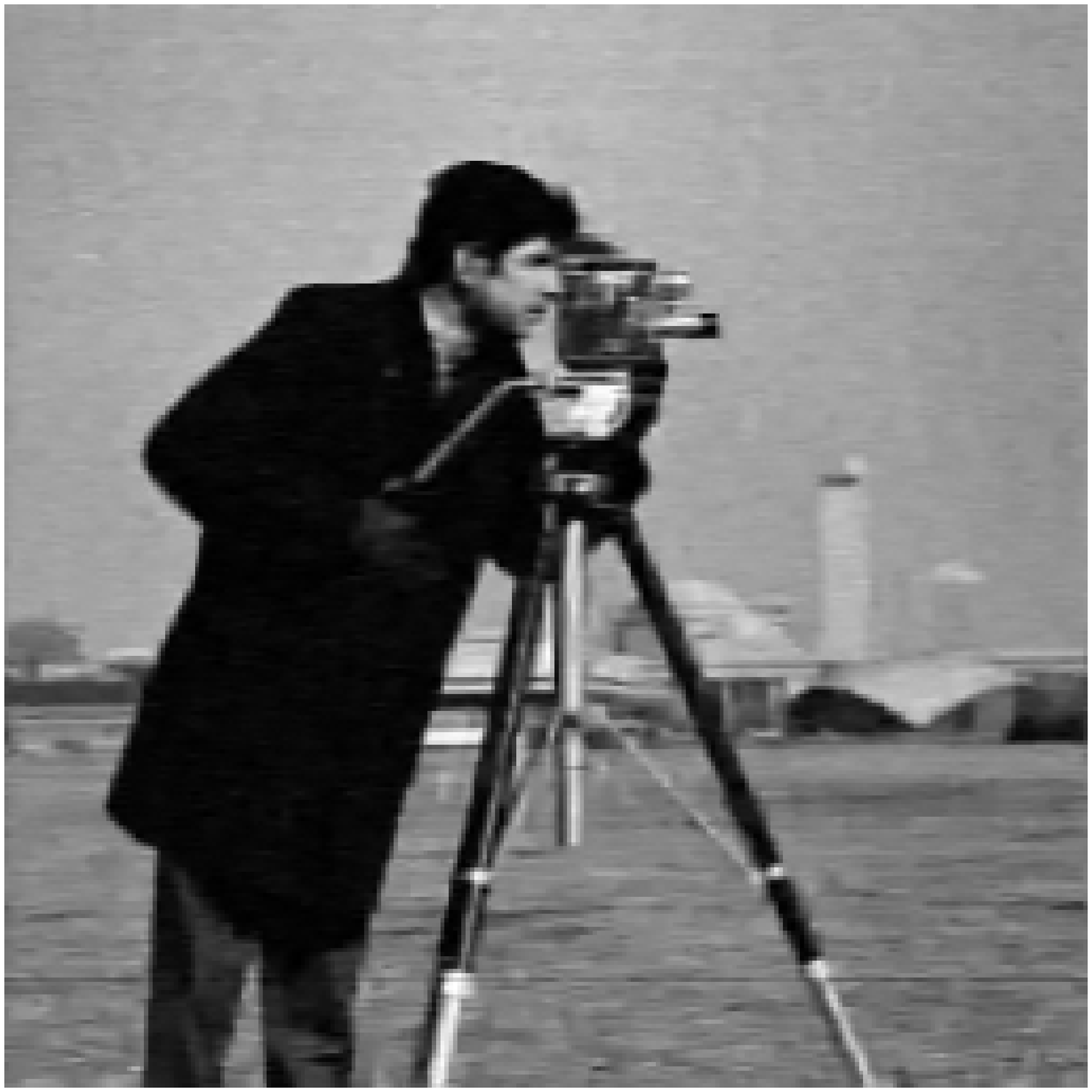}
	\includegraphics[width=0.19\textwidth]{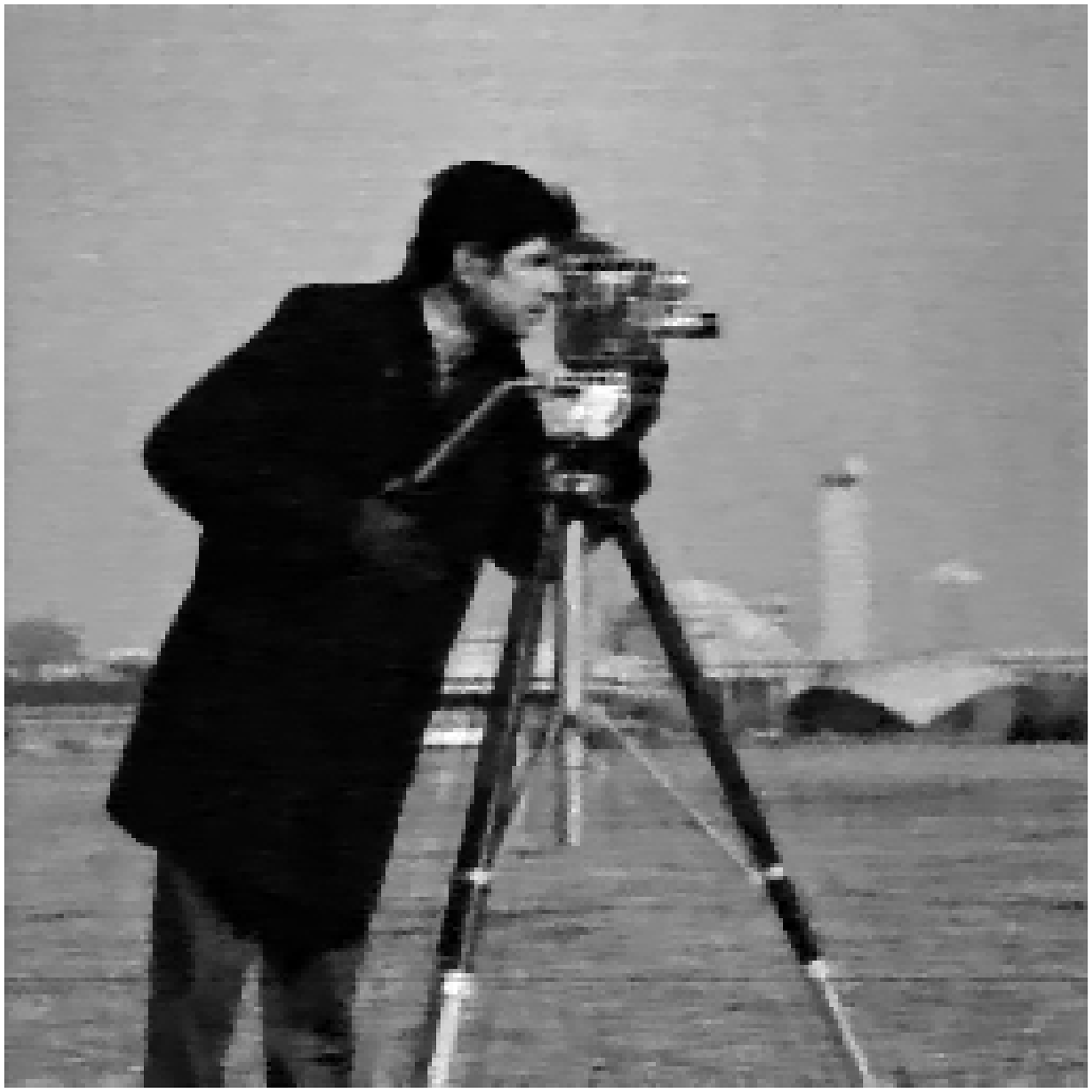}
	\includegraphics[width=0.19\textwidth]{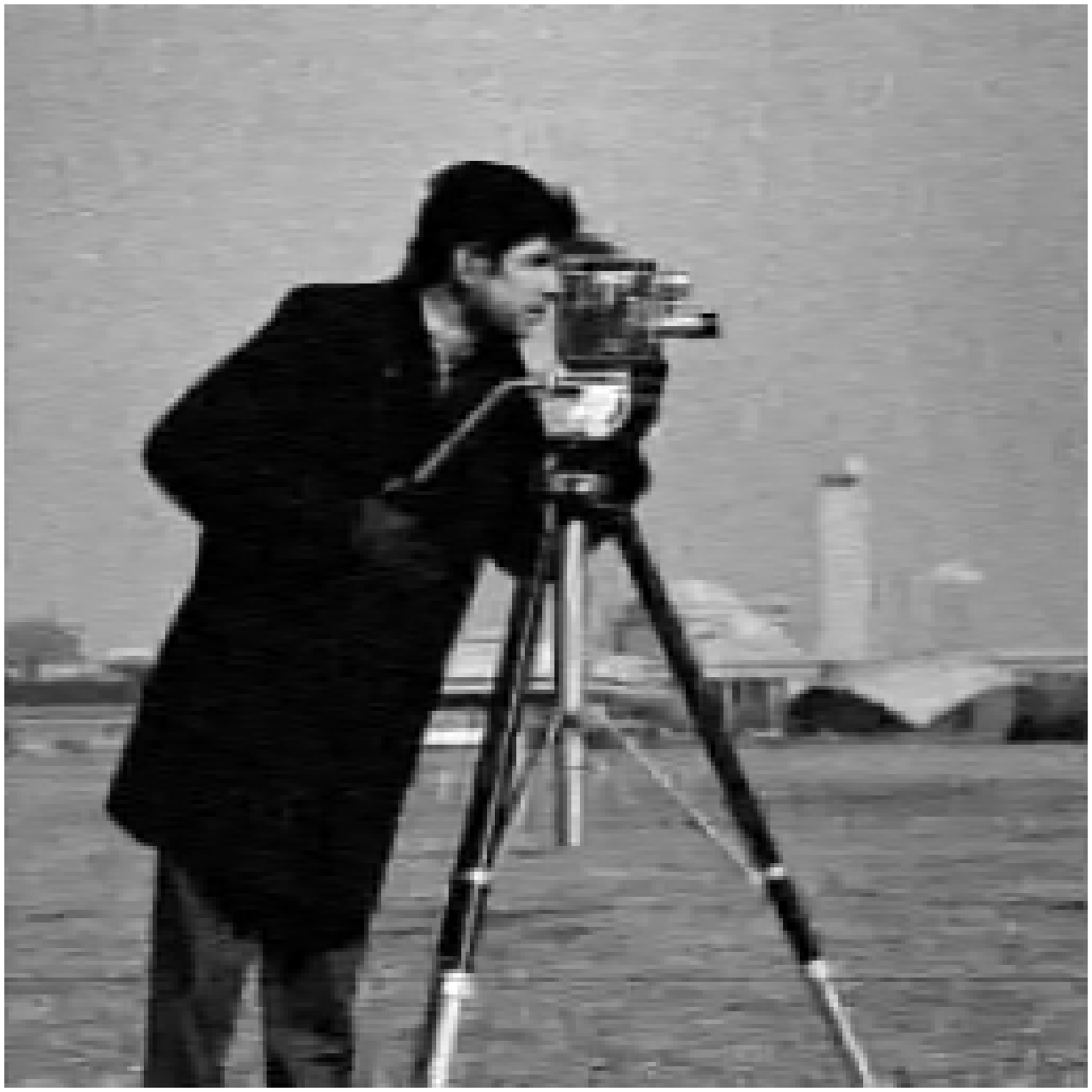}
	\includegraphics[width=0.19\textwidth]{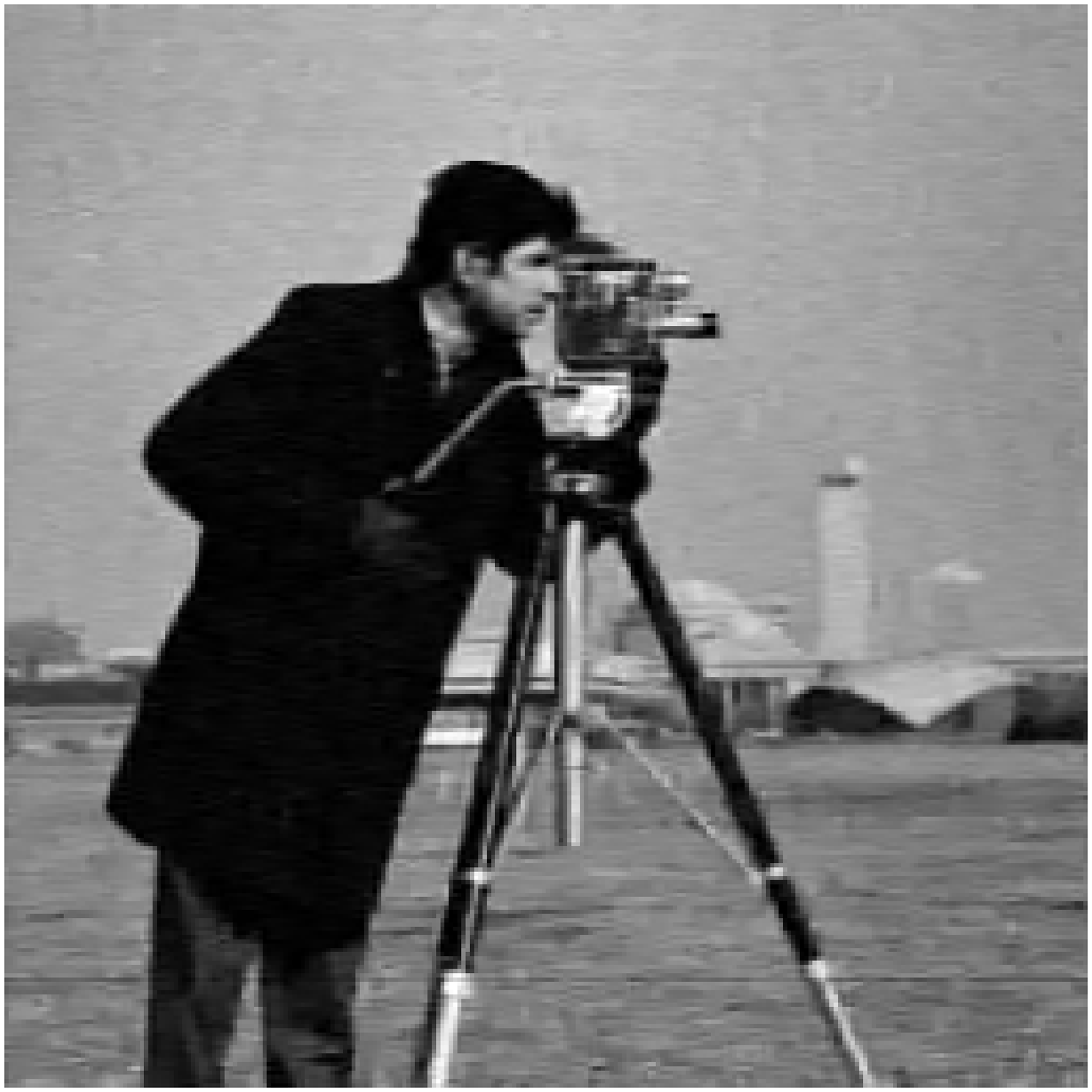}\\
    \includegraphics[width=0.19\textwidth]{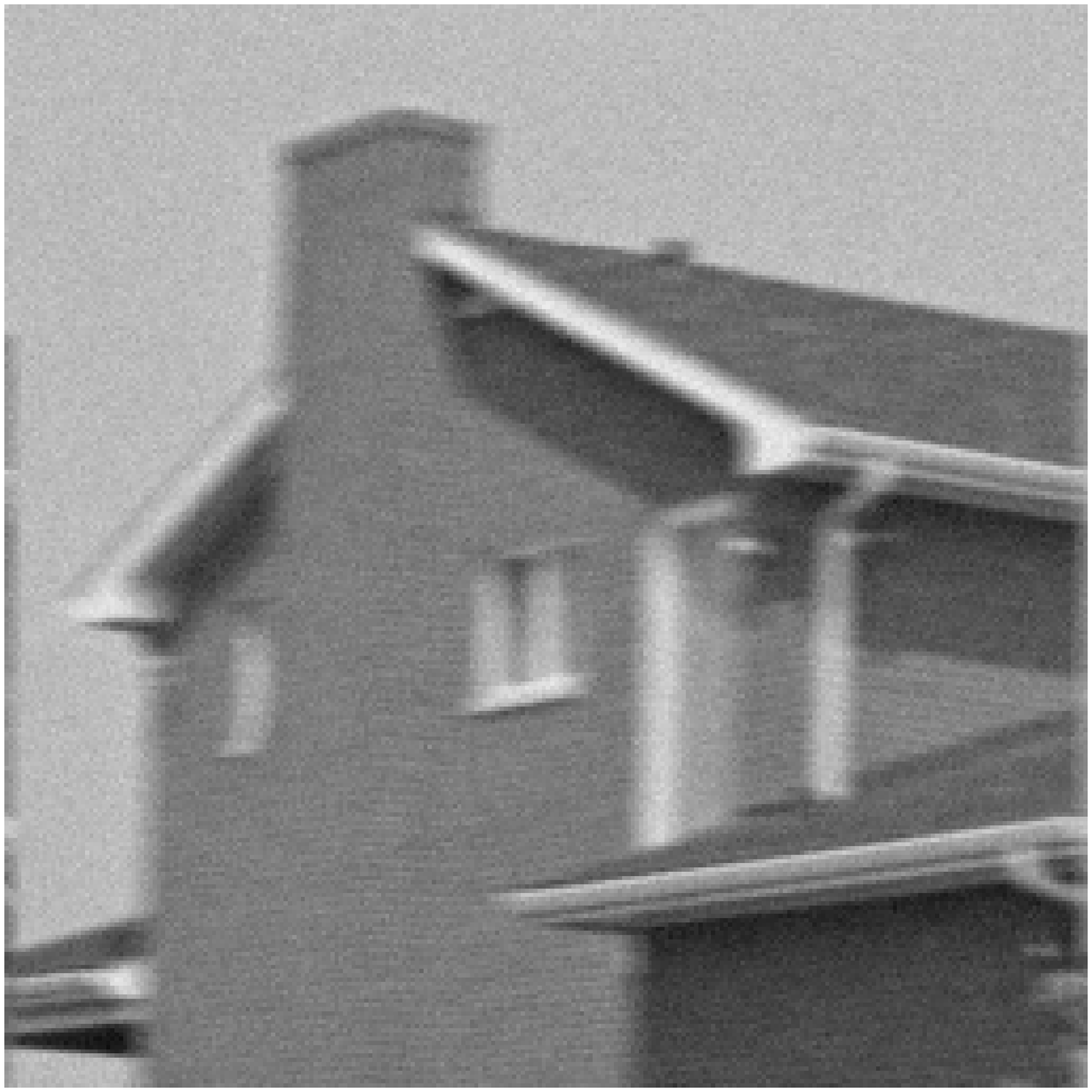}
    \includegraphics[width=0.19\textwidth]{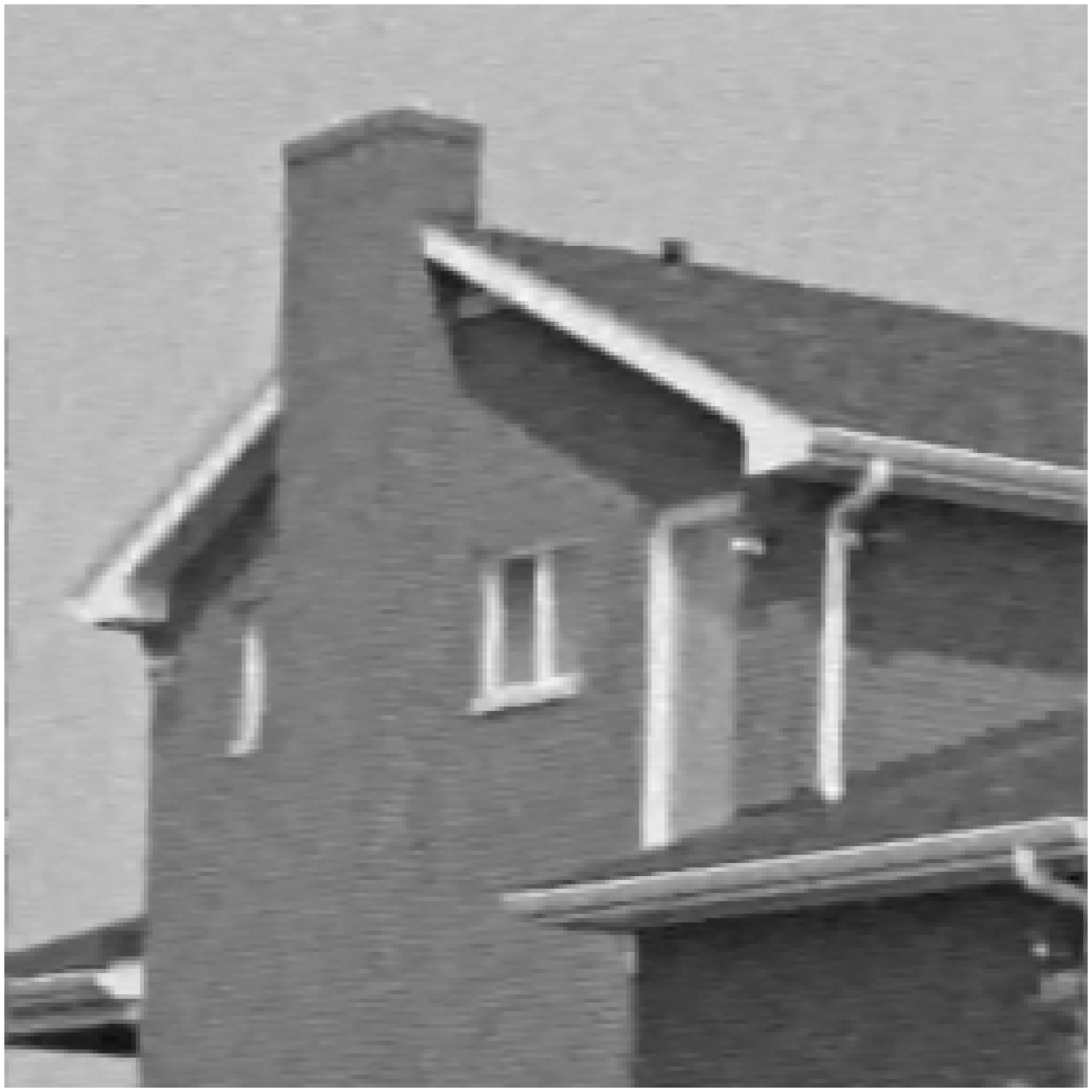}
    \includegraphics[width=0.19\textwidth]{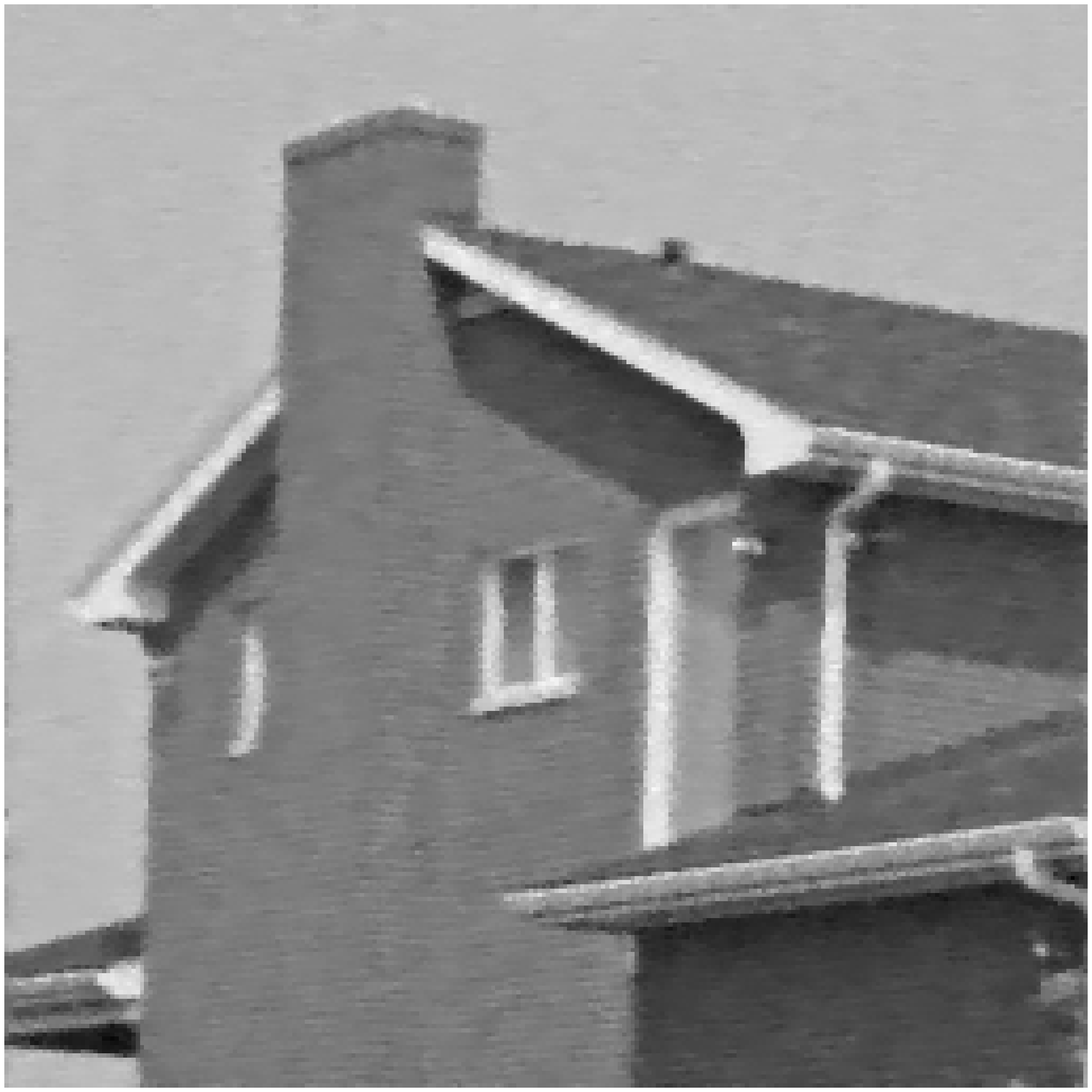}
    \includegraphics[width=0.19\textwidth]{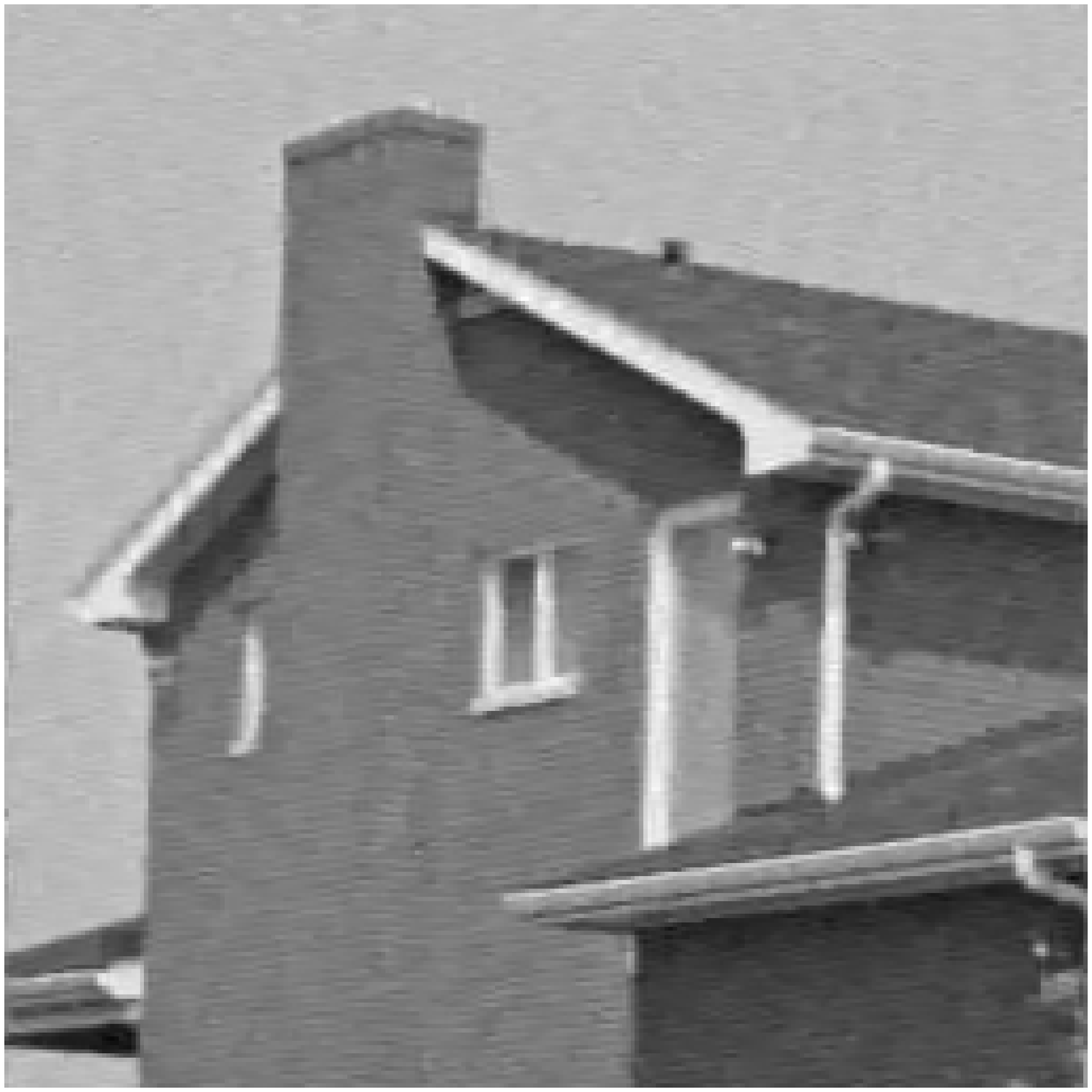}
    \includegraphics[width=0.19\textwidth]{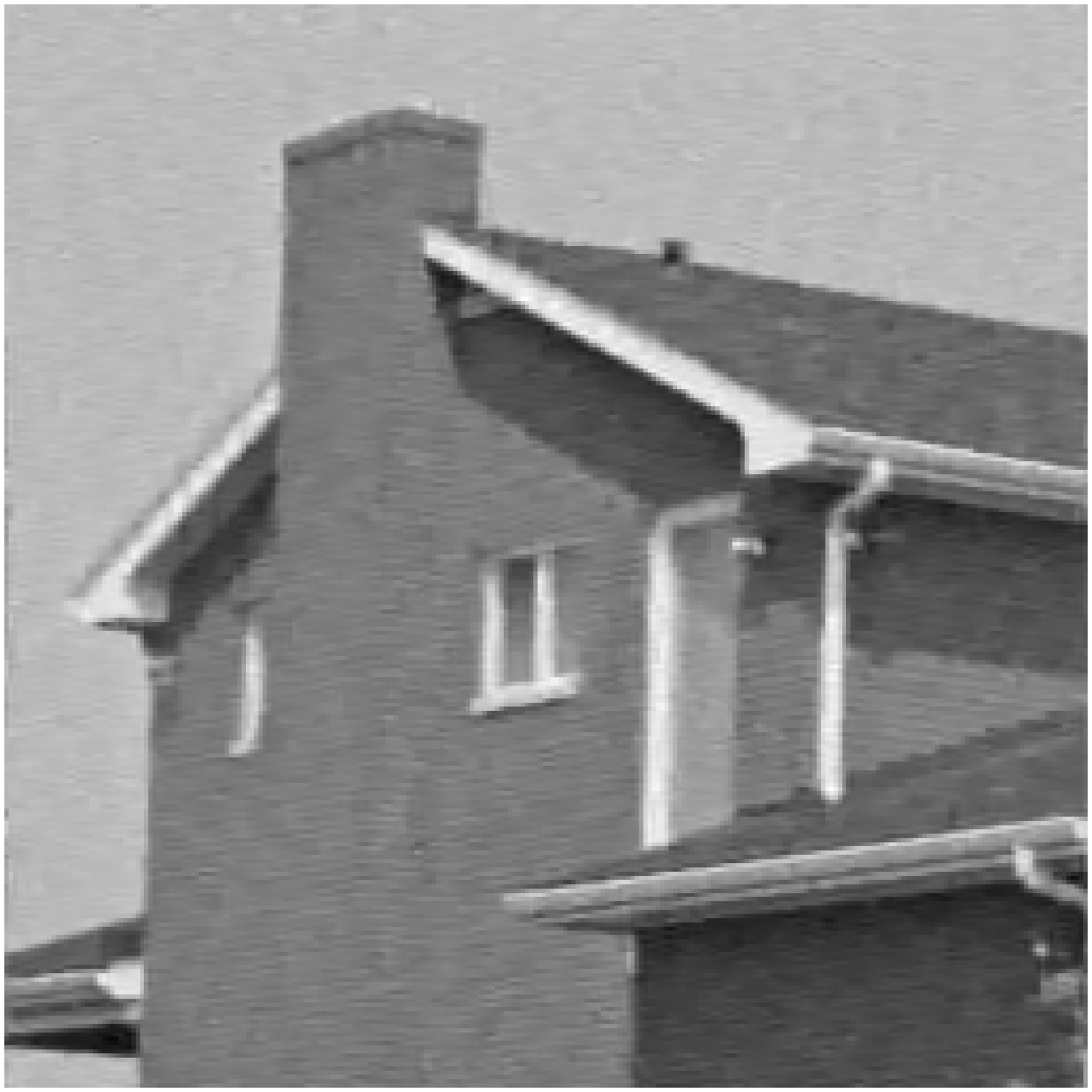}\\
    \includegraphics[width=0.19\textwidth]{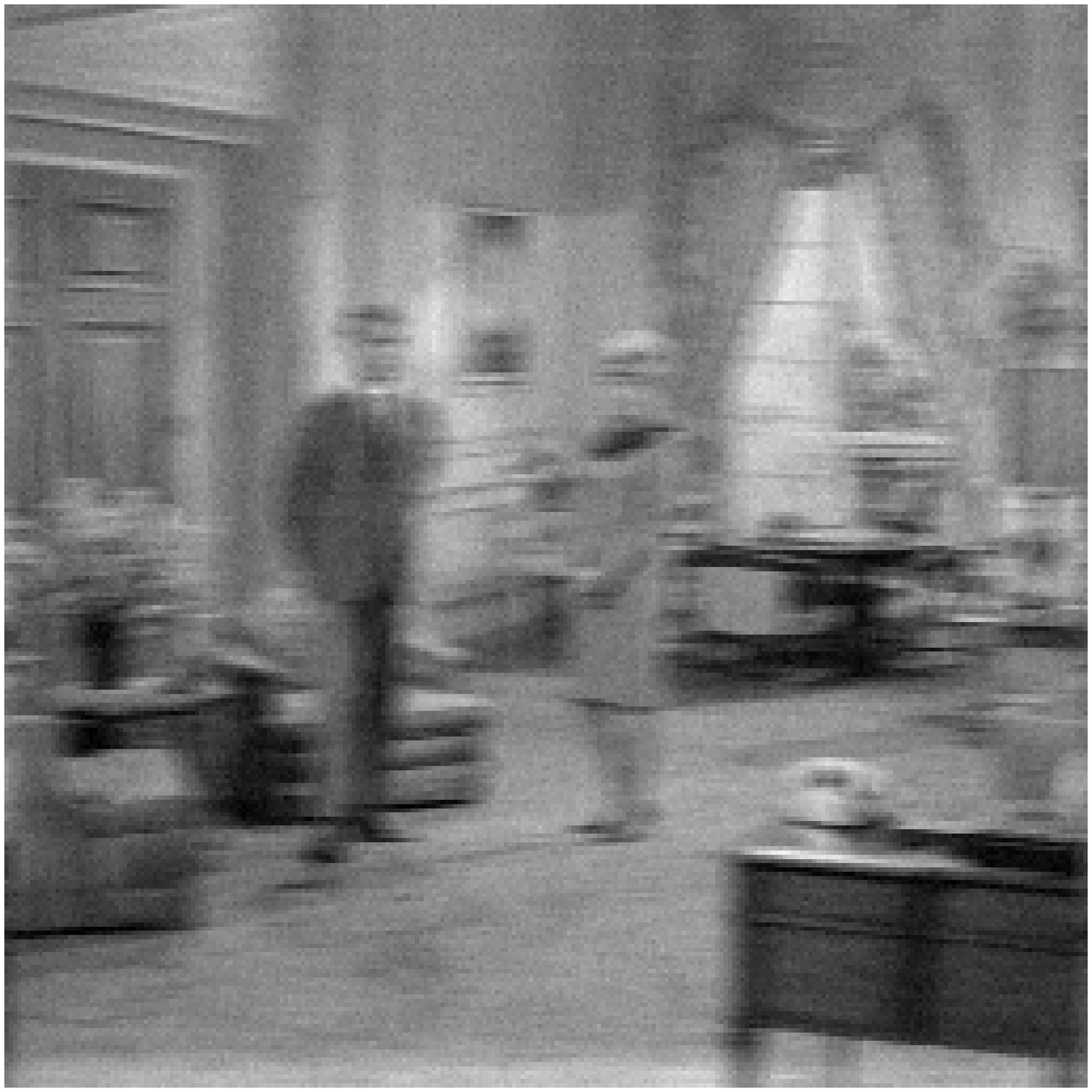}
    \includegraphics[width=0.19\textwidth]{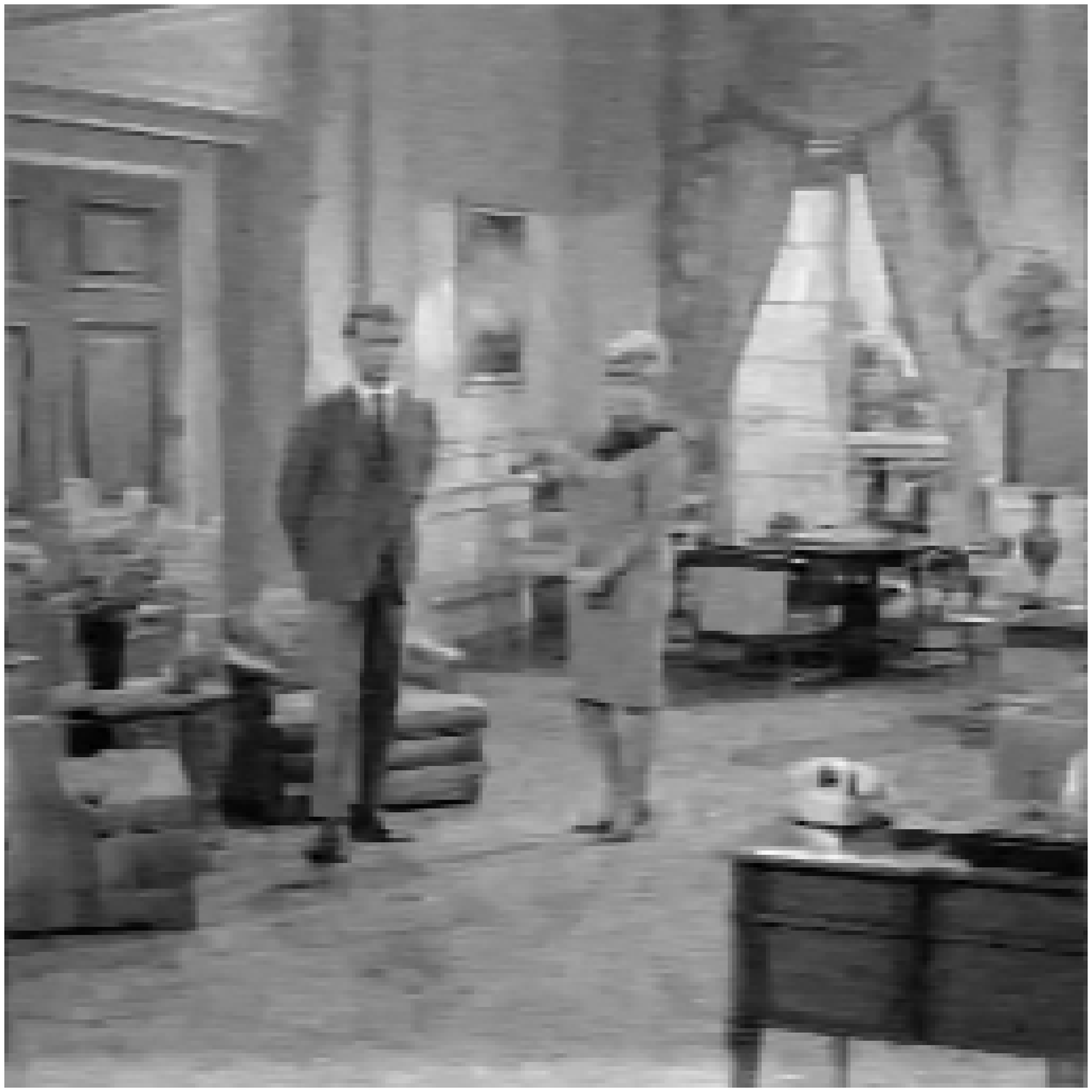}
    \includegraphics[width=0.19\textwidth]{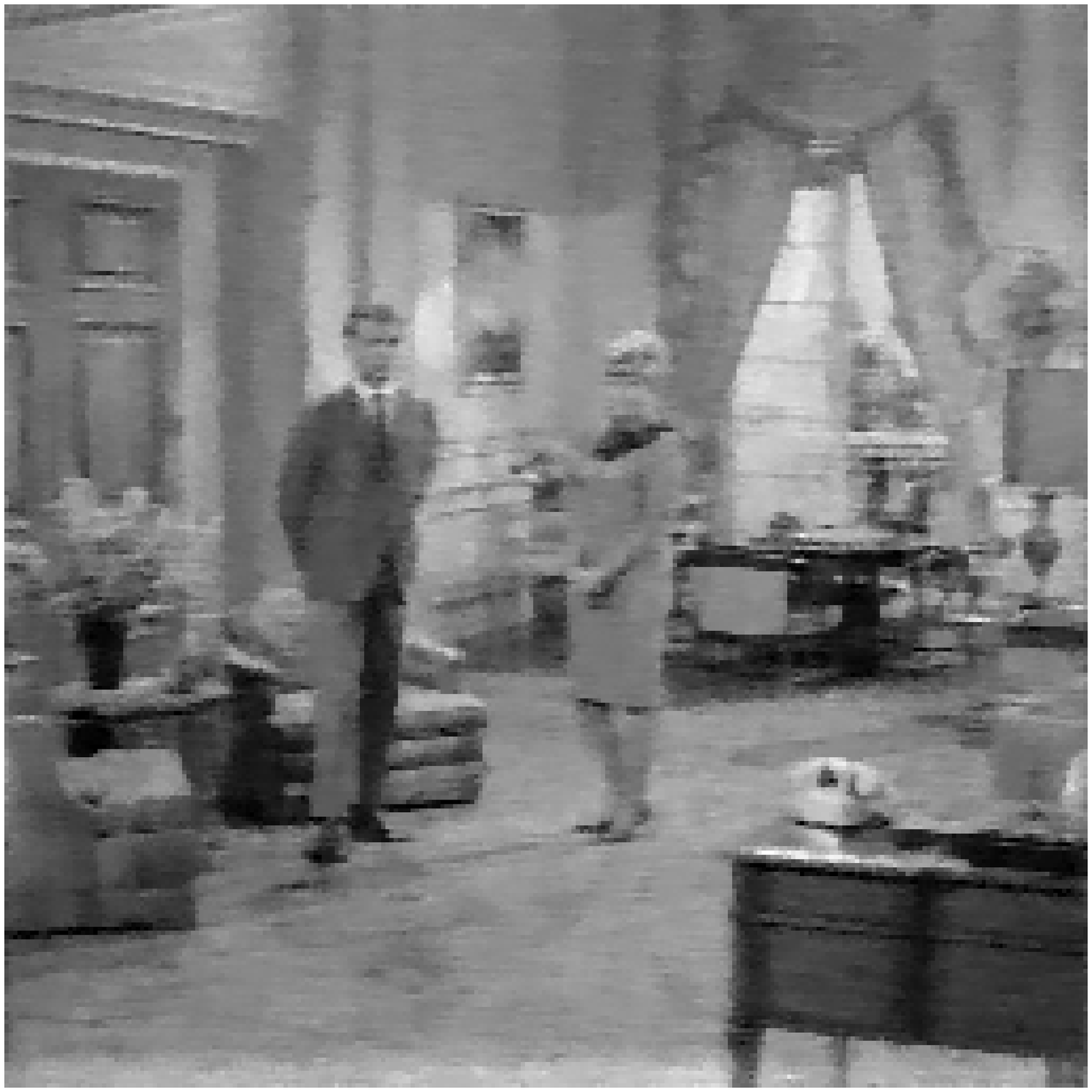}
    \includegraphics[width=0.19\textwidth]{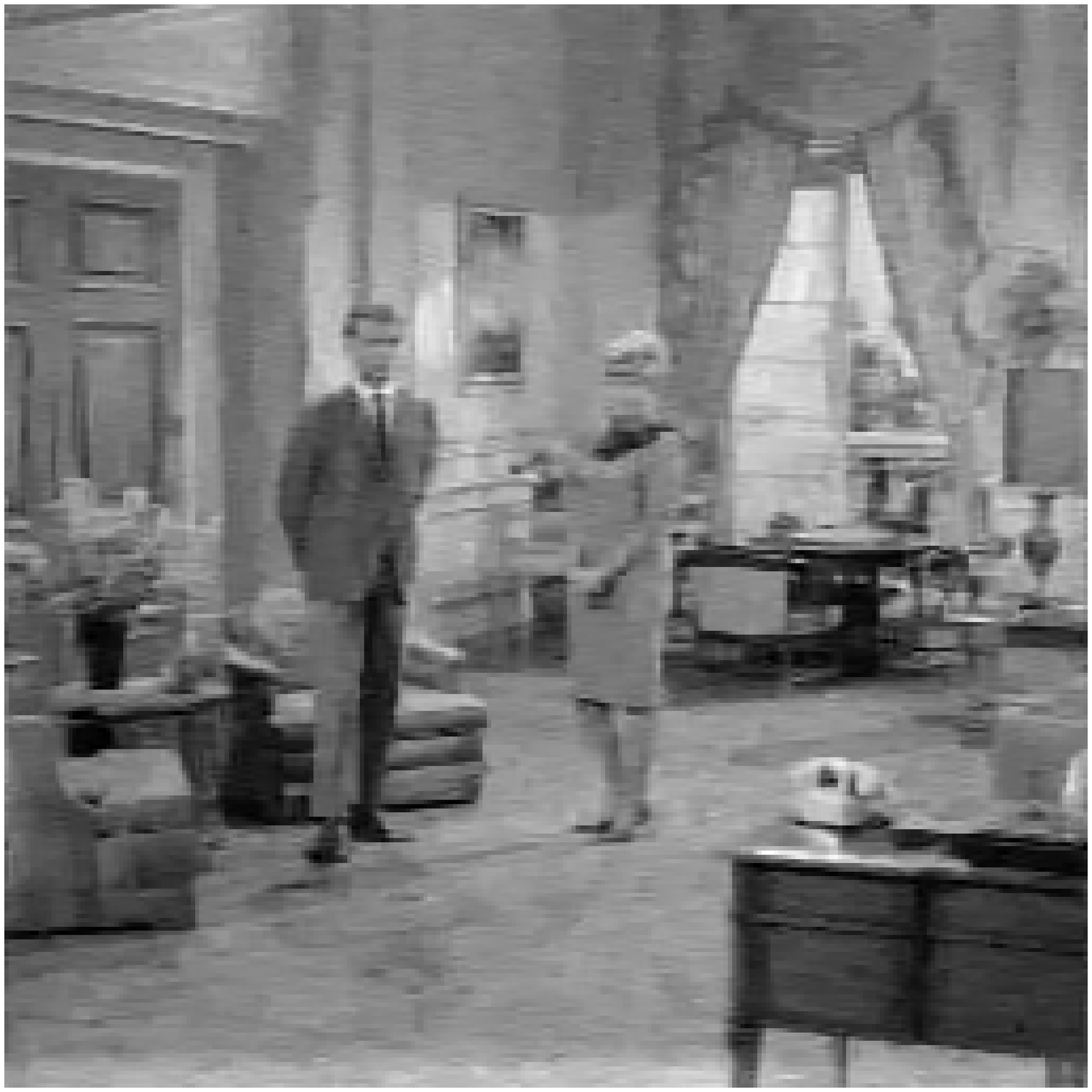}
    \includegraphics[width=0.19\textwidth]{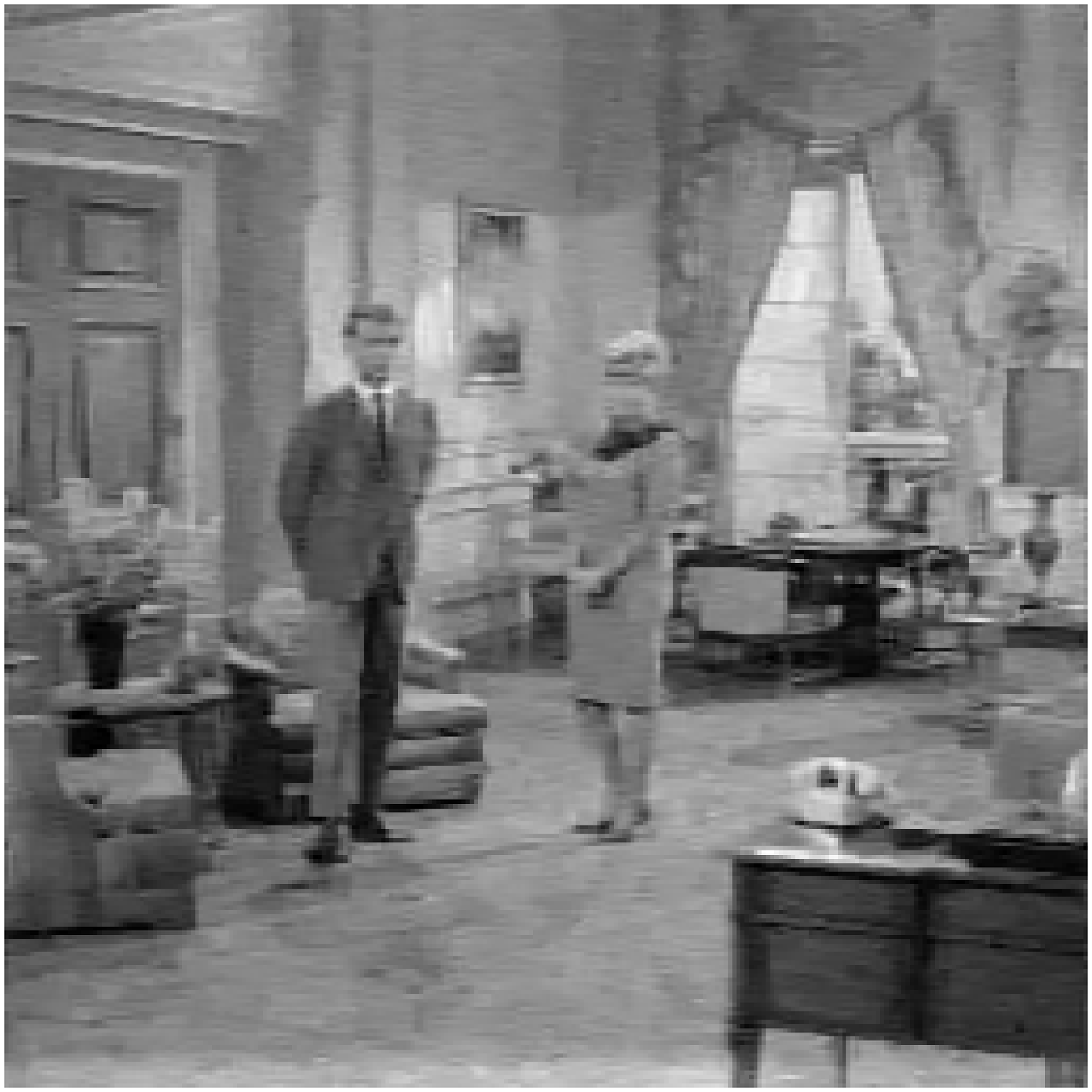}\\
    \includegraphics[width=0.19\textwidth]{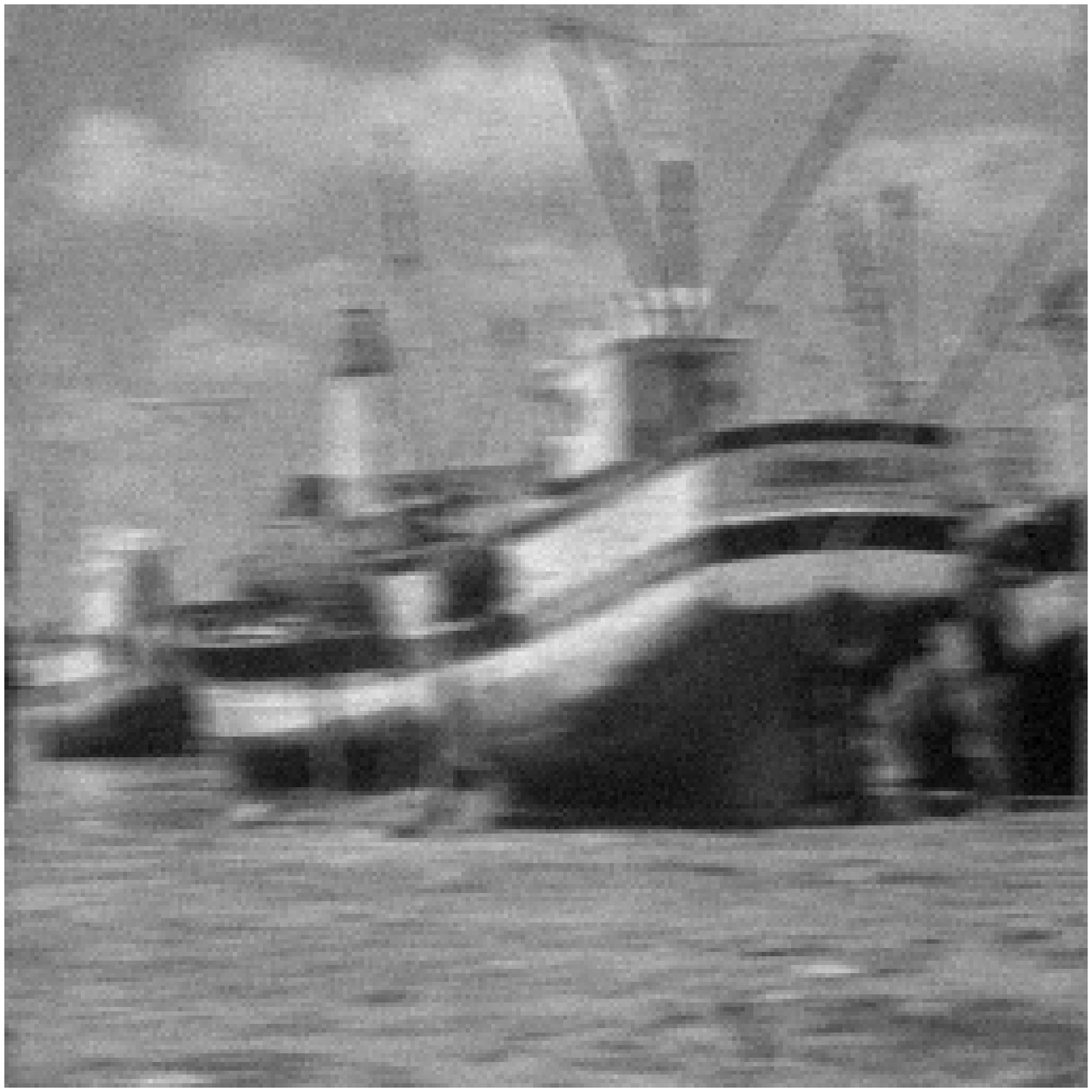}
    \includegraphics[width=0.19\textwidth]{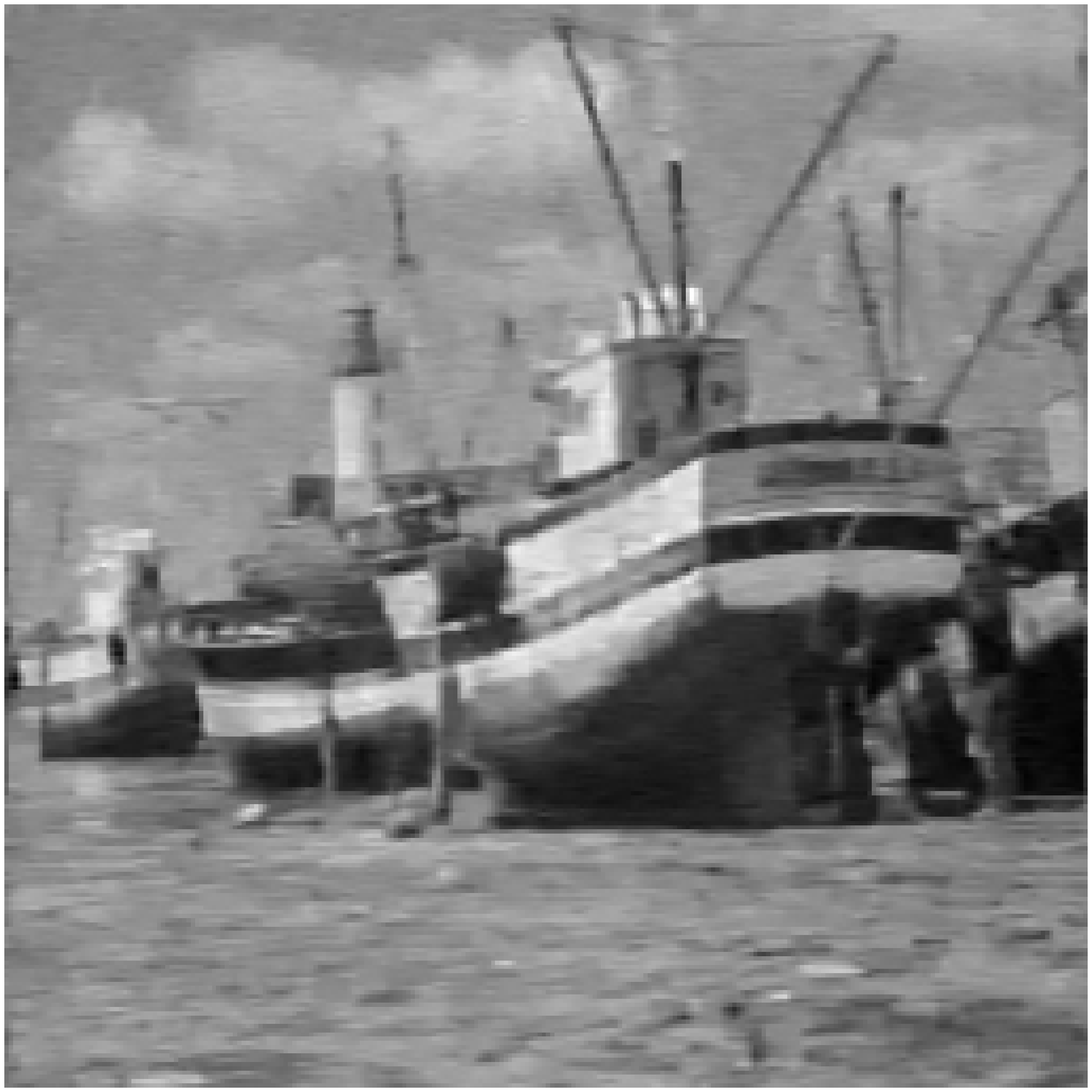}
    \includegraphics[width=0.19\textwidth]{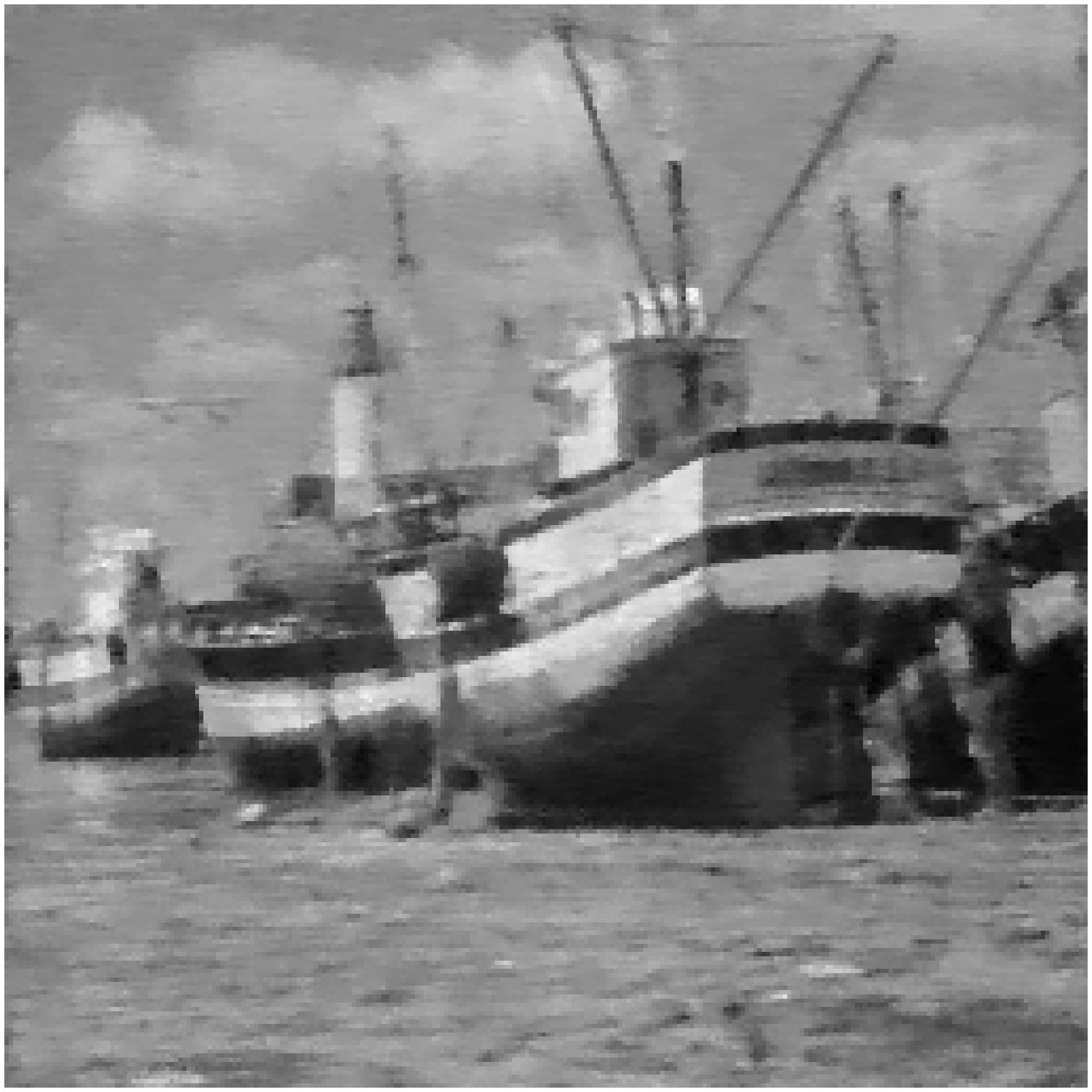}
    \includegraphics[width=0.19\textwidth]{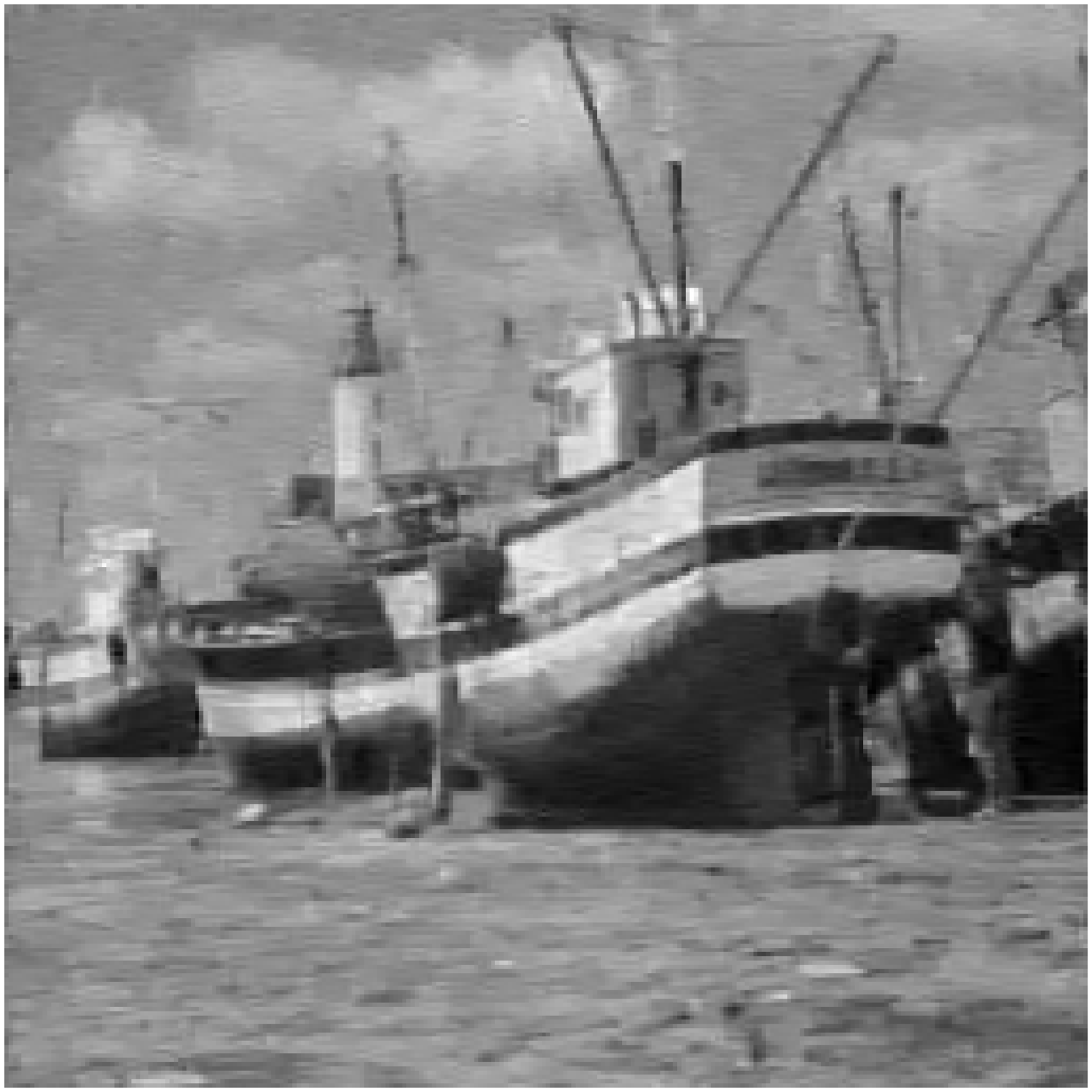}
    \includegraphics[width=0.19\textwidth]{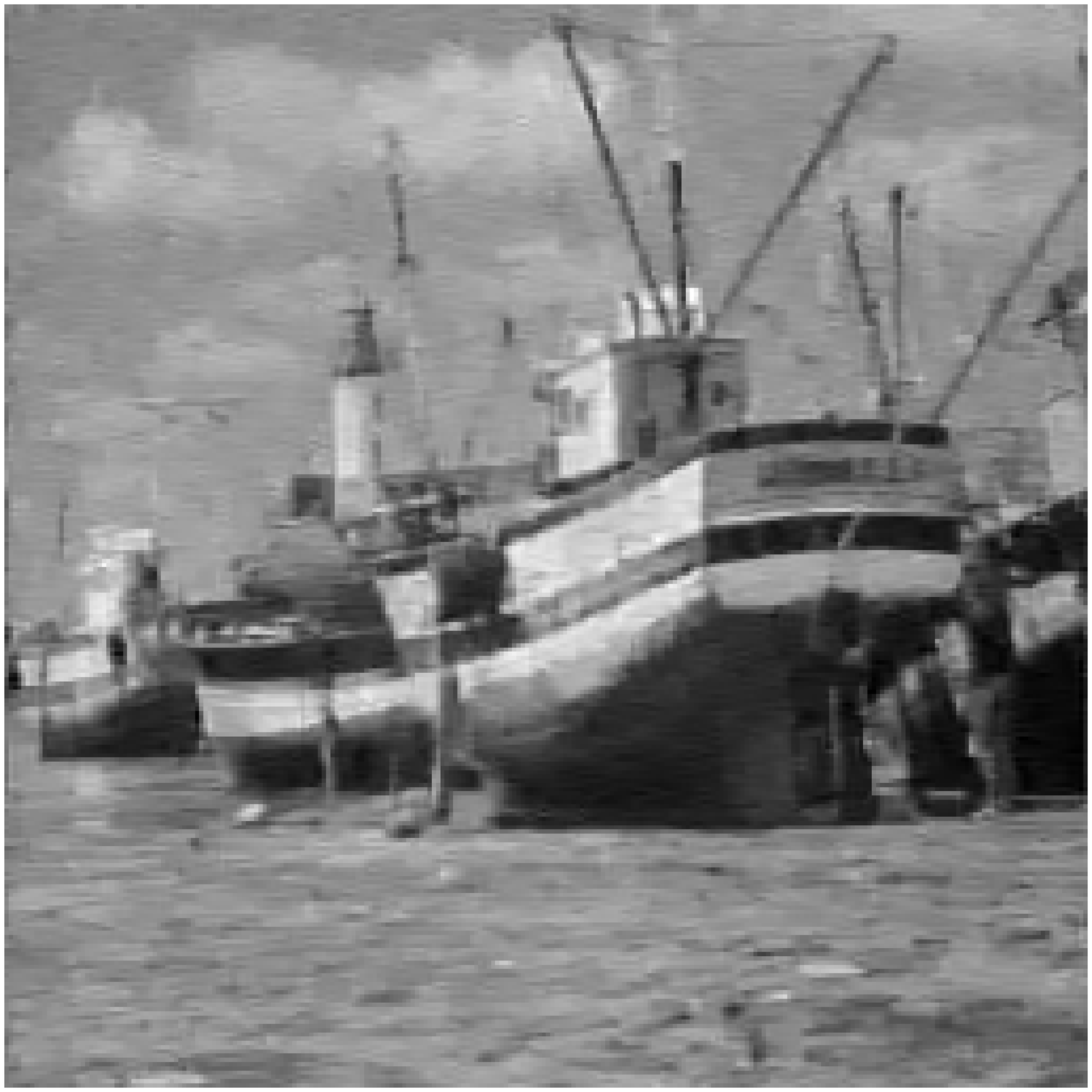}\\
    \includegraphics[width=0.19\textwidth]{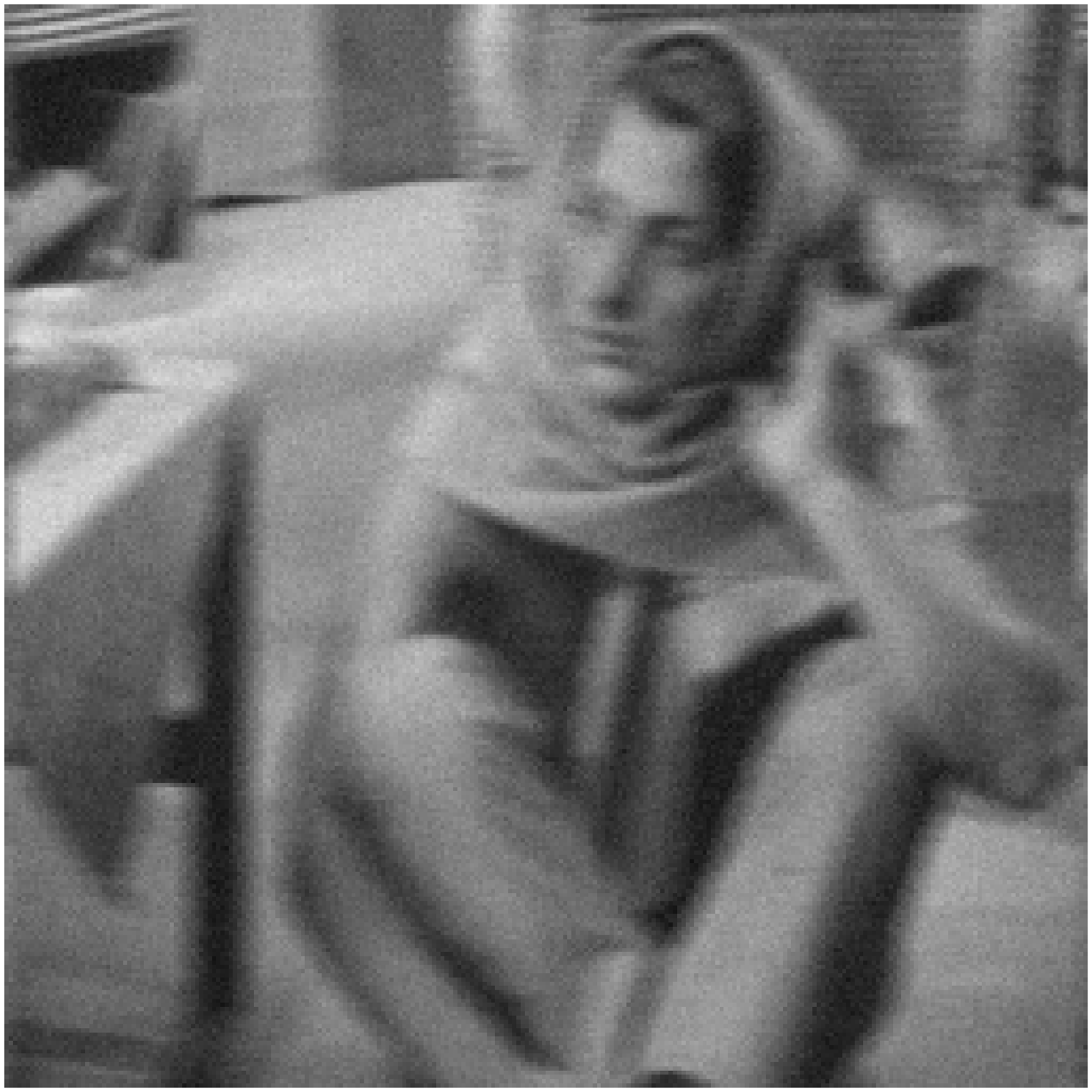}
    \includegraphics[width=0.19\textwidth]{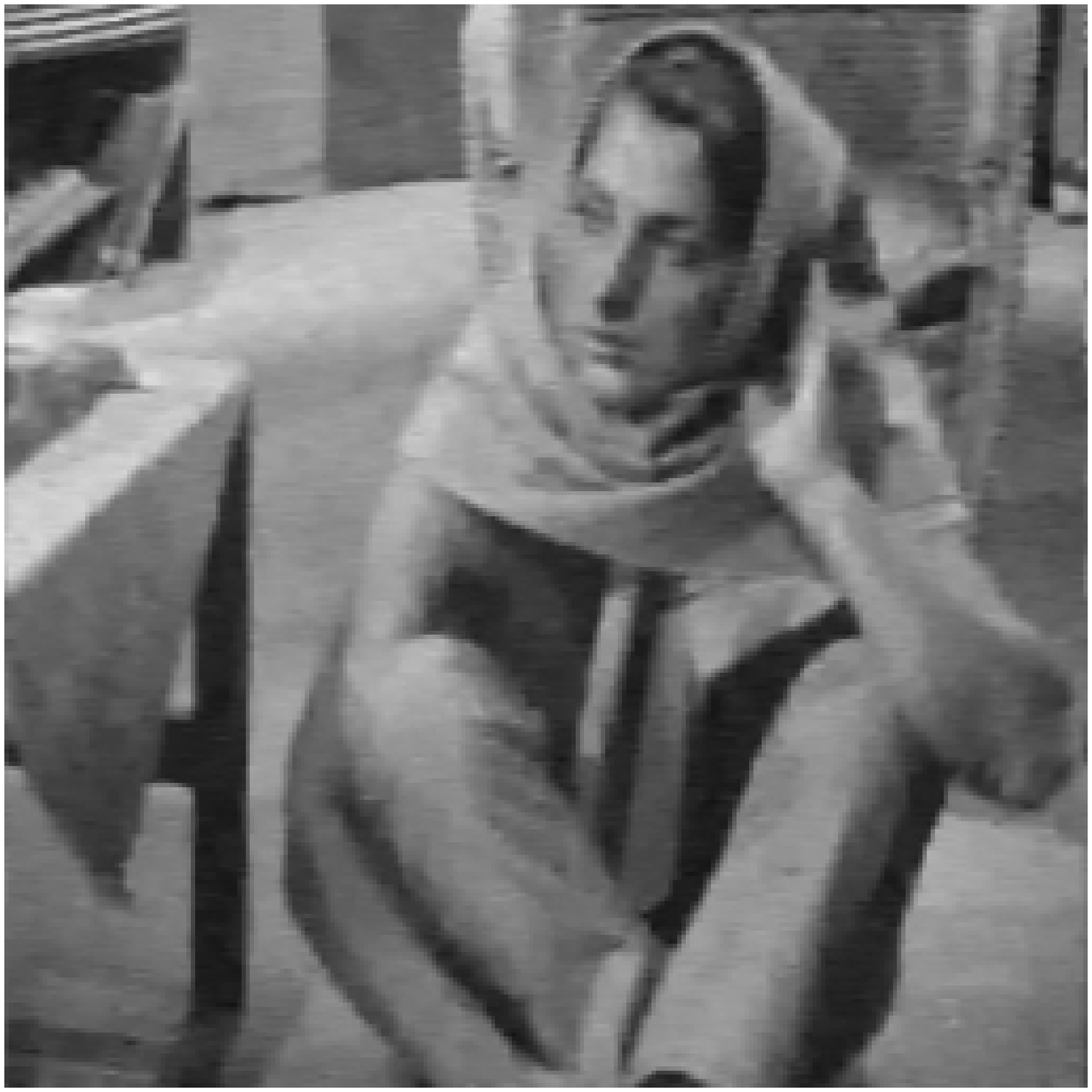}
    \includegraphics[width=0.19\textwidth]{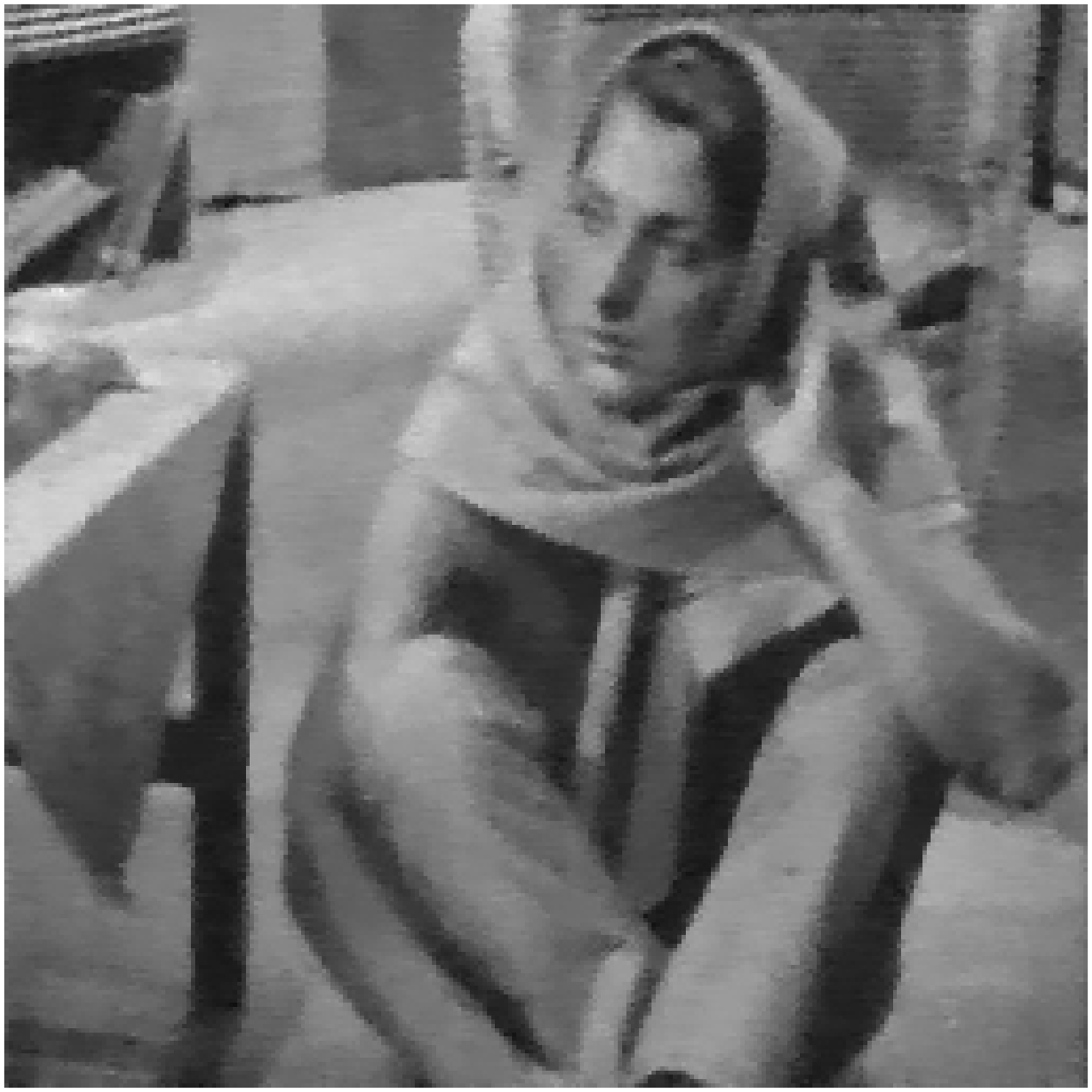}
    \includegraphics[width=0.19\textwidth]{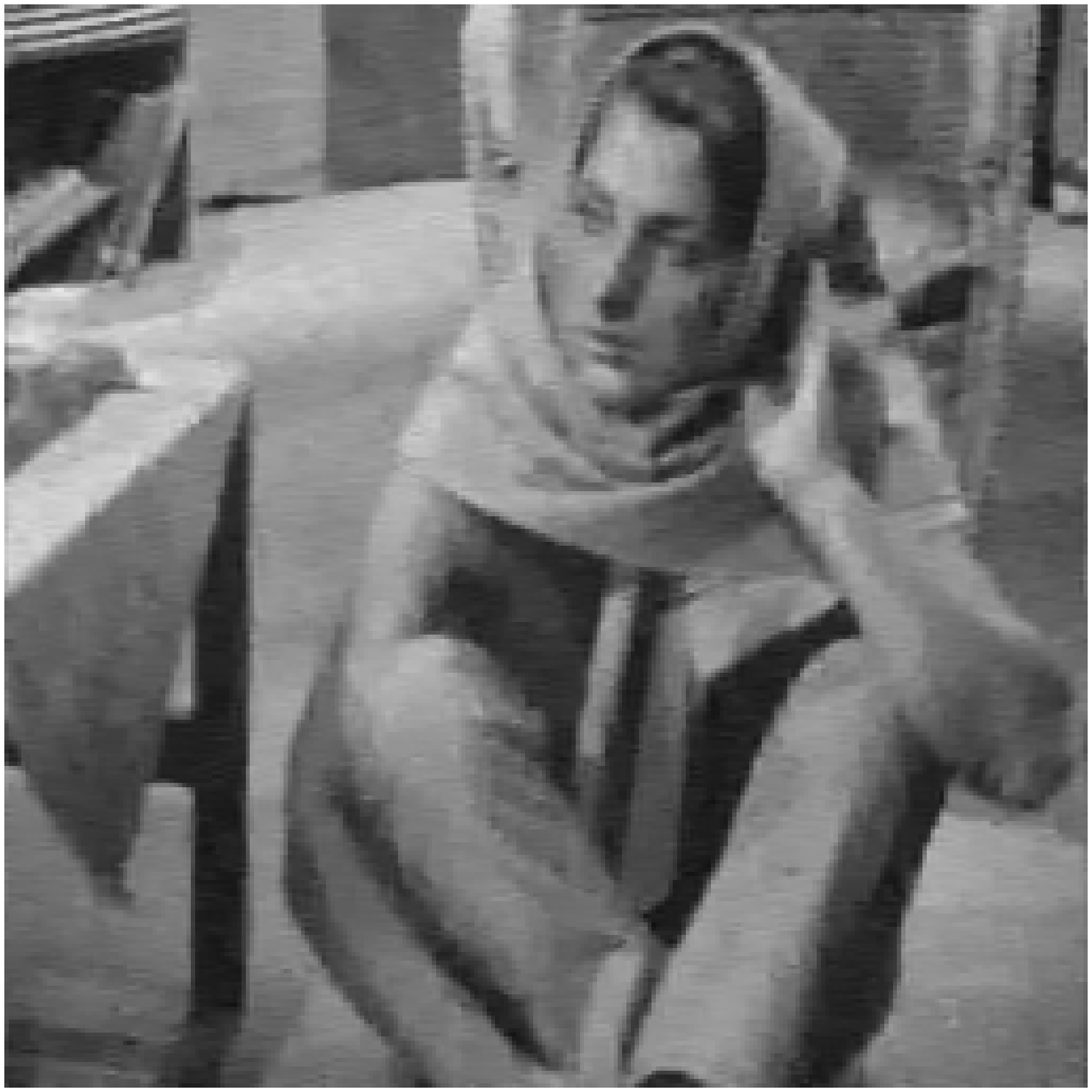}
    \includegraphics[width=0.19\textwidth]{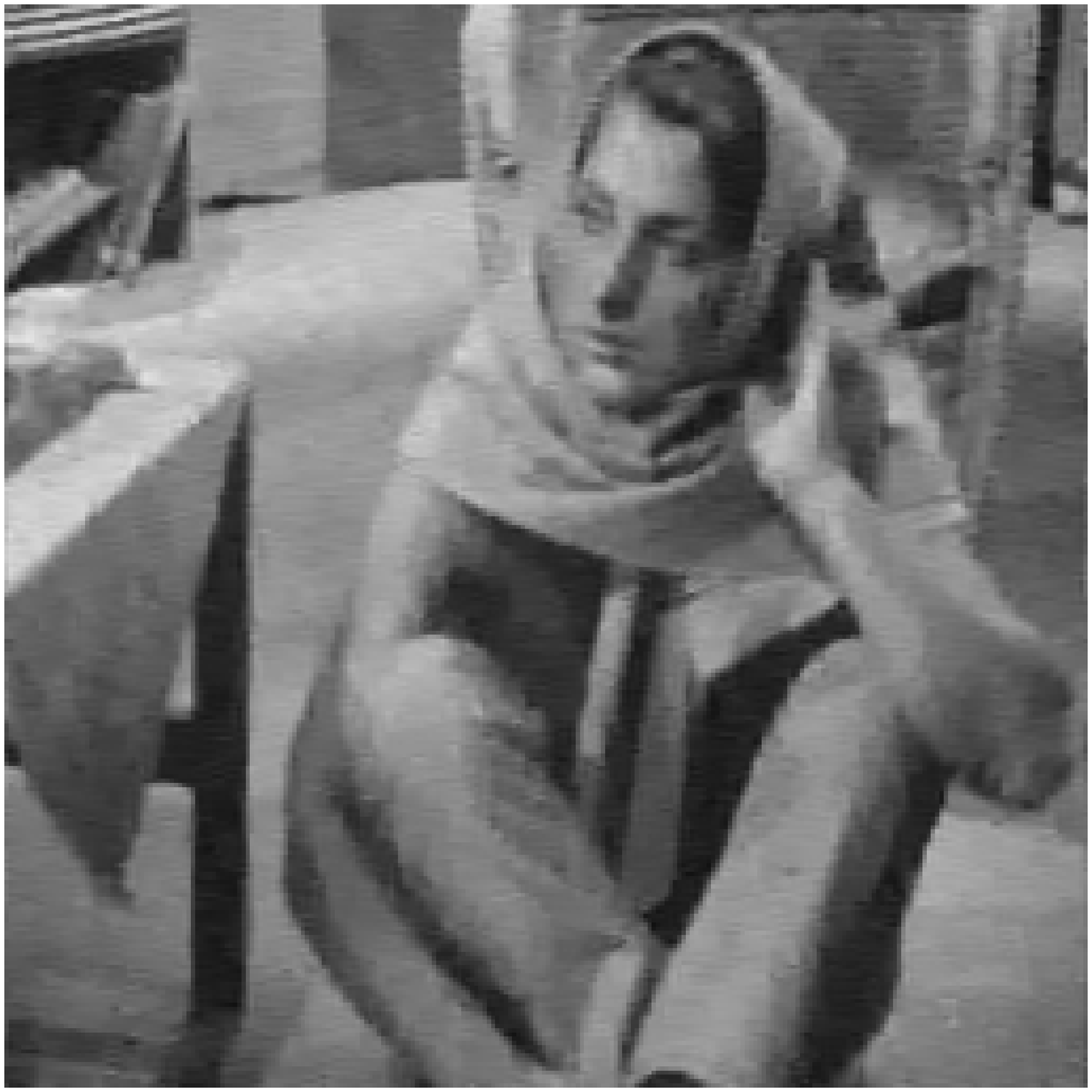}\\
    \includegraphics[width=0.19\textwidth]{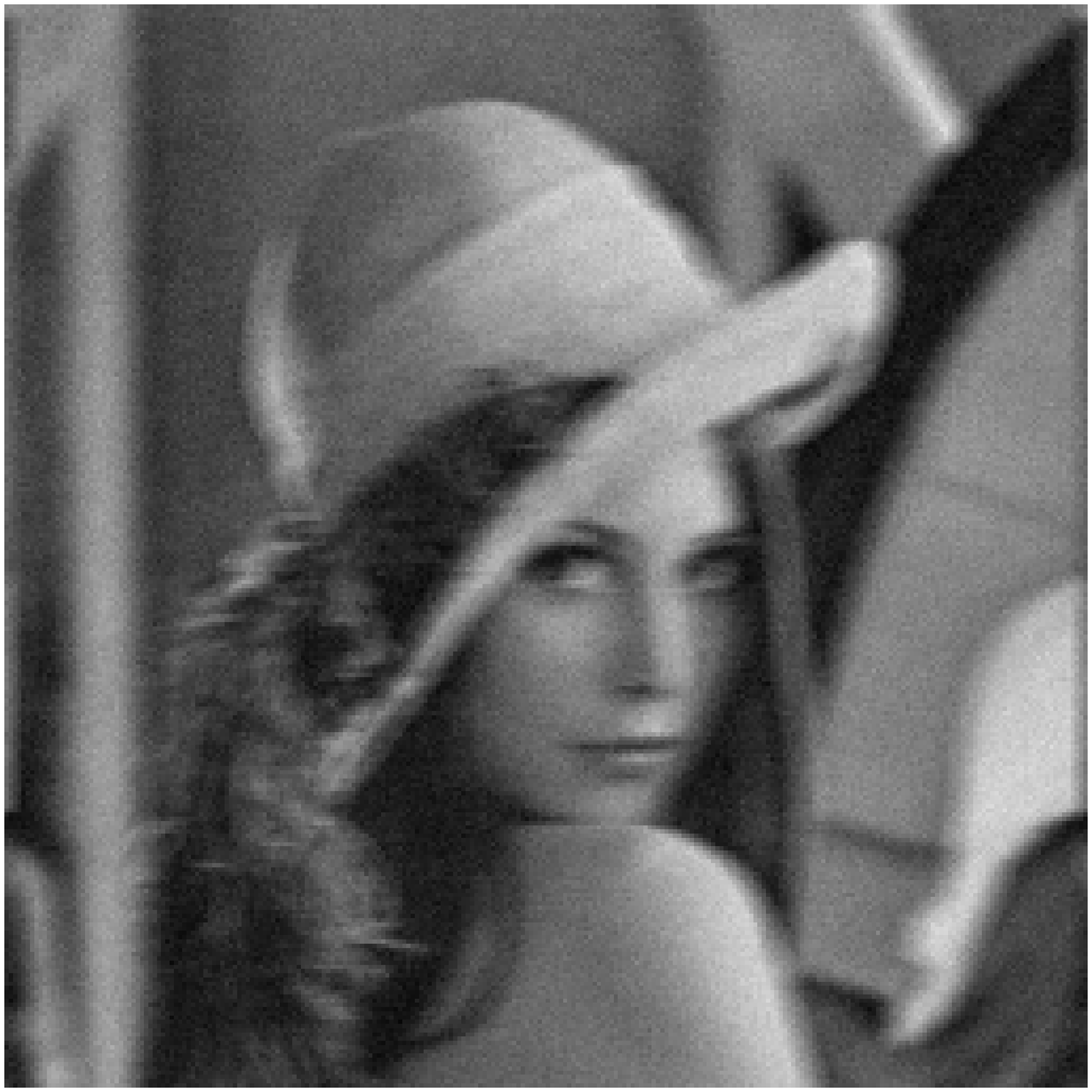}
    \includegraphics[width=0.19\textwidth]{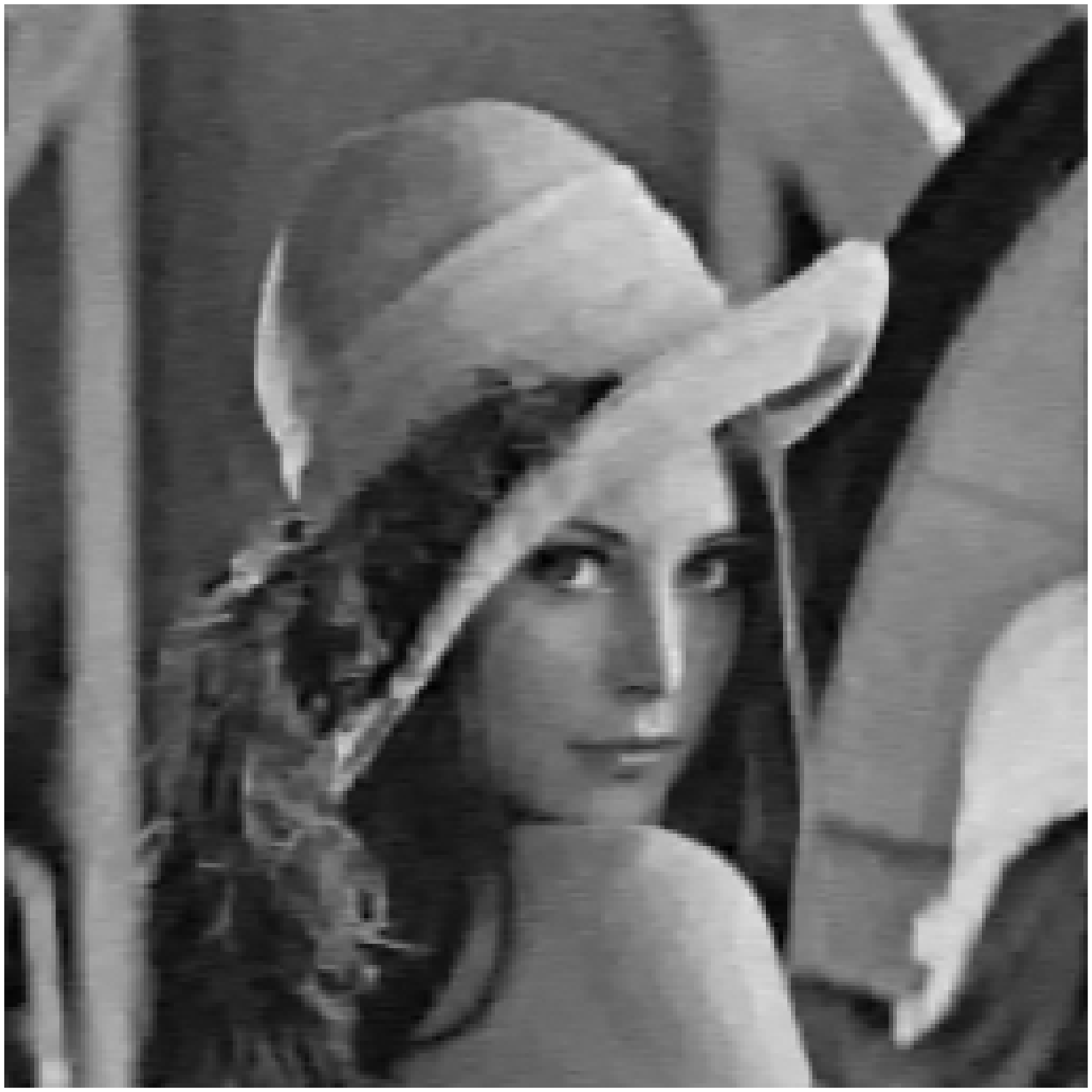}
    \includegraphics[width=0.19\textwidth]{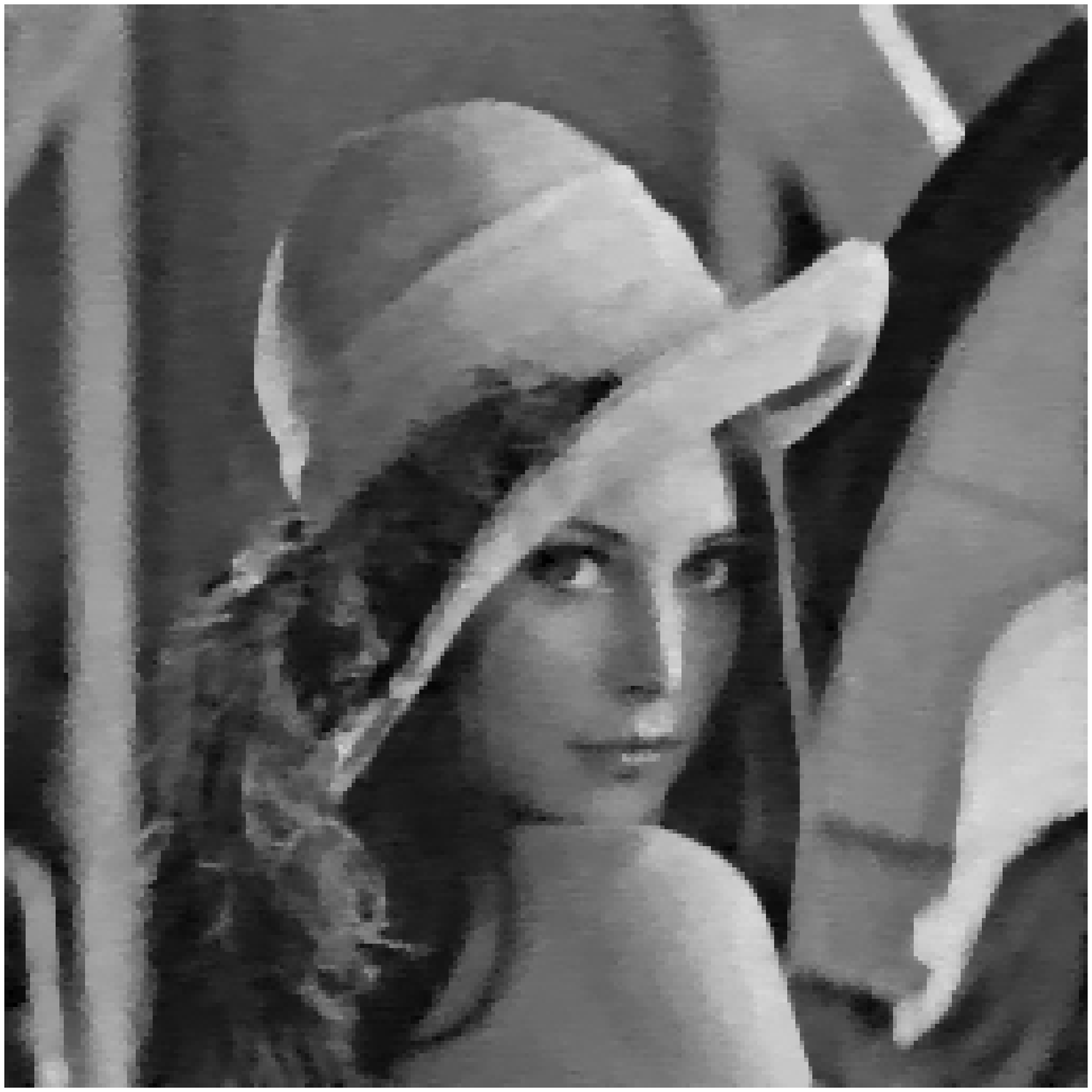}
    \includegraphics[width=0.19\textwidth]{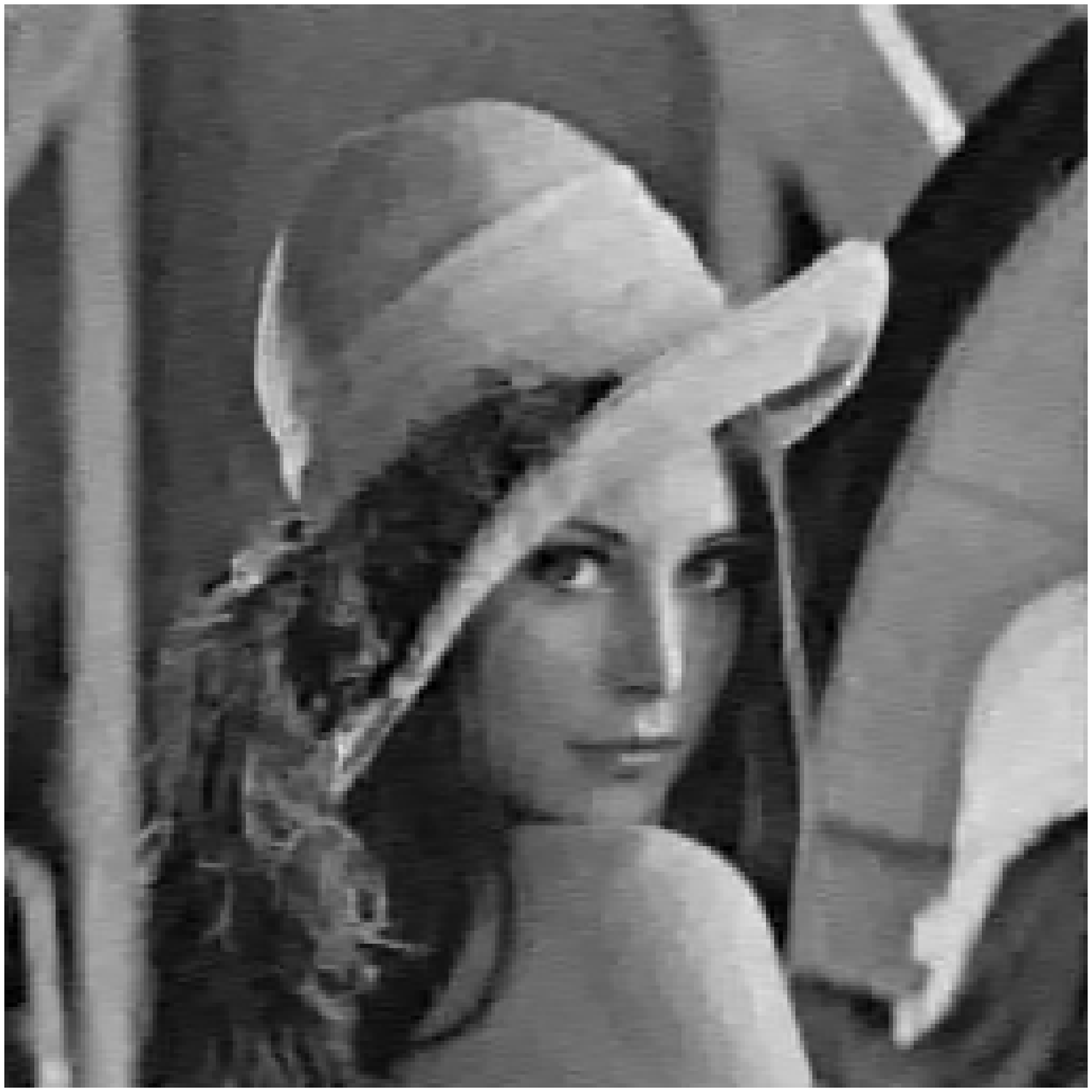}
    \includegraphics[width=0.19\textwidth]{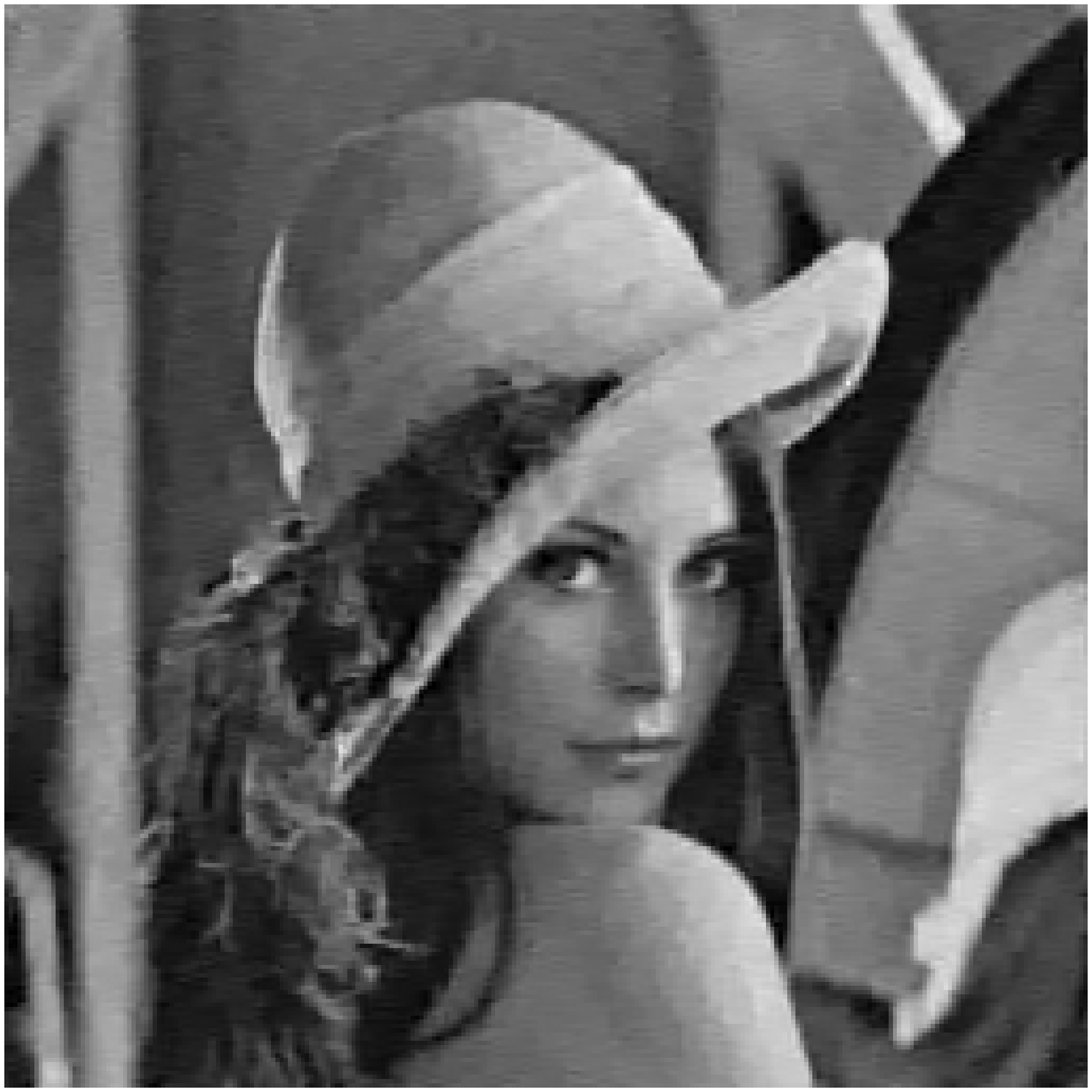}
	\caption{Comparison for image deblurring.From left to right are: blurry images and  deblurred by  ${l_o}$-WF, EDWF and ours using equations (\ref{13}) and (\ref{17}), respectively.}
	\label{F；DCONv1}
\end{figure}

\begin{table*}[htbp]
	\centering
	\caption{PSNR  of the compared four methods for image deblurring}
	\centering
	\begin{tabular}{c|c|c|c|c|c}
		\hline
		Image &	original& ${l_o}$-WF&   EDWF& \tabincell{c}{ours\\ using (\ref{13})}&  \tabincell{c}{ours\\ using (\ref{17})}\\
		\hline
		1&	21.43&	26.25&	26.16&  26.72   	&\textbf{26.85}\\
		2&	26.50&	32.55&	31.77&	32.56    &\textbf{32.61}\\
		3&	22.75&	25.60&	25.44&	25.94   &\textbf{25.98}\\
		4&	22.57&	25.39&	25.31&  25.78	&\textbf{25.82}\\ 
		5&	24.28&	28.07&	27.70&  28.14	&\textbf{28.17}\\
		6&	22.63&	27.73&	27.56&  27.94	&\textbf{28.02}\\
		average&	23.36&	27.60&	27.32&	27.84&\textbf{27.91}\\
		Total Time(s)&	-&	26.67&	141.34&\textbf{7.67}&	7.76\\
		\hline
	\end{tabular}%
\label{T2}
\end{table*}%

\section{Conclusions} In this paper, we proposed a vector total variation (VTV) of feature image model for image restoration.  The  VTV imposes different smoothing powers on different features and thus can simultaneously preserve edges and remove noises. Next, the existence of  solution for the model was proved and the split Bregman algorithm was used to solve the model. At last, we used the wavelet filter banks to explicitly define the feature operator and 
presented some experimental results to show its  advantage over the related methods  in both quality and efficiency.


\bibliographystyle{unsrt}
\bibliography{1}
\end{document}